\documentclass[letterpaper]{article} % DO NOT CHANGE THIS
\usepackage{aaai23}  % DO NOT CHANGE THIS
\usepackage{times}  % DO NOT CHANGE THIS
\usepackage{helvet}  % DO NOT CHANGE THIS
\usepackage{courier}  % DO NOT CHANGE THIS
\usepackage[hyphens]{url}  % DO NOT CHANGE THIS
\usepackage{graphicx} % DO NOT CHANGE THIS
\usepackage{natbib}  % DO NOT CHANGE THIS AND DO NOT ADD ANY OPTIONS TO IT
\usepackage{caption} % DO NOT CHANGE THIS AND DO NOT ADD ANY OPTIONS TO IT
\usepackage{algorithm}
\usepackage{algorithmic}

\usepackage{newfloat}
\usepackage{listings}
\DeclareCaptionStyle{ruled}{labelfont=normalfont,labelsep=colon,strut=off} % DO NOT CHANGE THIS
\floatstyle{ruled}
\newfloat{listing}{tb}{lst}{}
\floatname{listing}{Listing}
\title{Double Doubly Robust Thompson Sampling \\
for Generalized Linear Contextual Bandits}
\author{
    Wonyoung Kim,\textsuperscript{\rm 1}
    Kyungbok Lee, \textsuperscript{\rm 2}
    Myunghee Cho Paik \textsuperscript{\rm 2} \textsuperscript{\rm 3} \thanks{Corresponding author}
}
\affiliations {
    \textsuperscript{\rm 1} Department of Industrial Engineering and Operations Research, Columbia University \\
    \textsuperscript{\rm 2} Department of Statistics, Seoul National University \\
    \textsuperscript{\rm 3} Shepherd23 Inc.\\
    wk2389@columbia.edu, turtle107@snu.ac.kr, myungheechopaik@snu.ac.kr
}

\usepackage{amsmath}
\usepackage{amssymb}
\usepackage{mathtools}
\usepackage{amsthm}
\theoremstyle{plain}
\newtheorem{thm}{Theorem}[section]

\newtheorem{lem}[thm]{Lemma}
\newtheorem{cor}[thm]{Corollary}
\theoremstyle{definition}

\newtheorem{assumption}{Assumption}
\theoremstyle{remark}
\newtheorem{rem}[thm]{Remark}

\global\long\def\Expectation{\text{\ensuremath{\mathbb{E}}}}%
\global\long\def\Probability{\mathbb{P}}%
\global\long\def\Var{\mathbb{V}}%
\global\long\def\Real{\mathbb{R}}%
\global\long\def\CE#1#2{\Expectation\left[\left.#1\right|#2\right]}%
\global\long\def\CV#1#2{\Var\left[\left.#1\right|#2\right]}%
\global\long\def\CP#1#2{\Probability\left(\left.#1\right|#2\right)}%
\global\long\def\Indicator#1{\mathbb{I}\left(#1\right)}%
\global\long\def\abs#1{\left|#1\right|}%
\global\long\def\norm#1{\left\Vert #1\right\Vert }%
\global\long\def\Mineigen#1{\lambda_{\text{min}}\left(#1\right)}%
\global\long\def\Maxeigen#1{\lambda_{\text{max}}\left(#1\right)}%
\global\long\def\Trace#1{\text{tr}\left(#1\right)}%
\global\long\def\Regret#1{\texttt{regret}(#1)}%
\global\long\def\Optimalarm#1{a_{#1}^{*}}%
\global\long\def\Filtration#1{\mathcal{F}_{#1}}%
\global\long\def\History#1{\mathcal{H}_{#1}}%
\global\long\def\Context#1#2{X_{#1,#2}}%
\global\long\def\XX#1#2{\boldsymbol{X}_{#1,#2}}%
\global\long\def\Prederror#1#2{\mathcal{D}_{#1}(#2)}%
\global\long\def\APrederror#1#2{\mathcal{D}^{A}_{#1}(#2)}%
\global\long\def\Diff#1#2{\Delta_{#1}(#2)}%
\global\long\def\Setofcontexts#1{\mathcal{X}_{#1}}%
\global\long\def\Reward#1{Y_{#1}}%
\global\long\def\Action#1{a_{#1}}%
\global\long\def\SelectionP#1#2{\pi_{#1,#2}}%
\global\long\def\Tildepi#1#2{\tilde{\pi}_{#1,#2}}%
\global\long\def\DRreward#1#2#3{Y_{#1,#2}^{#3}}%
\global\long\def\BetaSampled#1#2{\tilde{\beta}_{#1,#2}}%
\global\long\def\Impute#1{\breve{\beta}_{#1}}%
\global\long\def\Estimator#1{\widehat{\beta}_{#1}}%
\global\long\def\Error#1{\eta_{#1}}%
\global\long\def\DRError#1#2#3{\eta_{#1,#2}^{#3}}%
\global\long\def\NMLE#1{\widehat{\beta}_{#1}^{S}}%
\global\long\def\Sampledreward#1#2{\tilde{Y}_{#1,#2}}%
\global\long\def\Candidatearm#1{\tilde{a}_{#1}}%
\global\long\def\Exploration{\mathcal{T}_{*}}%
\global\long\def\MarginExp{\mathcal{T}_{0}}%

\begin{document}
\maketitle

%---------------------------------------
% Abstract
%---------------------------------------
\begin{abstract}
We propose a novel algorithm for generalized linear contextual bandits (GLBs) with an $\tilde{O}(\sqrt{\kappa^{-1} \phi^{-1} T})$ regret over $T$ rounds where $\phi$ is the minimum eigenvalue of the covariance of contexts and $\kappa$ is a lower bound of the variance of rewards. 
In several identified cases of $\phi^{-1}=O(d)$, where $d$ is the dimension of contexts, our result is the first regret bound for generalized linear  bandits (GLBs) achieving the order $\sqrt{d}$ without discarding the observed rewards.
Previous approaches achieve the regret bound of order $\sqrt{d}$ by discarding the observed rewards, whereas our algorithm achieves the bound incorporating contexts from all arms in our double doubly-robust (DDR) estimator.
The DDR estimator is a subclass of doubly-robust estimator but with a tighter error bound.
We also provide an $O(\kappa^{-1} \phi^{-1} \log (NT) \log T)$ regret bound for $N$ arms under a probabilistic margin condition.
This is the first regret bound under the margin condition for linear models or GLMs when contexts are different for all arms but coefficients are common.
We conduct empirical studies using synthetic data and real examples, demonstrating the effectiveness of our algorithm.
%We achieve this bound using a novel estimator called double doubly-robust (DDR) estimator, a subclass of doubly-robust (DR) estimator but with a tighter error bound.
%The approach of \citet{auer2002using} 
%Regret bounds under the margin condition are given by \citet{bastani2020online} and \citet{bastani2021mostly} under the setting that contexts are common to all arms but coefficients are arm-specific.  
\end{abstract}

%---------------------------------------
% Introduction
%---------------------------------------
\section{Introduction}
In multi-armed bandits (MABs), a learner repeatedly chooses an action or arm from action sets given in an environment and observes the reward for the chosen arm.   
The goal is to find a rule for choosing arms to maximize the expected cumulative rewards.
A linear contextual bandit is an MAB with context vectors for each arm in which the expected reward is a linear function of the corresponding context vector.
Popular contextual bandit algorithms include upper confidence bound (\citet{abbasi2011improved}, \texttt{LinUCB}) and Thompson sampling (\citet{agrawal2013thompson}, \texttt{LinTS}) whose theoretical properties have been studied \citep{auer2002using,chu2011contextual,abbasi2011improved,agrawal2014thompson}.
More recently, extensions to generalized linear models (GLMs) have received significant attention.  
The study by \citet{filippi2010parametric} is one of the pioneering studies to propose contextual bandits for GLMs with regret analysis.
\citet{abeille2017linear} extended \texttt{LinTS} to GLM rewards.  
In the GLM, the variance of a reward is related to its mean and the regret bound typically depends on the lower bound of the variance, $\kappa$.  
\citet{faury2020improved} demonstrated the regret bound free of $\kappa$ for a logistic case.

When we focus on the dependence of $d$ on the regret bound, an $\tilde{O}(\kappa^{-1}\sqrt{dT})$ regret bound has been achieved by \citet{li2017provably}, where $\tilde{O}$ denotes big-$O$ notation up to logarithmic factors.
For logistic models, \citet{junimproved2021} achieved $\tilde{O}(\sqrt{dT})$.
These two methods used the approach of \citet{auer2002using} whose main idea is to carefully compose independent samples to develop an estimator and derive the bound using this independence.  
However, to maintain independence, the developed estimator ignores many observed rewards.
Despite of this limitation, no existing algorithms have achieved a regret bound sublinear in $d$ without using the approach of \citet{auer2002using}.
We propose a novel contextual bandit algorithm for generalized linear rewards with $\tilde{O}(\sqrt{\kappa^{-1}dT})$ in several practical cases.
To our knowledge, this regret bound is firstly achieved without relying on the approach of \citet{auer2002using}.

Our proposed algorithm is the first among \texttt{LinTS} variants with a regret bound of order $\sqrt{d}$ for linear or generalized linear payoff.
The proposed algorithm is equipped with a novel estimator called double doubly robust (DDR) estimator which is a subclass of doubly robust (DR) estimators with a tighter error bound than conventional DR estimators.
The DR estimators use contexts from unselected arms after imputing predicted responses using a class of imputation estimators.
Based on these estimators, \citet{kim2021doubly} have recently proposed a \texttt{LinTS} variant for linear rewards which achieves an $\tilde{O}(d\sqrt{T})$ regret bound for several practical cases.
With our proposed DDR estimator, we develop a novel \texttt{LinTS} variant for generalized linear rewards which improves the regret bound by a factor of $\sqrt{d}$.

\begin{table*}[t]
\caption{The main orders of regret bounds for GLMs/logistic bandit algorithms 
(See Appendix~\ref{sec:full_table} for the comparison with more algorithms). 
For details about the assumptions, see Section~\ref{subsec:regret_bound}}
\label{tab:regret_orders}
\footnotesize
\begin{center}
\begin{tabular}{|c|c|c|c|}
\hline
\textbf{Algorithm} & \textbf{Regret Upper Bound}    &   \textbf{Model} & \textbf{Assumptions} \\ \hline
\begin{tabular}[c]{@{}c@{}} \texttt{GLM-UCB} \citep{filippi2010parametric}\end{tabular} & $O(\kappa^{-1} d \sqrt{T} \log^{3/2} T)$ & GLM & 1-3 \\ \hline
\begin{tabular}[c]{@{}c@{}} \texttt{SupCB-GLM} \citep{li2017provably}\end{tabular}         & $O(\kappa^{-1}\sqrt{d T \log NT} \log T$)                                                           & GLM & 1-5\\ \hline
\begin{tabular}[c]{@{}c@{}} {\tt TS}{\tt (GLM)} \citep{abeille2017linear} \end{tabular}       & $O(\kappa^{-1}d^{3/2}\sqrt{T}\log T)$ & GLM & 1-3 \\ \hline
\begin{tabular}[c]{@{}c@{}} \texttt{Logistic UCB-2} \citep{faury2020improved}\end{tabular} & $O(d \sqrt{T} \log T)$ & Logistic & 1-3\\ \hline
\begin{tabular}[c]{@{}c@{}} \texttt{SupLogistic} \citep{junimproved2021} \end{tabular}      & $O(\sqrt{dT \log NT}\log T )$ & Logistic & 1-5\\ \hline
\texttt{DDRTS-GLM} (Proposed)  & $O(\sqrt{\kappa^{-1} d T \log NT})$ & GLM & 1-5, $\phi^{-1}\!=\!O(d)$ \\ \hline
\end{tabular}
\end{center}
\end{table*}

We also demonstrate a logarithmic cumulative regret bound under a margin condition.
Obtaining bounds under this margin condition is challenging because it requires a tight prediction error bound of order $1/\sqrt{t}$ for each round $t$.
Our DDR estimator, which has a tighter error bound than DR estimators, enables us to derive a logarithmic cumulative regret bound under the margin condition.

In our experiments, we demonstrate that our proposed algorithm performs better with less a priori knowledge than several generalized linear contextual bandit algorithms.
Many contextual bandit algorithms require the knowledge of time horizon $T$ \citep{li2017provably, junimproved2021} or a set which includes the true parameter \citep{filippi2010parametric, jun2017scalable,faury2020improved}.
This hinders the application of the algorithms to the real examples.
Our proposed algorithm does not require the knowledge of $T$ or the parameter set and widely applicable with superior performances.

The main contributions of this paper are as follows:
\begin{itemize}
    \item 
    We propose a novel generalized linear contextual bandit (GLB) algorithm that has $\tilde{O}(\sqrt{\kappa^{-1} d T})$ in several identified cases (Theorem~\ref{thm:regret_bound}).
    This is the first regret bound for GLBs with the order $\sqrt{d}$ without relying on the arguments of \citet{auer2002using}. 

    \item
    We provide a novel estimator called DDR estimator which is a subclass of the DR estimators but has a tighter error bound than conventional DR estimators.
    The proposed DDR estimator uses an explicit form of the imputation estimator, which is also doubly robust and guarantees our novel theoretical results.
    
    \item 
    We provide novel theoretical analyses that extend DR Thompson sampling (\citet{kim2021doubly}, \texttt{DR-LinTS}) to GLBs and improve its regret bound by a factor of $\sqrt{d}$.
    Our analyses are different from those of \texttt{DR-LinTS} in deriving a new regret bound capitalizing on independence (Lemma~\ref{lem:prediction_error}) and new maximal elliptical potential lemma (Lemma~\ref{lem:elliptical_potential_lemma}).
    
    \item 
    We provide an $O(\kappa^{-1} d \log NT \log T)$ regret bound under a probabilistic margin condition (Theorem~\ref{thm:fast_regret_bound}).  
    This is the first logarithmic regret bound for linear and generalized linear payoff with arm-specific contexts.
    
    \item 
    Simulations using synthetic datasets and analyses with two real datasets show that the proposed method outperforms existing GLMs/logistic bandit methods. 

\end{itemize}

%---------------------------------------
% Generalized Linear Contextual Bandit Problem
%---------------------------------------
\section{Generalized linear contextual bandit problem}

The GLB problem considers the case in which the reward follows GLMs. 
%(we provide a review for GLMs in Appendix \ref{sec:GLM}).
\citet{filippi2010parametric} presented a GLB problem with a finite number of arms.
In this setting, the learner faces $N$ arms and each arm is associated with contexts.
For each $i\in\{1,\ldots,N\}:=[N]$, we denote a $d$-dimensional context for the $i$-th arm at round $t$ by $\Context it\in\Real^{d}$.
At round $t$, the learner observes the contexts $\{\Context 1t,\ldots,\Context Nt\}:=\Setofcontexts{t}$, pulls one arm $\Action t\in[N]$, and observes $\Reward t$.  
Let $\History t$ be the history at round $t$ that contains the contexts $\{\Setofcontexts{\tau}\}_{\tau=1}^{t}$, chosen arms $\{\Action{\tau}\}_{\tau=1}^{t-1}$ and the corresponding rewards $\{\Reward{\tau}\}_{\tau=1}^{t-1}$.
We assume that the distribution of the reward is
\[
\CP{Y_{t}=y}{\History t,\Action t} \propto \exp\left\{ y\theta_{\Action{t},t}-b\left(\theta_{\Action{t},t}\right) \right\},
\]
where $\theta_{\Action{t},t}=\xi_{a_t,t}:=\Context{\Action t}t^{T}\beta^{*}$, with the function $b$ known and assumed to be twice differentiable.
Then we have $\CE{\Reward t}{\History t,\Action t}\!=\!b^{\prime}\left(\theta_{a_t,t}\right)=\mu(\xi_{a_t,t})$ and $\CV{Y_{t}}{\History t,\Action t}=b^{\prime\prime}\left(\theta_{a_t,t}\right)$, for some unknown $\beta^{*}\in\Real^{d}$.
Let $\Optimalarm t:=\arg\max_{i=1,\ldots,N}\left\{ \mu(\Context it^{T}\beta^{*})\right\}$ be the optimal arm that maximizes the expected reward at round $t$.  
We define the regret at round $t$ by $\Regret t:=\mu(\Context{\Optimalarm t}t^{T}\beta^{*})-\mu(\Context{\Action t}t^{T}\beta^{*})$.
Our goal is to minimize the sum of regrets over $T$ rounds, $R(T):=\sum_{t=1}^{T}\Regret t$.
The total number of rounds $T$ is finite but possibly unknown.

%--------------------------------
% Regret order table
%-------------------------------- 

%\begin{table}[t]
% \caption{The main orders of regret bounds for GLMs/logistic bandit algorithms (See Section~\ref{sec:full_table} for the comparison with more algorithms).}
% \vskip -0.05in
% \label{tab:regret_orders}
% \begin{center}
% \begin{tabular}{|c|c|c|}
% \hline
% \textbf{Algorithm} & \textbf{Regret Upper Bound}    &   \textbf{Model} \\ \hline
% \begin{tabular}[c]{@{}c@{}} \texttt{GLM-UCB} \citep{filippi2010parametric}\end{tabular} & $O(\kappa^{-1} d \sqrt{T} \log^{3/2} T)$ & GLM      \\ \hline
% \begin{tabular}[c]{@{}c@{}} \texttt{SupCB-GLM} \citep{li2017provably}\end{tabular}         & $O(\kappa^{-1}\sqrt{d T \log NT} \log T$)                                                           & GLM \\ \hline
% \begin{tabular}[c]{@{}c@{}} {\tt TS}{\tt (GLM)} \citep{abeille2017linear} \end{tabular}       & $O(\kappa^{-1}d^{3/2}\sqrt{T}\log T)$ & GLM  \\ \hline
% \begin{tabular}[c]{@{}c@{}} \texttt{Logistic UCB-2} \citep{faury2020improved}\end{tabular} & $O(d \sqrt{T} \log T)$ & Logistic \\ \hline
% \begin{tabular}[c]{@{}c@{}} \texttt{SupLogistic} \citep{junimproved2021} \end{tabular}      & $O(\sqrt{dT \log NT}\log T )$ & Logistic \\ \hline
% \texttt{DDRTS-GLM} (Proposed)  & $O(\sqrt{\kappa^{-1} d T \log NT})$ & GLM   \\ \hline
% \end{tabular}
% \end{center}
% \vskip -0.2in
% \end{table}

%--------------------------------
% Related works
%-------------------------------- 
\section{Related works}

Table~\ref{tab:regret_orders} summarizes the comparison of main regret orders for GLMs/logistic bandit algorithms.
\citet{filippi2010parametric} extended {\tt LinUCB} and proposed an algorithm for GLB (\texttt{GLM-UCB}) with a regret bound of $O(\kappa^{-1} d \sqrt{T} \log^{3/2}T)$, where $\kappa$ is a lower bound of $\mu^{\prime}$.
\citet{abeille2017linear} extended {\tt LinTS} to GLBs and demonstrated a regret bound of $O(\kappa^{-1} d^{3/2}\sqrt{T}\log T)$.
In contrast to linear models, the Gram matrix in GLMs depends on the mean and thus the regret bound has the factor of $\kappa^{-1}$ which can be large.
\citet{faury2020improved} solved this problem by proposing an UCB-based algorithm for logistic models with a regret bound of $O(d \sqrt{T} \log T)$ free of $\kappa$.

%Computationally, the dependence of the Gram matrix on the mean makes a simple round-wise update of the Gram matrix as in the linear payoff impossible, and enforces the storage of all histories. 
%\citet{jun2017scalable} provided a solution to this problem using a scalable algorithm based on optimism in the face of uncertainty principle.

Other related works have focused on achieving $\tilde{O}(\sqrt{dT})$ regret bounds using the argument of \citet{auer2002using}.  
The \texttt{SupCB-GLM} algorithm \citep{li2017provably} is extended from  
{\tt GLM-UCB} \citep{filippi2010parametric} yielding an $O(\kappa^{-1} \sqrt{dT\log T N} \log T)$ regret bound, whereas \texttt{SupLogistic} \citep{junimproved2021} is extended from
{\tt Logistic UCB-2} \citep{faury2020improved}  with an  $O(\sqrt{d T \log T N} \log T)$ regret bound.
\texttt{SupCB-GLM} and \texttt{SupLogistic} have the best-known regret bounds for contextual bandits with GLM and logistic models, respectively.
However, they do not incorporate many observed rewards to achieve independence, resulting in additional rounds for their estimators to achieve certain precision level.%, costing extra $\log T$ in the regret bound. 

The DR method has been employed in the bandit literature for linear payoffs \citep{kim2019doubly,dimakopoulou2019balanced,kim2021doubly}.  
Except in the work by \citet{dimakopoulou2019balanced}, the merit of this technique is the use of full contexts along with pseudo-rewards.  
Our approach shares the merit of using full contexts but also uses a new estimator with a tighter error bound and more elaborate pseudo-rewards than conventional DR methods. 
This calls for new regret analyses which are summarized in Lemmas~\ref{lem:prediction_error} and~\ref{lem:elliptical_potential_lemma}.
%In addition, we propose new regret analyses which are summarized in Lemmas~\ref{lem:prediction_error} and~\ref{lem:elliptical_potential_lemma}. 

Under a {\it probabilistic} margin condition, \citet{bastani2020online} and \citet{bastani2021mostly} presented $O(\log T)$ regret bounds for (generalized) linear contextual bandits when contexts are common to all arms but coefficients are arm-specific.  
In a setting in which contexts are arm-specific but with common coefficients, previous works have shown regret bounds under {\it deterministic} margin conditions \citep{dani2008stochastic,abbasi2011improved}.  
In the latter setting, our bound is the first regret bound under a {\it probabilistic} margin condition for linear models or GLMs.

%--------------------------------
% Proposed Methods
%-------------------------------- 
\section{Proposed methods}

\subsection{Subclass of doubly robust estimator: double doubly-robust (DDR) Estimator}

For linear contextual bandits, \citet{kim2021doubly} employed a DR estimator whose merit is to use all contexts, selected or unselected.  
To use unselected contexts in the DR estimator, the authors replaced the unobserved reward for the unselected context with a pseudo-reward defined as follows:
\begin{equation}
\DRreward it{\Impute t}:=\left\{ 1-\frac{\Indicator{\Action t=i}}{\SelectionP it}\right\} \mu(\Context it^{T}\Impute t)+\frac{\Indicator{\Action t=i}}{\SelectionP it}\Reward t,
\label{eq:pseudo_reward}
\end{equation}
where $\SelectionP it=\CP{\Action{t}=i}{\History t}>0$ \footnote{The selection probability is nonzero for all arms in Thompson sampling with Gaussian prior.}is the selection probability and $\Impute t$ is an imputation estimator for $\beta^{*}$ at round $t$.
\citet{kim2021doubly} studied the case of $\mu(x)=x$ and proposed to set $\Impute{t}$ as any $\History t$-measurable estimator that renders $\DRreward it{\Impute t}$ unbiased for the conditional mean because of $\mathbb{E}(\mathbb{I}(a_t=i)|\History t)=\SelectionP{i}{t}$. 
This choice of imputation estimators covers a wide class of DR estimators.
In this paper, we propose a subclass of the DR estimators called DDR estimator for GLB problem.
The proposed subclass estimator has an explicit form of $\Impute{t}$ which is crucial to our novel theoretical results.

To introduce our imputation estimator $\Impute{t}$, we begin by developing a new bounded estimator,
\[
\NMLE t:= 
    \begin{cases}
        \widehat{\beta}_{t}^{A} & \text{if }\norm{\widehat{\beta}_{t}^{A}}_{2}\le S\\
        S\frac{\widehat{\beta}_{t}^{A}}{\norm{\widehat{\beta}_{t}^{A}}_{2}} & \text{otherwise}
    \end{cases}
\]
for some $S > 0$, where $\widehat{\beta}^{A}_{t}$ is the solution to the score equation, $\sum_{\tau=1}^{t}\left\{ Y_{\tau}-\mu\left(\Context{\Action{\tau}}{\tau}^{T}\beta\right)\right\} \Context{\Action{\tau}}{\tau}=0$.
Now the imputation estimator $\Impute{t}$ at round $t$ is the solution to
\[
\sum_{\tau=1}^{t}\sum_{i=1}^{N}\left\{ \DRreward i{\tau}{\NMLE t}-\mu\left(\Context i{\tau}^{T}\beta\right)\right\} \Context i{\tau}-\lambda \beta = 0,
\]
where $\lambda>0$ is a regularization parameter. 
Regardless of the value of $S$, the pseudo-reward $\DRreward i{\tau}{\NMLE t}$ is unbiased.
Different from DR estimators, the imputation estimate $\Impute{t}$ is also robust and satisfies
\begin{equation}
\norm{\Impute{t}-\beta^{*}}_2 \le \sqrt{\frac{\kappa}{Nd}},
\label{eq:imputation_estimator_condition}
\end{equation}
when $t \ge \Exploration = \Omega(\kappa^{-3} \phi^{-2} N  d^2 \log T)$.
The proof of~\eqref{eq:imputation_estimator_condition} and detailed expression of $\Exploration$ is in Appendix~\ref{subsec:imputation_estimator_bound}.
With this newly defined imputation estimator $\Impute{t}$ and the corresponding pseudo-reward~\eqref{eq:pseudo_reward}, the proposed DDR estimator $\Estimator{t}$ is defined as the solution to,
\begin{equation}
U_{t}(\beta):=\sum_{\tau=1}^{t}\sum_{i=1}^{N}\left\{\DRreward i{\tau}{\Impute t}-\mu(\Context i{\tau}^{T}\beta)\right\}\Context i{\tau}-\lambda\beta=0.
\label{eq:score_equation}
\end{equation}
This DDR estimator uses not only a robust imputation estimator, but also a more elaborate pseudo-reward than that in DR estimators. 
Specifically, for all $\tau\in[t]$, DR estimators use $\DRreward{i}{\tau}{\Impute{\tau}}$, whereas our DDR estimator computes $\DRreward{i}{\tau}{\Impute{t}}$, updating pseudo-rewards on the basis of the most up-to-date imputation estimator.  
%This requires extra theoretical consideration using the covering number technique in Lemma~\ref{lem:general_DR_bound}.
Both the imputation estimator with a tighter estimation error bound and elaborate pseudo-rewards result in a subsequent reduction of the prediction error bound for the DDR estimator, which plays a crucial role in reducing $\sqrt{d}$ in the regret bound compared to that of \citet{kim2021doubly}.
%The improved prediction error bound of DDR estimator is in Lemma~\ref{lem:prediction_error}.
%See the estimation error bound of the proposed DDR estimator in Theorem~\ref{thm:estimation_error_bound} which is required to establish Theorems~\ref{thm:regret_bound} and~\ref{thm:fast_regret_bound}.

\subsection{Double doubly robust Thompson Sampling algorithm}

Our proposed algorithm, double doubly robust Thompson sampling algorithm for generalized linear bandits (\texttt{DDRTS-GLM}) is presented in Algorithm~\ref{alg:DRTS}.
At each round $t\ge 2$, the algorithm samples $\BetaSampled it$ from the distribution $\mathcal{N}(\Estimator{t-1},v^{2}V_{t-1}^{-1})$ for each $i\in[N]$ independently.
Define $\Sampledreward{i}{t}:=\mu (\Context{i}{t}^T\BetaSampled{i}{t})$ and let $m_{t}:=\arg\max_{i} \Sampledreward{i}{t}$ be a candidate action.  
After observing $m_t$, compute $\Tildepi{m_t}{t}:=\Probability(\Sampledreward{m_t}t=\max_{i}\Sampledreward it|\History t)$. 
If $\Tildepi{m_t}{t}>\gamma$, the arm $m_t$ is selected, i.e., $a_t=m_t$. 
Otherwise, the algorithm resamples $\BetaSampled it$ until it finds another arm satisfying $\Tildepi{i}{t} > \gamma $ up to a predetermined fixed value $M_{t}$.

\texttt{DDRTS-GLM} requires additional computations because of the resampling and the DDR estimator.
The computation for resampling is invoked when $\Tildepi{m_t}{t} \le \gamma$ but this does not occur often in practice.
In computing $\Tildepi{m_t}{t}$, we refer to Section \ref{sec:pi_computation} in \citet{kim2021doubly} which proposed a Monte-Carlo estimate and showed that the estimate is efficiently computable.
Furthermore, the estimate is consistent and does not affect the theoretical results.
Most additional computations in \texttt{DDRTS-GLM} occurs in estimating $\Estimator{t}$ which requires the imputation estimator $\Impute{t}$ and contexts of all arms.
This additional computation of \texttt{DDRTS-GLM} is a minor cost of achieving a regret bound sublinear to $d$ and superior performance compared to existing GLB algorithms.
%\textcolor{red}{ (how about) At the cost of minor additional computation, \texttt{DDRTS-GLM} gains a regret bound sublinear to $d$ and superior performance compared to existing GLB algorithms.  }Even though \texttt{DDRTS-GLM} requires these computations, it performs better than several GLB algorithms and achieves a regret bound sublinear to $d$ in several practical cases.

%---------------------------------------------
% Algorithm
%---------------------------------------------
\begin{algorithm}[t]
\caption{Double Doubly Robust Thompson Sampling for Generalized Linear Contextual Bandits (\texttt{DDRTS-GLM})}
%%%%% (?) V_t \to \widehat{W}_t 
%%%% Compute $\tilde{a}_t :\arg\max_{i\in[N]}$
\begin{algorithmic}[1]
\label{alg:DRTS}
\small
\STATE \textbf{INPUT:} exploration parameter $v$, regularization parameter $\lambda$, the number of maximum possible resampling $M_t>0$, threshold value for resampling $\gamma\in[(N+1)^{-1},N^{-1})$, $S>0$.
\STATE Initialize $V_1:= \lambda I_d$, $\Estimator{t}=0_d$ and sample $\Action 1$ from $\{1,\ldots,N\}$, randomly
\FOR{ $t\ge2$ } 
\STATE Initialize $n=1$ and observe contexts $\Setofcontexts{t}$
\STATE Sample $\BetaSampled 1t, \ldots \BetaSampled{N}{t}$ from  $\mathcal{N}(\Estimator{t-1},v^{2}V_{t-1}^{-1})$, independently and observe $m_t=\arg\max_{i} \mu(\Context it\BetaSampled it)$% $\Tildepi{m_t}t$. 
\IF{ $\Tildepi{m_t}t\le\gamma$ and $n\le M_t$} 
\STATE Set $n \leftarrow n+1$ and go to line 5
\ELSE
\STATE Set $\Action{t}=m_t$, play arm $\Action{t}$ and observe reward $\Reward{m_t}$
\ENDIF 
\STATE Compute $V_{t}\!=\!\sum_{\tau=1}^{t}\sum_{i=1}^{N}\!\mu^{\prime}(\Context i{\tau}^{T}\Estimator{t-1})\Context i\tau\Context i\tau^{T}\!+\!\lambda I_d$, $\Impute{t}$ and $\DRreward i{\tau}{\Impute{t}}$, then solve~\eqref{eq:score_equation} to update $\Estimator{t}$
\ENDFOR
\end{algorithmic}
\end{algorithm}

%-----------------------------------------------------
% Section 4. Regret Analysis
%-----------------------------------------------------
\section{Regret analysis}

In this section, we present two regret bounds for \texttt{DDRTS-GLM}: an $O(\sqrt{\kappa^{-1}d T\log NT})$ regret bound (Theorem~\ref{thm:regret_bound}) and an $O(\kappa^{-1}d\log NT \log T)$ regret bound under a margin condition (Theorem~\ref{thm:fast_regret_bound}).

\subsection{A regret bound of \texttt{DDRTS-GLM}}
\label{subsec:regret_bound}
We provide the following assumptions.

\begin{assumption}[Boundedness]
There exists $S^{*}>0$ such that $\norm{\Context it}_{2} \le 1$ and $\norm{\beta^{*}}_{2}\le S^{*}$ for all $i\in[N]$ and $t\in[T]$.
The value of $S^{*}$ is possibly unknown.
\label{assump:boundedness}
\end{assumption}

\begin{assumption}[Bounded rewards]
There exists $B>0$ such that $\abs{\Reward t}\le B$ for all $t\in[T]$ almost surely.
\label{assump:bounded_rewards}
\end{assumption}

\begin{assumption}[Mean function]
Define a set of vicinity of all possible $\beta^{*}$ by $\mathcal{B}_{r}^{*}:=\{\beta:\|\beta\|_{2}\le r+S^{*}\}$.
Then there exists $r>0$ such that the mean function $\mu$ is twice continuously differentiable on $\{x^{T}\beta:\|x\|_{2}\le1, \beta\in\mathcal{B}_{r}^{*} \}$, and $\kappa:=\inf_{\|x\|_{2}\le 1,\beta\in\mathcal{B}_{r}^{*}}\mu^{\prime}\left(x^{T}\beta\right)>0$.
This implies that $|\mu\prime|\le L_{1}$ and $|\mu^{\prime\prime}|\le L_{2}$ on the bounded set $\{x^{T}\beta:\|x\|_{2}\le1, \beta\in\mathcal{B}_{r}^{*} \}$ for some $L_{1},L_{2} \in (0,\infty)$. 
\label{assump:mean_function}
\end{assumption}

\begin{assumption}[Independently identically distributed contexts]
Define the set of all contexts at round $t\in[T]$ by $\Setofcontexts{t}:=\{\Context{1}{t},\ldots,\Context{N}{t}\}$.
Then the stochastic contexts, $\Setofcontexts{1},\ldots,\Setofcontexts{T}$ are independently generated from a fixed distribution $\mathcal{P}_{X}$.
At each round $t$, the contexts $\Context 1t,\ldots,\Context Nt$ can be correlated with each other. 
\label{assump:iid_contexts}
\end{assumption}

\begin{assumption}[Positive definiteness of the covariance of the contexts]
There exists a positive constant $\phi>0$ such that $\Mineigen{\Expectation[N^{-1}\sum_{i=1}^{N}\Context it\Context it^{T}]}\ge\phi$ for all $t$. 
\label{assump:minimum_eigenvalue}
\end{assumption}

Assumptions \ref{assump:boundedness}-\ref{assump:mean_function} are standard in the GLB literature (see e.g. \citet{filippi2010parametric,jun2017scalable,selfcon2021}) except that Assumption~\ref{assump:boundedness} does not require the knowledge of $S^{*}$.
Assumptions~\ref{assump:iid_contexts} and~\ref{assump:minimum_eigenvalue} were used by \citet{li2017provably} and \citet{junimproved2021} which achieved regret bounds that have only $\sqrt{d}$.
Under these assumptions we present the regret bound of {\tt DDRTS-GLM} in the following theorem.
\begin{thm}
\label{thm:regret_bound}
(A regret bound of {\tt DDRTS-GLM}) Suppose Assumptions 1-5 hold.
For any $\gamma \in [1/(N+1),1/N)$ and $\delta\in(0,1)$ set $v=(\kappa/L_1)\{2\log (N/(1-\gamma N))\}^{-1/2}$ and $M_t=\log(t^2/\delta) / \log(1/(1-\gamma))$ in Algorithm~\ref{alg:DRTS}.
Then with probability at least $1-8\delta$, the regret bound of {\tt DDRTS-GLM} is bounded by
%\begin{equation}
%\begin{split}
%R(T)\!\le&L_{1}S\Exploration+O\left(\phi^{-3}\kappa^{-3}\log^{2}T\right)\\&+8C_{L_{1}}\sqrt{3L_{1}NT\log\frac{2NT}{\delta}}\!\left(\sqrt{\frac{L_{1}}{\lambda}\log\frac{2}{\delta}}\!+\!\sqrt{d}\right)
%\end{split}
%\label{eq:regret_bound_d}
%\end{equation}
%or
\begin{equation}
\begin{split}
R(T) \le & L_{1}S^{*}\Exploration+O\left(\phi^{-3}\kappa^{-3}\log^{2}T\right)\\
& +(16L_1+32L_1^2)\sqrt{\frac{6T}{\kappa\phi}\log\frac{2NT}{\delta}}.
\end{split}
\label{eq:regret_bound_phi}
\end{equation}
\end{thm}

The term $\Exploration$ represents the number of rounds required for the imputation estimator to satisfy~\eqref{eq:imputation_estimator_condition}.
Since the order of $\Exploration$ is $O(\log T)$ with respect to $T$, the first term in~\eqref{eq:regret_bound_phi} is not the main order term.
%For $\phi$ in the second and third terms, have $\phi^{-1} = O(d)$ in several practical cases such as uniform distribution and truncated multivariate normal distribution on the unit ball of $\Real^{d}$ (See Appendix~\ref{subsec:phi_d_condition} and \citet{kim2021doubly} for details).
For $\phi$ in the second and third terms, Lemma~\ref{lem:phi_d_condition} identifies the cases of $\phi^{-1}=O(d)$. 

\begin{lem}
\label{lem:phi_d_condition}
For $i\in[N]$, let $p_i$ be the density for the marginal distribution of $X_i \in \Real^{d}$.
Suppose that $0 < p_{\min} < p_i(x)$ for all $i\in[N]$ and $x$ such that $\|x\|_2 \le 1$.
Then we have
\[
\Mineigen{\Expectation\left[N^{-1}\sum_{i=1}^{N}X_{i}X_{i}^{T}\right]}\ge  \frac{p_{\min}\text{vol}\left(\mathcal{B}_{d}\right)}{\left(d+2\right)},
\]
where $\mathcal{B}_d$ represents the $l_2$-unit ball in $\Real^{d}$.
\end{lem}
\begin{rem}
When $p_i$ is the uniform density then $\phi^{-1}=d+2$.
For the truncated multivariate normal distribution with mean $0_d$ and covariance $\Sigma$, $\phi^{-1}=\left(d+2\right)\exp\left(\frac{\Mineigen{\Sigma}^{-1}-\Maxeigen{\Sigma}^{-1}}{2}\right)$.
\end{rem}

When there is a lower bound for the marginal density of $X_i$, the main order of the regret bound is $O(\sqrt{\kappa^{-1}dT \log NT})$.
The best known regret bound for GLBs is the $O(\kappa^{-1}\sqrt{d T \log N} \log T)$ regret bound of {\tt SupCB-GLM}, and our bound is improved by $\kappa^{-1/2} \log T$.
To our knowledge, our bound is the best among previously proven bounds for GLB algorithms.  
Furthermore, this is the first regret bound sublinear in $d$ among \texttt{LinTS} variants.
For logistic bandits, our regret bound is comparable with the bound of \texttt{SupLogistic} in terms of $d$.
Even though our bound has extra $\kappa^{-1/2}$, \texttt{SupLogistic} has extra $\log T$. 
If $\log T>\kappa^{-1/2}$, the proposed method has a tighter bound than \texttt{SupLogistic}.
%Because our regret bound requires the condition in Lemma~\ref{lem:phi_d_condition}, the comparisons of the regret bounds are not obvious.
%However, our work identifies a condition where the regret bound can be further improved than previously proven regret bounds.

In the case of linear payoffs, our regret bound is $O(\sqrt{d T \log N T})$, which is tighter than that of previously known linear contextual bandit algorithms.
Although the lower bound $\Omega(\sqrt{dT\log N \log T})$ obtained by \citet{li2019nearly} is larger than our upper bound, this is not a contradiction because the lower bound does not apply to our setting because of Assumptions~\ref{assump:iid_contexts} and~\ref{assump:minimum_eigenvalue}.

\subsection{Key derivations for the regret bound}
In this subsection, we show how the improvement of the regret bound is possible.
For each $i$ and $t$, let $\Diff it := \mu(\Context{\Optimalarm t}{t}^T\beta^{*}) - \mu(\Context{i}{t}^T\beta^{*})$, and 
$\Prederror{i}{t}:=|\mu(\Context{i}{t}^T{\Estimator{t-1}})-\mu(\Context{i}{t}^T\beta^{*})|$.
Denote the weighted Gram matrix by $W_{t}:=\sum_{\tau=1}^{t}\sum_{i=1}^{N} \mu^{\prime}(\Context{i}{\tau}^T\beta^{*}) \Context{i}{\tau}\Context{i}{\tau}^{T}+\lambda I$.
We define a set of super-unsaturated arms at round $t$ as
\begin{equation}
\begin{split}
S_{t}:=\bigg\{ i\in[N]: & \Diff it \le \Prederror{i}t+\Prederror{\Optimalarm t}t \\
&+\sqrt{\kappa L_{1}}\sqrt{\norm{\Context{i}t}_{W_{t-1}^{-1}}^{2}+\norm{\Context{\Optimalarm t}t}_{W_{t-1}^{-1}}^{2}} \bigg\}.
\end{split}
\label{eq:super_unsaturated_arms}
\end{equation}
The following lemma shows that $a_t$ is in $S_t$ with well-controlled $\pi_{a_t,t}$ with high probability.
\begin{lem}
\label{lem:superunsaturated_arms}
Suppose Assumptions~\ref{assump:boundedness} and~\ref{assump:mean_function} hold.
Let $S_t$ be the super-unsaturated arms defined in~\eqref{eq:super_unsaturated_arms}.
For any $\gamma \in [1/(N+1),1/N)$ and $\delta\in(0,1)$ set $v=(\kappa/L_1)\{2\log (N/(1-\gamma N))\}^{-1/2}$ and $M_t=\log(t^2/\delta) / \log(1/(1-\gamma))$ in Algorithm~\ref{alg:DRTS}.
Then the action $\Action{t}$ selected by \texttt{DDRTS-GLM} satisfies
\begin{equation}
\Probability\left(\bigcap_{t=\Exploration}^{T}\left\{ \Action t\in S_{t}\right\} \cap\bigcap_{t=1}^{T}\left\{ \SelectionP{\Action t}t>\gamma\right\} \right)\ge1-\delta.
\label{eq:action_in_superunsaturated_arms}
\end{equation}
\end{lem}
\begin{rem}
%The resampling is required to ensure that $\SelectionP{\Action{t}}{t} > \gamma$ and thus $\Indicator{\Action{t}=i}/\SelectionP{i}{t}$ in~\eqref{eq:pseudo_reward} is bounded for all $i\in[N]$.
The second event in~\eqref{eq:action_in_superunsaturated_arms} helps bound $\Indicator{\Action{t}=i}/\SelectionP{i}{t}$ in the pseudo-reward~\eqref{eq:pseudo_reward}.
\end{rem}

If $\Action{t}$ is in $S_t$, the instantaneous regret is bounded by
\begin{equation}
\begin{split}
\Regret t &= \Diff{\Action{t}}{t} \\
&\le 2\max_{i\in[N]}\left\{ \Prederror it + \sqrt{\kappa L_{1}}\norm{\Context it}_{W_{t-1}^{-1}}\right\}.
\end{split}
\label{eq:regret_decompostion}
\end{equation}
\citet{kim2021doubly} adopted a similar approach but had a different super-unsaturated set, resulting in a different bound for $\Regret t$.
For comparison, in the case of $\mu(x)=x$, \citet{kim2021doubly} proposed DR estimator $\widehat{\beta}_{t}^{DR}$ to derive
\begin{align*}
\Regret{t} \le & 2\norm{\widehat{\beta}_{t-1}^{DR}-\beta^{*}}_2 \\
&+\sqrt{\norm{\Context{\Action{t}}{t}}_{W_{t-1}^{-1}}^2+\norm{\Context{\Optimalarm{t}}{t}}_{W_{t-1}^{-1}}^2},
\end{align*}
and bounded the $l_2$ estimation error by $O(\phi^{-1} t^{-1/2})$ resulting in an $O(d \sqrt{T})$ regret bound in cases when $\phi^{-1}=O(d)$.
This bound has an additional $\sqrt{d}$ compared to~\eqref{eq:regret_decompostion} because $\|\widehat{\beta}^{DR}_{t-1}-\beta^{*}\|_2$ is greater than $\max_{i\in[N]}\Prederror{i}{t}$ because of  Cauchy-Schwartz inequality and using a less accurate imputation estimators.
To obtain faster rates on $d$, we propose the DDR estimator and directly bound the $\Regret t$ with $\Prederror{i}{t}$ as in~\eqref{eq:regret_decompostion} without using Cauchy-Schwartz inequality.
In the following lemma, we show a bound for the prediction error $\Prederror{i}{t}$.
\begin{lem}
\label{lem:prediction_error}
(Prediction error bound for DDR estimator)
Suppose Assumptions 1-5 hold and the event in~\eqref{eq:action_in_superunsaturated_arms} holds.
Then for each $t >\Exploration$, with probability at least $1-8\delta/T$
\begin{equation}
\begin{split}
\Prederror it \le & \mu^{\prime}\left(\Context it^{T}\beta^{*}\right)\left(2+4L_1\right)\sqrt{3N\log\frac{2NT}{\delta}}\norm{\Context it}_{W_{t-1}^{-1}} \\ 
& +\frac{D_{\mu,B,\lambda,S}}{\phi^{3}\kappa^{3}\left(t-1\right)}\log\frac{4T}{\delta}.
\end{split}
\label{eq:prediction_error_bound}
\end{equation}
for all $i\in[N]$, where $D_{\mu,B,\lambda,S}$ is a constant defined in Section~\ref{subsec:proof_of_pred_error}.
\end{lem}

\begin{proof}
For each $t\in(\Exploration,T]$, and $i\in[N]$,
\begin{align*}
\Prederror it \le & \mu^{\prime} \left(\Context it^{T}\beta^{*}\right) \abs{\sum_{\tau=1}^{t-1}\sum_{j=1}^{N}\DRError j\tau{\Impute{t-1}}\Context it^{T}W_{t-1}^{-1}\Context j\tau}\\
&+\frac{O\left(\phi^{-3}\kappa^{-3}\log T\right)}{t-1},
\end{align*}
where $\DRError{j}{\tau}{\beta}:=\DRreward{j}{\tau}{\beta}-\mu(\Context{j}{\tau}^T\beta^{*})$ is the residual for pseudo-rewards.
Now for each $\tau\in[t]$, define the filtration as $\Filtration{\tau}=\History{\tau} \cup \{\Setofcontexts{1}, \ldots, \Setofcontexts{t}\}$ and $\Filtration{0}:=\{\Setofcontexts{1}, \ldots, \Setofcontexts{t}\}$.
Then the random variable $\DRError j{\tau}{\beta}\Context jt^{T}W_{t-1}^{-1}\Context j{\tau}$ is $\mathcal{F}_{\tau+1}$-measurable and
\begin{equation}
\begin{split}
\Expectation\left[\left.\DRError j\tau{\beta}\Context it^{T}W_{t-1}^{-1}\Context j\tau\right|\Filtration \tau\right]
&=\CE{\DRError j\tau{\beta}}{\Filtration \tau}\Context it^{T}W_{t-1}^{-1}\Context j\tau\\
&=\CE{\DRError j\tau{\beta}}{\History \tau}\Context it^{T}W_{t-1}^{-1}\Context j\tau\\
&=0.
\end{split}
\label{eq:equal_to_zero_part}
\end{equation}
The first equality holds since $\Context{i}{t}^{T}W_{t-1}\Context{j}{\tau}$ depends only on $\mathcal{F}_{0}$.
The second equality holds due to the independence between $\DRError{i}{\tau}{\beta}$ and $\{\Setofcontexts{u}\}_{u=\tau+1}^{t}$ induced by Assumption~\ref{assump:iid_contexts}.
The remainder of the proof follows from the Azuma-Hoeffding inequality. 
For details, see Section~\ref{subsec:proof_of_pred_error}.
\end{proof}

%We present the key part of the proof as follows and move the details to Appendix.
%For each $t\in(\Exploration,T]$ and $i\in[N]$,
%exploiting the independence between the Gram matrix and residual terms.  
%The resulting order becomes $\sqrt{\phi^{-1}}$, thus  $\sqrt{d}$.

We highlight that the key part of the proof is in \eqref{eq:equal_to_zero_part} which holds because the Gram matrix $W_{t-1}$ contains all contexts.
Let $A_t:=\sum_{\tau=1}^{t}\mu^{\prime}(\Context{\Action{\tau}}{\tau}^{T}\beta^{*})\Context{\Action{\tau}}{\tau}\Context{\Action{\tau}}{\tau}^{T}+I_{d}$ be the Gram matrix consists of the selected contexts only and $\Error{\tau}:=\Reward{\tau}-\mu(\Context{\Action{\tau}}{\tau}^T\beta^{*})$ be the error.  
In general,
\begin{align*}
\CE{\Error{\tau}\Context it^{T}A_{t-1}^{-1}\Context{\Action{\tau}}{\tau}}{\History{\tau}}
\!\neq\!\CE{\Error{\tau}}{\History{\tau}}\Context it^{T}A_{t-1}^{-1}\Context{\Action{\tau}}{\tau},
\end{align*}
due to the dependency between $A_{t-1}$ and $\Error{\tau}$ through the actions $\{\Action{\tau}, \Action{\tau+1}, \ldots, \Action{t-1}\}$.
To avoid this dependency, \citet{li2017provably} and \citet{junimproved2021} used the approach of \citet{auer2002using} by devising $\Action{\tau}$ to be independent of $\Error{\tau}$. 
Based on this independence, they invoked the equality analogous to the first equality in~\eqref{eq:equal_to_zero_part}, achieving an $\tilde{O}(\sqrt{dT})$ regret bound.
%, even though they incurred a $\log N$ term by repeating the same procedure for all arms. 
However, the crafted independence negatively affects inefficiency by ignoring dependent samples during estimation.
In contrast, we achieve independence using all contexts, generating the Gram matrix free of $\Action{1}, \ldots, \Action{T}$.

From~\eqref{eq:regret_decompostion} and~\eqref{eq:prediction_error_bound}, the $\Regret{t}$ is bounded by,
\begin{equation}
\begin{split}
    &\Regret t \\
    & \le  2\left\{ \left(2+4L_1\right)w_{i,t}\sqrt{3N\log\frac{2NT}{\delta}}\!\!+\!\sqrt{\kappa L_{1}}\right\} \max_{i\in[N]}s_{i,t} \\
    &\;\; + \frac{D_{\mu,B,\lambda,S}}{\phi^{3}\kappa^{3}\left(t-1\right)}\log\frac{4T}{\delta}
    \\
    &\le4\left(2+4L_1\right)\sqrt{3L_{1}N\log\frac{2NT}{\delta}}\max_{i\in[N]}\sqrt{w_{i,t}}s_{i,t} \\
    &\;\;+ \frac{D_{\mu,B,\lambda,S}}{\phi^{3}\kappa^{3}\left(t-1\right)}\log\frac{4T}{\delta},
\end{split}
\label{eq:regret_decompostion_norms}
\end{equation}
where $w_{i,t}:=\mu^{\prime}(\Context{i}{t}^T\beta^*)$, and $s_{i,t}\!:=\|\Context{i}{t}\|_{W_{t-1}^{-1}}$.
Now we need a bound for the weighted sum of $s_{i,t}$ over $t\in(\Exploration,T]$.
In this point, many existing regret analyses have used a version of the \textit{elliptical potential lemma}, i.e., Lemma 11 in \citet{abbasi2011improved}, yielding an $O(\sqrt{dT \log T})$ bound for $\sum_{t=1}^{T} \|\Context{\Action{t}}{t}\|_{A_{t-1}^{-1}}$.  %\citep{auer2002using,filippi2010parametric,abbasi2011improved,chu2011contextual,li2017provably}.
This bound cannot be applied to our case because we have different Gram matrix composed of the contexts of all arms. 
Therefore we develop the following lemma to prove a bound for the weighted sum.
%n $O(\sqrt{\phi^{-1}\kappa^{-1}N^{-1}T})$

\begin{lem}
(Maximal elliptical potential lemma)
Suppose Assumptions 1-5 hold.
Set $w_{i,t}:=\mu^{\prime}(\Context{i}{t}^T\beta^{*})$, and $s_{i,t}:=\|\Context{i}{t}\|_{W_{t-1}^{-1}}$.
Then with probability at least $1-\delta$,
%\begin{equation}
%    \max_{i\in[N]}\sqrt{w_{i,t}}s_{i,t} \le \sqrt{\frac{2L_1}{\phi \kappa N (t-1)}}
%    \label{eq:singleton_bound_phi}
%\end{equation}
%for all $t\in(\Exploration,T]$ and thus,
\begin{equation}
    \sum_{t=\mathcal{\Exploration}}^{T}\max_{i\in[N]}\sqrt{w_{i,t}}s_{i,t}\le\sqrt{\frac{8L_{1}T}{\phi N\kappa}}.
    \label{eq:elliptical_bound_phi}
\end{equation}
\label{lem:elliptical_potential_lemma}
\end{lem}

Now the regret bound of \texttt{DDRTS-GLM} is proved by applying~\eqref{eq:elliptical_bound_phi} to~\eqref{eq:regret_decompostion_norms}.

%------------------------------------------
%Simulation Figures
%------------------------------------------
\begin{figure*}[t]
\centering
\begin{tabular}{cc}
\includegraphics[width=0.45\textwidth]{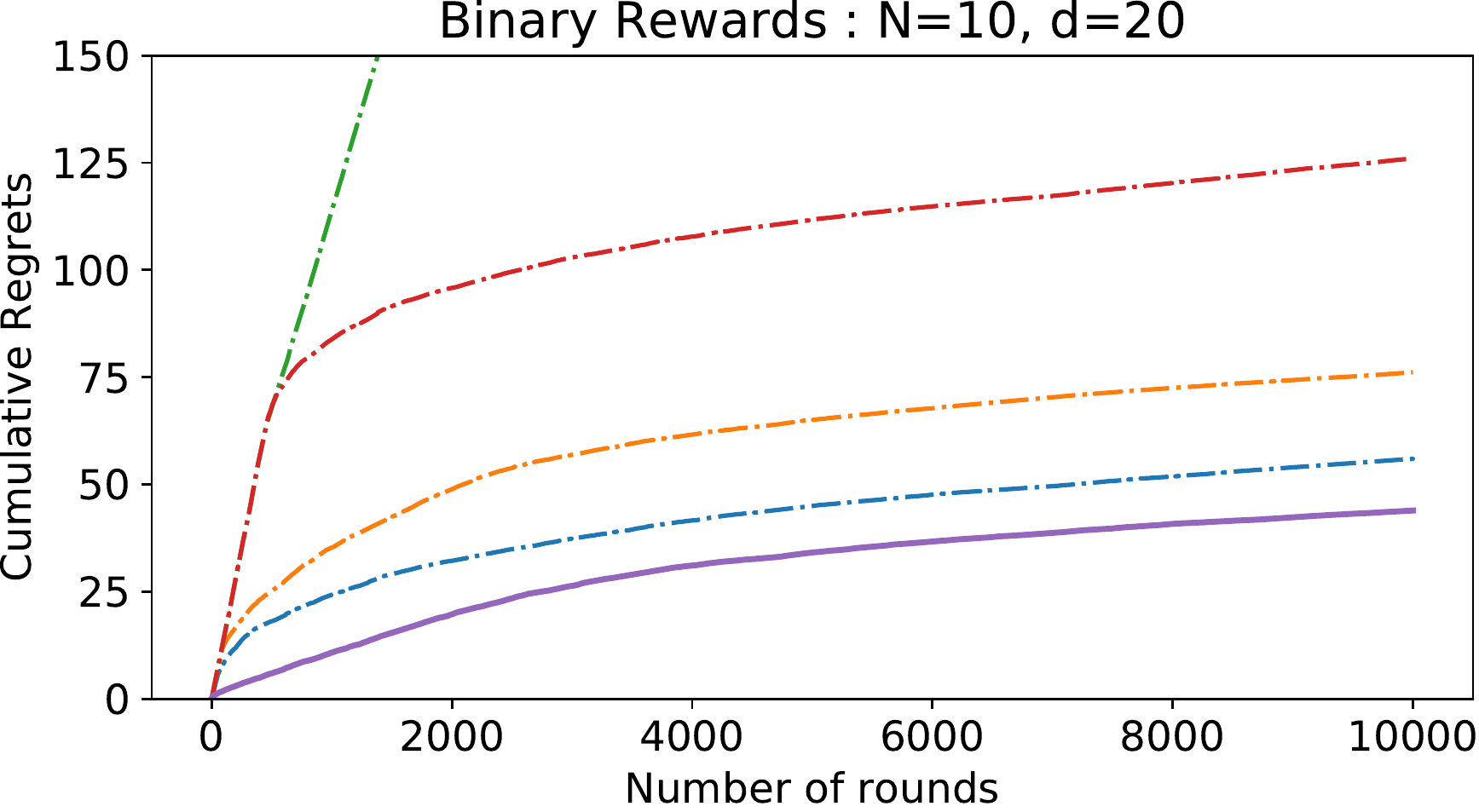}& \includegraphics[width=0.45\textwidth]{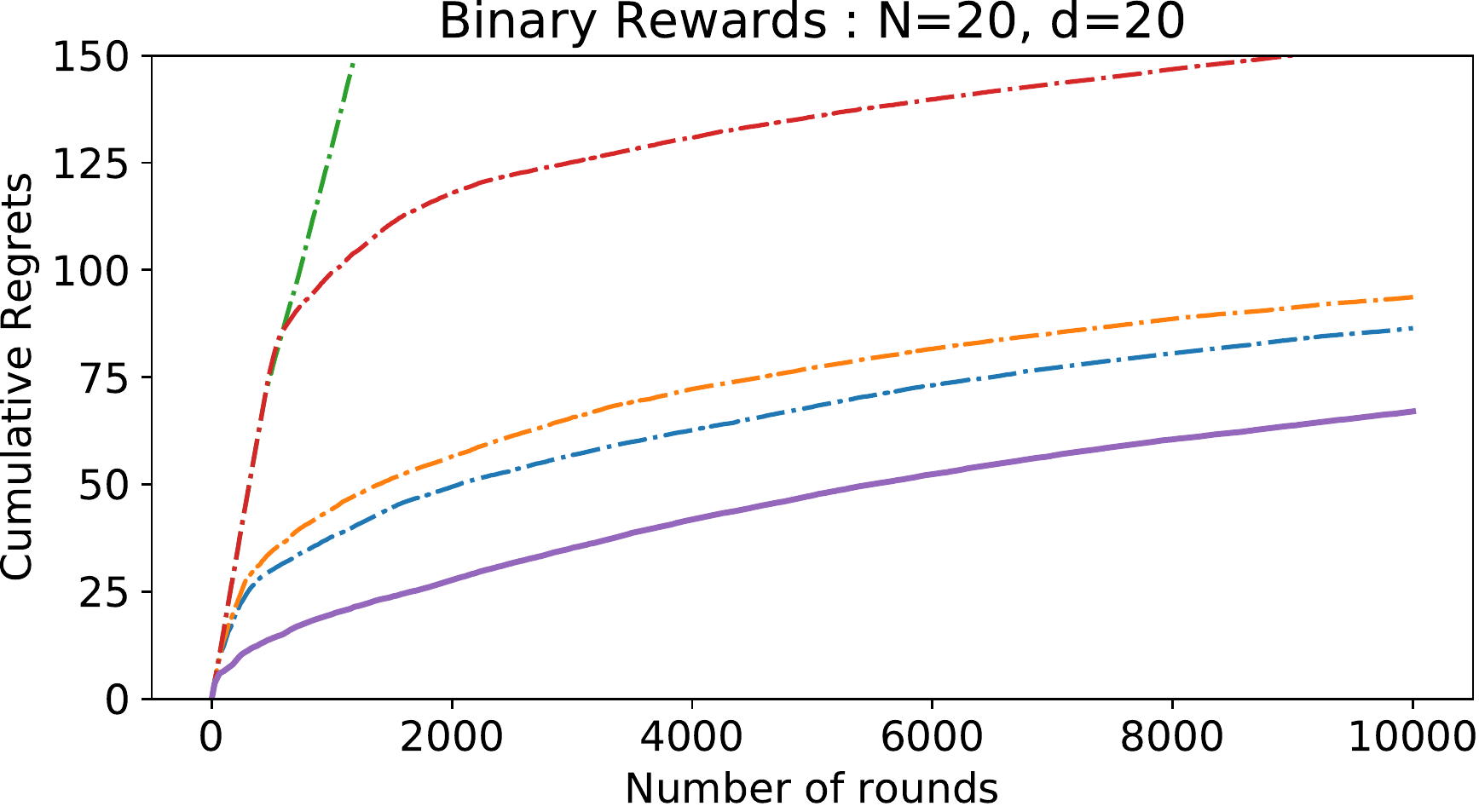}\\ \includegraphics[width=0.45\textwidth]{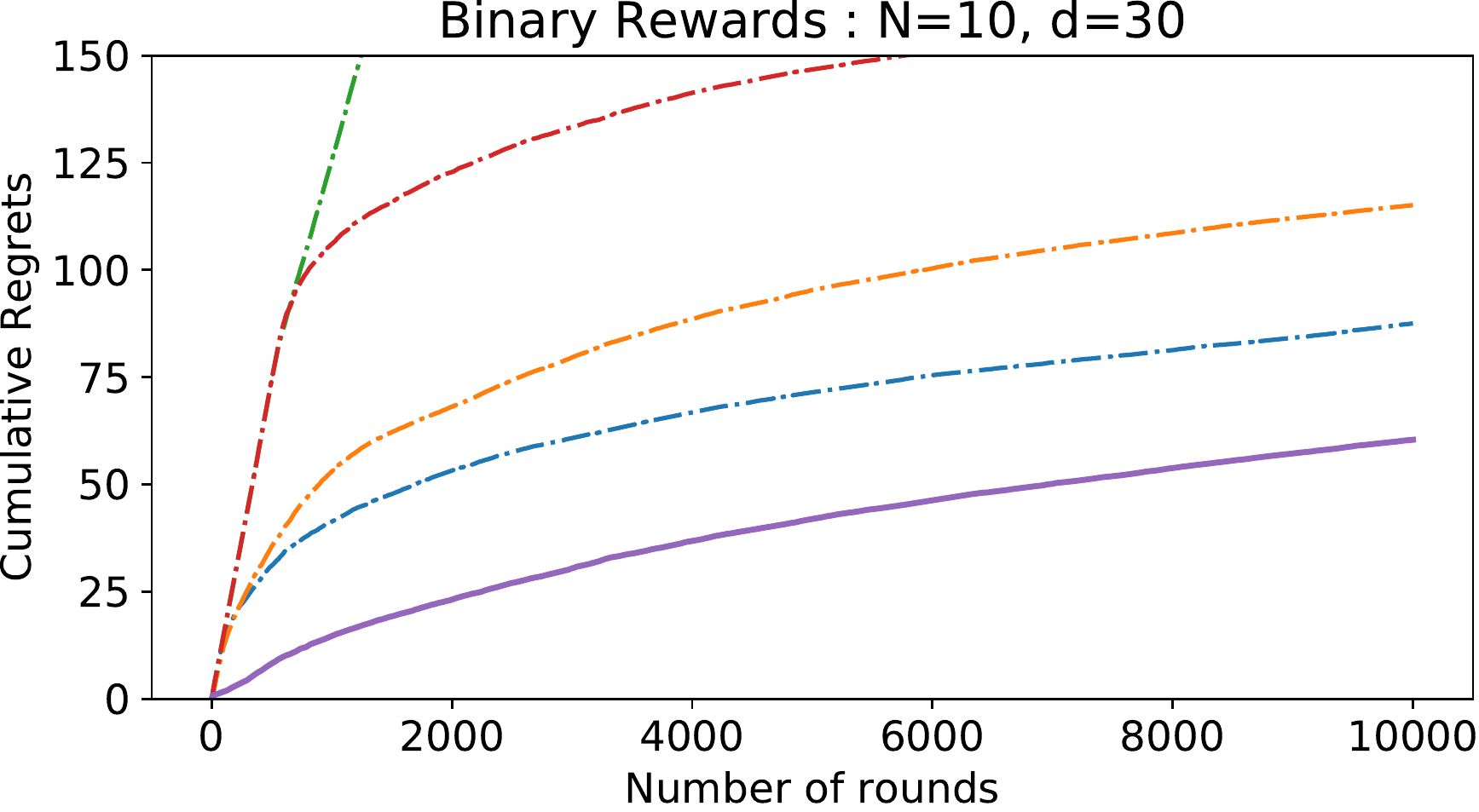}& \includegraphics[width=0.45\textwidth]{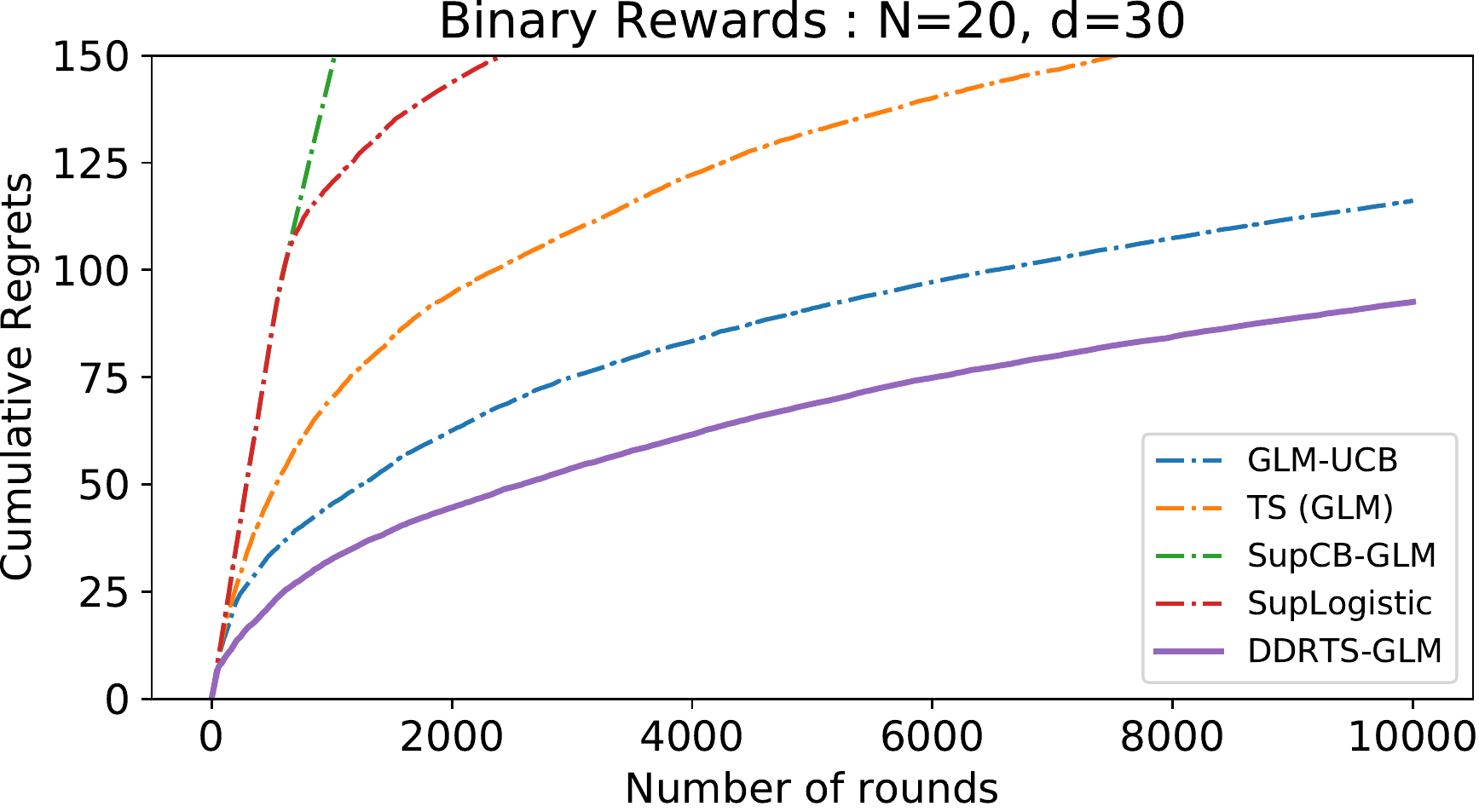}
\end{tabular}
\caption{Comparison of the average cumulative regret on synthetic dataset over 5 repeated runs with $T=10000$.}% 4 pairs of $N$ and $d$.} 
\label{fig_sim}
\end{figure*}

%---------------------------------------
% Section 5. A tighter regret bound under margin condition
%---------------------------------------

\subsection{A logarithmic cumulative regret bound under a margin condition}
\label{subsec:fast_regret_bound}
In this subsection, we present an $O(\kappa^{-1}d \log NT \log T)$ regret bound for \texttt{DDRTS-GLM} under the margin condition stated as follows.

\begin{assumption}[Margin condition]
\label{assum:margin_condition}
For all $t$, there exist \textit{unknown} $\rho_{0}>0$ and $h\ge0$ such that
\begin{equation}
    \Probability\left(\mu(\Context{\Optimalarm{t}}t^{T}\beta^{*})\le\max_{j\neq \Optimalarm{t}}\mu(\Context jt^{T}\beta^{*})+\rho\right)\le h\rho,
    \label{eq:margin_condition}
\end{equation}
for all $\rho\in(0,\rho_{0}]$.
\end{assumption}

This margin condition guarantees a probabilistic positive gap between the expected rewards of the optimal arm and the other arms.
In the margin condition, $h>0$ represents how heavy the tail probability is for the margin.
\citet{bastani2020online} and \citet{bastani2021mostly} adopted this assumption and proved $O(\log T)$ regret bounds with contextual bandit problems when contexts are the same for all arms and coefficients are arm-specific.
When coefficients are the same for all arms and contexts are arm-specific, \citet{dani2008stochastic,abbasi2011improved} and \citet{selfcon2021} used a {\it deterministic} margin condition, which is a special case of Assumption~\ref{assum:margin_condition} when $h=0$.
Now we show that {\tt DDRTS-GLM} has a logarithmic cumulative regret bound under the margin condition.

\begin{thm}
\label{thm:fast_regret_bound} 
Suppose Assumptions 1-6 hold.
Then with probability at least $1-10\delta$, the cumulative regret of \texttt{DDRTS-GLM} is bounded by 
\begin{equation}
\begin{split}
R(T)\le & 2L_{1}S^{*}\mathcal{T}_{0}+O\left(\phi^{-2}\kappa^{-2}\log NT\right)     \\
& +\frac{192hL_{1}^{2}(2+4L_1)^{2}}{\kappa\phi}\log T\log\frac{2NT}{\delta}
\end{split}
\label{eq:fast_regret_bound}
\end{equation}
for $\mathcal{T}_{0}=\rho_{0}^{-2}\Omega\left(\phi^{-4}\kappa^{-4}\log NT\right)$.
\end{thm}

Since $\mathcal{T}_{0}$ has only $\log NT$, the first term is not the main order term.
In the practical cases of $\phi^{-1}\!=\!O(d)$ (see Lemma~\ref{lem:phi_d_condition}), the order of the regret bound is $O(d\kappa^{-1} \log NT \log T)$.
To our knowledge, our work is the first to derive a logarithmic cumulative regret bound for \texttt{LinTS} variants.
We defer the challenges and the intuition deriving the regret bound~\eqref{eq:fast_regret_bound} to Appendix~\ref{sec:fast_regret_bound_proof}.

%------------------------------------------
% Experiment Results
%------------------------------------------
\section{Experiment results}
\label{sec:experiments}

In this section, we compare the performance of the five
algorithms: (i) {\tt GLM-UCB} \citep{filippi2010parametric}, (ii) {\tt TS (GLM)}  \citep{abeille2017linear}, (iii) {\tt SupCB-GLM} \citep{li2017provably}, (iv) {\tt SupLogistic} \citep{junimproved2021}, and (v) the proposed {\tt DDRTS-GLM} using simulation data (Section~\ref{subsec:simluation_data}) and the two real datasets (Section~\ref{subsec:forest_cover} and \ref{subsec:news}).

\subsection{Simulation data}
\label{subsec:simluation_data}

To generate data, we set the number of arms as $N=10$ and $20$ and the dimension of contexts as $d=20$ and $30$.
For $j\in[d]$, the $j$-th elements of the contexts, $[X_{1,t}^{(j)}, \dots, X_{N,t}^{(j)}]$ are sampled from the normal distribution $\mathcal{N}(\mu_N, V_N)$ with mean $\mu_{10} = [-5, -4, \ldots , -1, 1, \ldots , 4, 5]^T$, and $\mu_{20} = [-10, -9, \ldots , -1, 1, \ldots , 9, 10]^T$.
The covariance matrix $V_N \in \mathbb{R}^{N \times N}$ has $V(i,i)=1$ for every $i$ and $V(i,k)=0.5$ for every $i \neq k$. 
The sampled contexts are truncated to satisfy $\norm{X_i(t)}_2 \le 1$.
For rewards, we sample $Y_t$ independently from Ber$(\mu(\Context{\Action{t}}{t}^T \beta^{*}))$, where $\mu(x):=1/(1+e^{-x})$. 
Each element of $\beta^{*}$ follows a uniform distribution, $\mathcal{U}(-1,1)$. 
%We have run 5 simulations with $T=10000$ for each case.

As hyperparameters of the algorithms, {\tt GLM-UCB}, {\tt SupCB-GLM}, and {\tt SupLogistic}  have $\alpha$ as an exploration parameter. 
For \texttt{TS(GLM)} and the proposed method, $v$ controls the variance of $\tilde{\beta}_i(t)$. 
In each algorithm, we choose the best hyperparameter from $\{0.001,0.01,0.1,1\}$.
The proposed method requires a positive threshold $\gamma$ for resampling; however, we do not tune $\gamma$ but fix the value to be $1/(N+1)$.
Figure~\ref{fig_sim} shows the mean cumulative regret $R(T)$, and the proposed algorithm, as represented by the solid line, outperforms four other candidates in all four scenarios.

\subsection{Forest Cover Type dataset}
\label{subsec:forest_cover}

%----------------------------------------------
% Forest Cover figure
%----------------------------------------------
\begin{figure}
    \centering
    \includegraphics[width=0.41\textwidth]{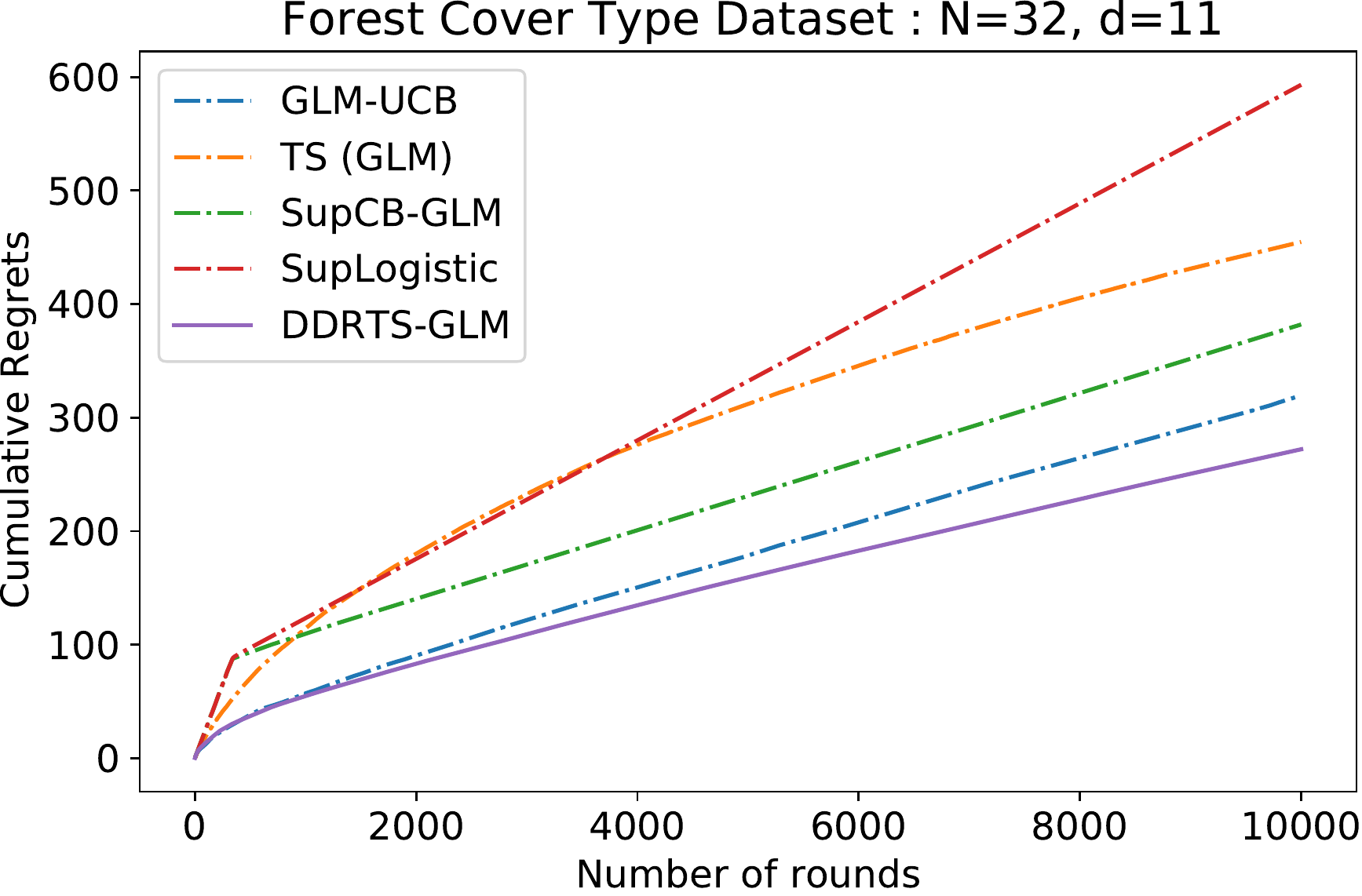}
    \caption{Comparison of the average cumulative regret on forest cover type dataset over 10 repeated runs.}
    \label{fig_forest}
\end{figure}

We use the Forest Cover Type dataset from the UCI Machine Learning repository \citep{blake1999uci}, as used by \citet{filippi2010parametric}.  
The dataset contains 581,021 observations, where the response variable is the label of the dominant species of trees of each region, and covariates include ten continuous cartographic variables characterizing features of the forest.
We divide the dataset into $32$ clusters using the k-means clustering algorithm, and the resulting clusters of the forest represent arms.  We repeat the experiment 10 times with $T=10,000$.  
Each centroid of each cluster is set to be a 10-dimensional context vector of the corresponding arm, and by introducing an intercept, we obtain $d=11$. 
In this example, context vectors remain unchanged in each round. We dichotomize the reward based on whether the forest's dominant class is Spruce/Fir. The goal is to find arms with Spruce/Fir as the dominant class. 
We execute a 32-armed 11-dimensional contextual bandit with binary rewards. 
Figure~\ref{fig_forest} shows that {\tt DDRTS-GLM} outperforms other algorithms.

%-----------------------------------
% CTR table
%-----------------------------------
\begin{table}[t]
\centering
\caption{Average/first quartile/third quartile of CTR of news articles over 10 repeated runs for each algorithm.}
\label{table_news}
\begin{tabular}{|c|c|c|c|}
\hline
CTR     & 1st Q & average & 3rd Q \\ \hline
\texttt{DDRTS-GLM} & \textbf{0.0410}       & \textbf{0.0449}     & \textbf{0.0476}     \\ \hline
\texttt{GLM-UCB}  & 0.0356       & 0.0420     & 0.0468     \\ \hline
\texttt{TS (GLM)}   & 0.0393       & 0.0438     & 0.0467     \\ \hline
uniform random   & 0.0337       & 0.0344     & 0.0351     \\ \hline
\end{tabular}
\end{table}

\subsection{Yahoo! news article recommendation log data}
\label{subsec:news}
%We compare \texttt{GLM-UCB}, \texttt{TS (GLM)} and our proposed method using the Yahoo! Front Page Today Module User Click Log Dataset \citep{yahooR6A}.
The Yahoo! Front Page Today Module User Click Log Dataset \citep{yahooR6A} contains 45,811,883 user click logs for news articles on Yahoo! Front Page from May 1st, 2009, to May 19th, 2009.% and contains 45,811,883 user visits. 
Each log consists of a randomly chosen article from $N=20$ articles and the binary reward $\Reward{t}$ which takes the value 1 if a user clicked the article and 0 otherwise.
For each article $i\in[20]$, we have a context vector $\Context{i}{t} \in \Real^{11}$ which comprising 10 extracted features of user-article information and an intercept using a dimension reduction method, as in \citet{chu2009case}.

%as in \citet{chu2009case}
%For each visit, a randomly chosen article was uniformly randomly chosen from 20 articles and was displayed in the featured tab of the today module on Yahoo! front page. 
%The binary reward $\Reward{t}$ takes the value 1 if the user clicked the $\Action{t}$-th article and 0 otherwise. 
%For each article $i\in[20]$, we have a context vector $\Context{i}{t} \in \Real^{11}$ which constitutes 10 extracted features of user-article information and a constant. 
%The extracted features were constructed using a dimension reduction method as in \citet{chu2009case}.
%Evaluation of a decision-making algorithm using log data is difficult,

In log data, when the algorithm chooses an action not selected by the original logger, it cannot observe the reward, and the regret is not computable.
Instead, we use the click-through rate (CTR): the percentage of the number of clicks. 
We tune the hyperparameters using the log data from May 1st and run the algorithms on the randomly sampled $10^6$ logs in each run.
We evaluate the algorithms on the basis of the method by \citet{li2011unbiased}, counting only the rounds in which the reward is observed in $T$.
Thus, $T$ is not known a priori and {\tt SupCB-GLM} and {\tt SupLogistic} are not applicable.
As a baseline, we run a uniform random policy to observe the CTR lift of each algorithm.
Table~\ref{table_news} presents the average/first quartile/third quartile of CTR of each algorithm (the higher is the  better) over 10 repetitions, showing that {\tt DDRTS-GLM} achieves the highest CTR.

\section*{Acknowledgments}
This work is supported by the National Research Foundation of Korea (NRF) grant funded by the Korea government (MSIT, No.2020R1A2C1A01011950).

%----------------------------------------------------
% Bibiliography
%----------------------------------------------------
\bibliography{aaai23}

\begin{thebibliography}{29}
\providecommand{\natexlab}[1]{#1}

\bibitem[{Abbasi-Yadkori, P{\'a}l, and
  Szepesv{\'a}ri(2011)}]{abbasi2011improved}
Abbasi-Yadkori, Y.; P{\'a}l, D.; and Szepesv{\'a}ri, C. 2011.
\newblock Improved algorithms for linear stochastic bandits.
\newblock In \emph{Advances in Neural Information Processing Systems},
  2312--2320.

\bibitem[{Abeille and Lazaric(2017)}]{abeille2017linear}
Abeille, M.; and Lazaric, A. 2017.
\newblock Linear thompson sampling revisited.
\newblock In \emph{Artificial Intelligence and Statistics}, 176--184. PMLR.

\bibitem[{Agrawal and Goyal(2013)}]{agrawal2013thompson}
Agrawal, S.; and Goyal, N. 2013.
\newblock Thompson sampling for contextual bandits with linear payoffs.
\newblock In \emph{International Conference on Machine Learning}, 127--135.

\bibitem[{Agrawal and Goyal(2014)}]{agrawal2014thompson}
Agrawal, S.; and Goyal, N. 2014.
\newblock Thompson Sampling for Contextual Bandits with Linear Payoffs.
\newblock arXiv:1209.3352.

\bibitem[{Auer(2002)}]{auer2002using}
Auer, P. 2002.
\newblock Using confidence bounds for exploitation-exploration trade-offs.
\newblock \emph{Journal of Machine Learning Research}, 3(Nov): 397--422.

\bibitem[{Azuma(1967)}]{azuma1967weighted}
Azuma, K. 1967.
\newblock Weighted sums of certain dependent random variables.
\newblock \emph{Tohoku Mathematical Journal, Second Series}, 19(3): 357--367.

\bibitem[{Bastani and Bayati(2020)}]{bastani2020online}
Bastani, H.; and Bayati, M. 2020.
\newblock Online decision making with high-dimensional covariates.
\newblock \emph{Operations Research}, 68(1): 276--294.

\bibitem[{Bastani, Bayati, and Khosravi(2021)}]{bastani2021mostly}
Bastani, H.; Bayati, M.; and Khosravi, K. 2021.
\newblock Mostly exploration-free algorithms for contextual bandits.
\newblock \emph{Management Science}, 67(3): 1329--1349.

\bibitem[{Blake, Keogh, and Merz(1999)}]{blake1999uci}
Blake, C.; Keogh, E.; and Merz, C. 1999.
\newblock UCI repository of machine learning databases (Machinereadable data
  repository). Irvine, CA: Department of Information and Computer Science,
  University of California at Irvine.

\bibitem[{Chen et~al.(1999)Chen, Hu, Ying et~al.}]{chen1999strong}
Chen, K.; Hu, I.; Ying, Z.; et~al. 1999.
\newblock Strong consistency of maximum quasi-likelihood estimators in
  generalized linear models with fixed and adaptive designs.
\newblock \emph{The Annals of Statistics}, 27(4): 1155--1163.

\bibitem[{Chu et~al.(2011)Chu, Li, Reyzin, and Schapire}]{chu2011contextual}
Chu, W.; Li, L.; Reyzin, L.; and Schapire, R. 2011.
\newblock Contextual bandits with linear payoff functions.
\newblock In \emph{Proceedings of the Fourteenth International Conference on
  Artificial Intelligence and Statistics}, 208--214.

\bibitem[{Chu et~al.(2009)Chu, Park, Beaupre, Motgi, Phadke, Chakraborty, and
  Zachariah}]{chu2009case}
Chu, W.; Park, S.-T.; Beaupre, T.; Motgi, N.; Phadke, A.; Chakraborty, S.; and
  Zachariah, J. 2009.
\newblock A case study of behavior-driven conjoint analysis on Yahoo! Front
  Page Today module.
\newblock In \emph{Proceedings of the 15th ACM SIGKDD international conference
  on Knowledge discovery and data mining}, 1097--1104.

\bibitem[{Chung and Lu(2006)}]{chung2006concentration}
Chung, F.; and Lu, L. 2006.
\newblock Concentration inequalities and martingale inequalities: a survey.
\newblock \emph{Internet Mathematics}, 3(1): 79--127.

\bibitem[{Dani, Hayes, and Kakade(2008)}]{dani2008stochastic}
Dani, V.; Hayes, T.; and Kakade, S. 2008.
\newblock Stochastic linear optimization under bandit feedback.
\newblock In \emph{21st Annual Conference on Learning Theory}, 355--366.

\bibitem[{Dimakopoulou et~al.(2019)Dimakopoulou, Zhou, Athey, and
  Imbens}]{dimakopoulou2019balanced}
Dimakopoulou, M.; Zhou, Z.; Athey, S.; and Imbens, G. 2019.
\newblock Balanced linear contextual bandits.
\newblock In \emph{Proceedings of the AAAI Conference on Artificial
  Intelligence}, volume~33, 3445--3453.

\bibitem[{Faury et~al.(2020)Faury, Abeille, Calauz{\`e}nes, and
  Fercoq}]{faury2020improved}
Faury, L.; Abeille, M.; Calauz{\`e}nes, C.; and Fercoq, O. 2020.
\newblock Improved optimistic algorithms for logistic bandits.
\newblock In \emph{International Conference on Machine Learning}, 3052--3060.
  PMLR.

\bibitem[{Filippi et~al.(2010)Filippi, Cappe, Garivier, and
  Szepesv{\'a}ri}]{filippi2010parametric}
Filippi, S.; Cappe, O.; Garivier, A.; and Szepesv{\'a}ri, C. 2010.
\newblock Parametric bandits: The generalized linear case.
\newblock In \emph{Advances in Neural Information Processing Systems},
  586--594.

\bibitem[{Freedman(1975)}]{freedman1975tail}
Freedman, D.~A. 1975.
\newblock On tail probabilities for martingales.
\newblock \emph{the Annals of Probability}, 100--118.

\bibitem[{Jun et~al.(2017)Jun, Bhargava, Nowak, and Willett}]{jun2017scalable}
Jun, K.-S.; Bhargava, A.; Nowak, R.; and Willett, R. 2017.
\newblock Scalable generalized linear bandits: Online computation and hashing.
\newblock In \emph{Advances in Neural Information Processing Systems}, 99--109.

\bibitem[{Jun et~al.(2021)Jun, Jain, Mason, and Nassif}]{junimproved2021}
Jun, K.-S.; Jain, L.; Mason, B.; and Nassif, H. 2021.
\newblock Improved Confidence Bounds for the Linear Logistic Model and
  Applications to Bandits.
\newblock In Meila, M.; and Zhang, T., eds., \emph{Proceedings of the 38th
  International Conference on Machine Learning}, volume 139 of
  \emph{Proceedings of Machine Learning Research}, 5148--5157. PMLR.

\bibitem[{Kim and Paik(2019)}]{kim2019doubly}
Kim, G.-S.; and Paik, M.~C. 2019.
\newblock Doubly-Robust Lasso Bandit.
\newblock In \emph{Advances in Neural Information Processing Systems},
  5869--5879.

\bibitem[{Kim, Kim, and Paik(2021)}]{kim2021doubly}
Kim, W.; Kim, G.-S.; and Paik, M.~C. 2021.
\newblock Doubly Robust Thompson Sampling with Linear Payoffs.
\newblock In Beygelzimer, A.; Dauphin, Y.; Liang, P.; and Vaughan, J.~W., eds.,
  \emph{Advances in Neural Information Processing Systems}.

\bibitem[{Lee, Peres, and Smart(2016)}]{lee2016}
Lee, J.~R.; Peres, Y.; and Smart, C.~K. 2016.
\newblock A Gaussian upper bound for martingale small-ball probabilities.
\newblock \emph{Ann. Probab.}, 44(6): 4184--4197.

\bibitem[{Li et~al.(2011)Li, Chu, Langford, and Wang}]{li2011unbiased}
Li, L.; Chu, W.; Langford, J.; and Wang, X. 2011.
\newblock Unbiased offline evaluation of contextual-bandit-based news article
  recommendation algorithms.
\newblock In \emph{Proceedings of the fourth ACM international conference on
  Web search and data mining}, 297--306.

\bibitem[{Li, Lu, and Zhou(2017)}]{li2017provably}
Li, L.; Lu, Y.; and Zhou, D. 2017.
\newblock Provably optimal algorithms for generalized linear contextual
  bandits.
\newblock In \emph{Proceedings of the 34th International Conference on Machine
  Learning-Volume 70}, 2071--2080.

\bibitem[{Li, Wang, and Zhou(2019)}]{li2019nearly}
Li, Y.; Wang, Y.; and Zhou, Y. 2019.
\newblock Nearly minimax-optimal regret for linearly parameterized bandits.
\newblock In \emph{Conference on Learning Theory}, 2173--2174. PMLR.

\bibitem[{Russac et~al.(2021)Russac, Faury, Capp{\'e}, and
  Garivier}]{selfcon2021}
Russac, Y.; Faury, L.; Capp{\'e}, O.; and Garivier, A. 2021.
\newblock Self-Concordant Analysis of Generalized Linear Bandits with
  Forgetting.
\newblock In Banerjee, A.; and Fukumizu, K., eds., \emph{Proceedings of The
  24th International Conference on Artificial Intelligence and Statistics},
  volume 130 of \emph{Proceedings of Machine Learning Research}, 658--666.
  PMLR.

\bibitem[{Segerstedt(1992)}]{segerstedt1992ordinary}
Segerstedt, B. 1992.
\newblock On ordinary ridge regression in generalized linear models.
\newblock \emph{Communications in Statistics-Theory and Methods}, 21(8):
  2227--2246.

\bibitem[{{Yahoo! Webscope}(2009)}]{yahooR6A}
{Yahoo! Webscope}. 2009.
\newblock Yahoo! Front Page Today Module User Click Log Dataset, version 1.0.
\newblock Accessed 07-April-2021.

\end{thebibliography}

\newpage
%----------------------------------------------------
% Appendix
%----------------------------------------------------
\appendix
\onecolumn

\section{Generalized linear models}
\label{sec:GLM}
We consider a model in which the reward $Y$ is sampled from one-dimensional exponential family. 
When the reward is related to a context vector $X\in\Real^{d}$, the GLM assumes that the probability density function of $Y$ is 
\[
\CP{Y=y}{\theta} \propto \exp\left\{ \theta y-b(\theta)\right\},
\]
where $b$ represents a known function and $\theta\in\Real$ is a canonical parameter related to both the reward and the context vector. 
When $b$ is twice differentiable, we have $b^{\prime}(\theta)=\Expectation[Y]$, and $b^{\prime \prime}(\theta)=\Var[Y]$. 
In the GLM, we assume that the expected value of $Y$ is related to context $X$ based on linear predictor, $\xi:=X^{T}\beta^*$ for some unknown parameter $\beta^{*}\in\Real^{d}$.
Then $\Expectation[Y]:=\mu(\xi)$, with the inverse link function $\mu$.
When $\theta=\xi$, the link function is called canonical link.  
Under a canonical link, $b^{\prime \prime}(\theta)\!=\!\mu^{\prime}(\xi)$.
In the case of Bernoulli rewards, the inverse canonical link function is $\mu(\xi)\!=\!1/(1\!+\!\exp(-\xi))$. 
Throughout this study, we consider a canonical link case. 
An extension to non-canonical cases requires minor changes.

Let $(X_{1},Y_{1}),\ldots,(X_{t},Y_{t}$) be the independent pairs of contexts and rewards. 
Then, the log-likelihood function in a canonical link case is 
\[
    \log l(\beta)= \sum_{\tau=1}^{t}\left\{ Y_{\tau}X_{\tau}^{T}\beta-b\left(X_{\tau}^{T}\beta\right)\right\} +C\left(Y_{\tau}\right),
\]
where $C(\cdot)$ is a function not related to $\beta$.
Instead of minimizing $-\log l(\beta)$ to obtain an estimator, \citet{segerstedt1992ordinary} proposed adding $\frac{\lambda_{t}}{2}\beta^{T}\beta$ to $-\log l(\beta)$ for regularization.
The regularization term plays a significant role in ensuring that the second derivative of $-\log l(\beta)$,
\[
\sum_{\tau=1}^{t}b^{\prime \prime}\left(X_{\tau}^{T}\beta\right)X_{\tau}X_{\tau}^{T}+\lambda_{t}I,
\]
has a sufficiently large positive eigenvalue. 
Since the objective function is strictly convex, we can obtain a ridge-type estimator by taking the derivative with respect to $\beta$ and solving $\sum_{\tau=1}^{t}\left\{ Y_{\tau}X_{\tau}-\mu\left(X_{\tau}^{T}\beta\right)X_{\tau}\right\} -\lambda_{t}\beta=0.$
In practice, this equation can be solved using Newton's algorithms.

\section{Comparison of the order of regret bounds for generalized/logistic bandits}
\label{sec:full_table}
In this section, we provide a table that compares the regret bounds for generalized/logistic bandit algorithms, with respect to $\kappa$, $d$, $N$ and $T$. 

\begin{table}[h]
\begin{center}
\caption{Comparison of the main orders of the regret bounds for generalized/logistic bandit with respect to $\kappa, d,N$ and $T$.}
\begin{tabular}{|c|c|c|}
\hline
\textbf{Algorithm} & \textbf{Regret Upper Bound}    &   \textbf{Model} \\ \hline
\texttt{GLM-UCB} \citep{filippi2010parametric} & $O(\kappa^{-1} d \sqrt{T} \log^{3/2} T)$ & GLM \\ \hline
\texttt{SupCB-GLM} \citep{li2017provably} & $O(\kappa^{-1}\sqrt{d T \log NT} \log T$)  & GLM \\ \hline
\texttt{TS (GLM)} \citep{abeille2017linear} & $O(\kappa^{-1}d^{3/2}\sqrt{T}\log T)$ & GLM  \\ \hline
\texttt{GLOC} \citep{jun2017scalable} & $O(\kappa^{-1}d\sqrt{T}\log^{3/2} T)$ & GLM  \\ \hline
\texttt{GLOC-TS} \citep{jun2017scalable} & $\hat{\mathcal{O}}(\kappa^{-1}d^{3/2}\sqrt{T}\log^{3/2} T)$ & GLM  \\ \hline
\texttt{Logistic UCB-1} \citep{faury2020improved} & $O(\kappa^{-1/2}d\sqrt{T}\log T)$ & Logistic \\ \hline
\texttt{Logistic-UCB-2} \citep{faury2020improved} & $O(d \sqrt{T} \log T)$ & Logistic  \\ \hline
\texttt{SupLogistic} \citep{junimproved2021}  & $O(\sqrt{dT \log NT} \log T)$ & Logistic  \\ \hline
\texttt{DDRTS-GLM} (Proposed)  & $O(\sqrt{\kappa^{-1} d T \log NT})$ & GLM  \\ \hline
\end{tabular}
\end{center}
\end{table}

\section{Detailed proofs for Theorem~\ref{thm:regret_bound} }

In this section, we present the proof for the $\tilde{O}(\sqrt{\kappa^{-1} \phi T})$ regret bound for \texttt{DDRTS-GLM} and its related lemmas.

\subsection{Proof of Lemma~\ref{lem:superunsaturated_arms}}
\begin{proof}
This proof is modified from that of Lemma 2 in \citet{kim2021doubly}.
Fix $t\in(\Exploration,T]$ and define $\Gamma_t := \{i\in[N]:\Tildepi{i}{t} > \gamma \}$ for given $\gamma \in [1/(N+1),1/N)$.
\paragraph{Step 1. Proving $\Gamma_t \subseteq S_t$:}
By definition of $S_t$ in~\eqref{eq:super_unsaturated_arms}, we have $\Optimalarm{t} \in S_t$.
Let $\Sampledreward{i}{t}:=\mu(\Context{i}{t}^T\BetaSampled{i}{t})$, where $\BetaSampled{i}{t}$ is the sampled from Gaussian distribution defined in Algorithm~\ref{alg:DRTS}.
Suppose the estimated reward for the optimal arm, $\Sampledreward{\Optimalarm{t}}{t}$ is greater than $\Sampledreward{j}{t}$ for all $j\in S_t$, i.e. $\max_{j\notin S_t} \Sampledreward{j}{t} \le \Sampledreward{\Optimalarm{t}}{t}$.
In this case, any arm $j\notin S_t$ cannot be selected as a candidate arm $\Candidatearm{t} := \arg\max_{i\in[N]} \Sampledreward{i}{t}$ and $\Candidatearm{t} \in S_t$.
Thus,
\begin{align*}
\CP{\Candidatearm t\in S_{t}}{\History t}
&\ge\CP{\bigcap_{j\notin S_{t}}\left\{ \Sampledreward{\Optimalarm t}t>\Sampledreward jt\right\} }{\History t}\\
&=\CP{\bigcap_{j\notin S_{t}}\left\{ \Context{\Optimalarm t}t^{T}\BetaSampled{\Optimalarm t}t>\Context jt^{T}\BetaSampled jt\right\} }{\History t}\\
&=\CP{\bigcap_{j\notin S_{t}}\left\{ Z_{j,t}>\left(\Context jt-\Context{\Optimalarm t}t\right)^{T}\Estimator{t-1}\right\} }{\History t},
\end{align*}
where $Z_{j,t}:=\Context{\Optimalarm t}t^{T}\left(\BetaSampled{\Optimalarm t}t-\Estimator{t-1}\right)-\Context jt^{T}\left(\BetaSampled jt-\Estimator{t-1}\right)$.
The first equality holds since $\mu$ is a nondecreasing function.
Because the distribution of $\BetaSampled{j}{t}$ is Gaussian, we can deduce that $Z_{j,t}$ is a centered Gaussian random variable with variance $v^2 (\|\Context{\Optimalarm{t}}{t}\|^2_{\widehat{W}_{t-1}^{-1}}+\|\Context{j}{t}\|_{\widehat{W}_{t-1}^{-1}}^2)$, where $\widehat{W}_t:=\sum_{\tau=1}^{t}\sum_{i=1}^{N}\mu^{\prime}(\Context i\tau^{T}\Estimator t)\Context i\tau\Context i\tau^{T}+\lambda I_{d}$.
Denote the inverse function of $\mu$ by $\mu^{-1}$. %(By Assumption~\ref{assump:mean_function}, the inverse function $\mu^{-1}$ exists by the inverse function theorem.).  
Then for each $j\notin S_t$, by mean value theorem, there exists $\bar{\mu}$ between $\mu(\Context{\Optimalarm{t}}{t}^{T}\Estimator{t-1})$ and $\mu(\Context{j}{t}^{T}\Estimator{t-1})$ such that 
\begin{align*}
\left(\Context{j}{t}\!-\!\Context{\Optimalarm{t}}{t}\right)^{T}\!\!\Estimator{t-1}\!\! 
&=\! \left(\mu^{-1}\right)^{\prime}\!(\bar{\mu})\!\left\{\mu(\Context{j}{t}^{T}\Estimator{t-1}) - \mu(\Context{\Optimalarm{t}}{t}^{T}\Estimator{t-1}) \right\}\\
&=\! \left(\mu^{-1}\right)^{\prime}\!(\bar{\mu})\!\left\{\!\mu(\Context{j}{t}^{T}\Estimator{t-1}) \!-\! \mu(\Context{j}{t}^{T}\beta^{*}) \!-\! \mu(\Context{\Optimalarm{t}}{t}^{T}\Estimator{t-1}) \!+\! \mu(\Context{\Optimalarm{t}}{t}^{T}\beta^{*}) \!-\! \Diff{j}{t} \!\right\}\\
&\le\! \left(\mu^{-1}\right)^{\prime}\!(\bar{\mu})\!\left\{\Prederror{j}{t}+\Prederror{\Optimalarm{t}}{t} - \Diff{j}{t} \right\}\\
&\le\! \left(\mu^{-1}\right)^{\prime}\!(\bar{\mu})\!\left\{-\sqrt{\kappa L_1}\sqrt{\norm{\Context{\Optimalarm t}t}_{W_{t-1}^{-1}}^{2}+\norm{\Context jt}_{W_{t-1}^{-1}}^{2}} \right\}\\
&\le -\sqrt{\frac{\kappa}{L_1}} \sqrt{\norm{\Context{\Optimalarm t}t}_{W_{t-1}^{-1}}^{2}+\norm{\Context jt}_{W_{t-1}^{-1}}^{2}},
\end{align*}
where the second inequality holds due to $j \notin S_t$ and~\eqref{eq:super_unsaturated_arms}, and the last inequality holds due to the fact that $(\mu^{-1})^{\prime}(y) = 1/\mu^{\prime}(\mu^{-1}(y)) \ge L_1^{-1}$ for all $y\in\{\mu(x^{T}\beta):\|x\|_2\le 1, \beta \in \mathcal{B}_r^{*}\}$ (Assumption~\ref{assump:mean_function}). 
By Assumption~\ref{assump:mean_function}, and $\kappa \le L_1$,
\begin{align*}
W_{t}:=&\sum_{\tau=1}^{t}\sum_{i=1}^{N}\mu^{\prime}\left(\Context i\tau^{T}\beta^{*}\right)\Context i\tau\Context i\tau^{T}+\lambda I_{d}\\
\preceq&L_{1}\sum_{\tau=1}^{t}\sum_{i=1}^{N}\Context i\tau\Context i\tau^{T}+\lambda I_{d}\\
\preceq&\frac{L_{1}}{\kappa}\sum_{\tau=1}^{t}\sum_{i=1}^{N}\mu^{\prime}\left(\Context i\tau^{T}\Estimator t\right)\Context i\tau\Context i\tau^{T}+\lambda I_{d}\\
\preceq&\frac{L_{1}}{\kappa}\sum_{\tau=1}^{t}\sum_{i=1}^{N}\mu^{\prime}\left(\Context i\tau^{T}\Estimator t\right)\Context i\tau\Context i\tau^{T}+\frac{L_{1}}{\kappa}\lambda I_{d}\\
=&\frac{L_{1}}{\kappa}\widehat{W}_{t}.
\end{align*}
This implies that for each $j\notin S_t$,
\begin{equation}
\left(\Context jt-\Context{\Optimalarm t}t\right)^{T}\Estimator{t-1}
\le -\frac{\kappa}{L_1}\sqrt{\norm{\Context{\Optimalarm t}t}_{\widehat{W}_{t-1}^{-1}}^{2}+\norm{\Context jt}_{\widehat{W}_{t-1}^{-1}}^{2}}.
\label{eq:tilde_action_to_U_bound}
\end{equation}
Thus, we have
\begin{align*}
\CP{\Candidatearm t\in S_{t}}{\History t}\ge
&\CP{\bigcap_{j\notin S_{t}}\left\{ Z_{j,t}>-\frac{\kappa}{L_1}\sqrt{\norm{\Context{\Optimalarm t}t}_{\widehat{W}_{t-1}^{-1}}^{2}+\norm{\Context jt}_{\widehat{W}_{t-1}^{-1}}^{2}}\right\} }{\History t}\\
=&\CP{\bigcap_{j\notin S_{t}}\left\{ \frac{Z_{j,t}}{v\sqrt{\norm{\Context{\Optimalarm t}t}_{\widehat{W}_{t-1}^{-1}}^{2}+\norm{\Context jt}_{\widehat{W}_{t-1}^{-1}}^{2}}}>-\frac{\kappa}{L_1}v^{-1}\right\} }{\History t}\\
:=&\CP{\bigcap_{j\notin S_{t}}\left\{ U_{j,t}>-\frac{\kappa}{L_1}v^{-1}\right\} }{\History t}.
\end{align*}
Here, $\{U_{j,t}\}_{j\notin S_t}$ are independent standard Gaussian random variables given $\History{t}$.
Setting $v=(\kappa/L_1)\left\{ 2\log\left(\frac{N}{1-\gamma N}\right)\right\} ^{-1/2}$ gives
\[
\CP{U_{j,t}<-\frac{\kappa}{L_1}v^{-1}}{\History t}\le\frac{1-\gamma N}{N},
\]
for each $j\notin S_t$.
Then we have
\begin{align*}
\CP{\bigcap_{j\notin S_{t}}\left\{ U_{j,t}>-\frac{\kappa}{L_1}v^{-1}\right\} }{\History t}=&\prod_{j\notin S_{t}}\CP{U_{j,t}>-\frac{\kappa}{L_1}v^{-1}}{\History t}\\
=&\prod_{j\notin S_{t}}\left\{ 1-\CP{U_{j,t}<-\frac{\kappa}{L_1}v^{-1}}{\History t}\right\} \\
\ge&\prod_{j\notin S_{t}}\left\{ 1-\frac{1-\gamma N}{N}\right\} \\
\ge&1-\left(1-\gamma N\right)\\=&\gamma N\\
\ge&1-\gamma,
\end{align*}
where the last inequality holds due to $\gamma > 1/(N+1)$.
With~\eqref{eq:tilde_action_to_U_bound}, we have $\CP{\Candidatearm t\in S_{t}}{\History t} \ge 1-\gamma$.
For any $j \notin S_t$,
\[
\Tildepi jt=\CP{\Candidatearm t=j}{\History t}\le\CP{\bigcup_{j\notin S_{t}}\left\{ \Candidatearm t=j\right\} }{\History t}=\CP{\Candidatearm t\notin S_{t}}{\History t}\le\gamma.
\]
Thus, we conclude that $S_t^{c} \subseteq \Gamma_t^{c}$.
\paragraph{Step 2. Proving $\Action{t} \in \Gamma_t$ with high probability:}
For $m\in[M_t]$, let $a_t^{(m)}$ be the candidate arm in $m$-th resampling.
Since $a_t^{(1)}, \ldots, a_t^{(M_t)}$ are independent,
\[
\CP{\Action t\notin\Gamma_{t}}{\History t}\le\CP{\bigcap_{m=1}^{M_{t}}\left\{ a_{t}^{(m)}\notin\Gamma_{t}\right\} }{\History t}=\prod_{m=1}^{M_{t}}\CP{a_{t}^{(m)}\notin\Gamma_{t}}{\History t}\le\left(1-\gamma\right)^{M_{t}},
\]
where the last inequality holds because $\gamma < 1/N$ and there exists at least one arm $i$ such that $\Tildepi{i}{t}>\gamma$, which implies $\CP{a_{t}^{(m)}\notin\Gamma_{t}}{\History t}\le1-\gamma$.
Setting $M_{t}=\log\frac{(t+1)^{2}}{\delta}\Big/\log\frac{1}{1-\gamma}$ proves 
\[
\CP{\Action t \in \Gamma_{t}}{\History t} \ge 1-\frac{\delta}{t^2},
\]
for all $t\in[T]$.
\paragraph{Step 3. Proving the lemma:}
Because $\Gamma_t \subseteq S_t$ for $t\in(\Exploration,T]$, we can deduce that $\Action{t} \in \Gamma_{t}$ implies $\Action{t} \in S_t$.
Because $\SelectionP{i}{t} \ge \Tildepi{i}{t}$ for all $i \in \Gamma_t$ and $t\in[T]$, we have that $\Action{t} \in \Gamma_{t}$ implies $\SelectionP{\Action{t}}{t} > \gamma$.
Thus, we conclude that
\[
\Probability\left(\bigcap_{t=\Exploration}^{T}\left\{ \Action t\in S_{t}\right\} \cap\bigcap_{t=1}^{T}\left\{ \SelectionP{\Action t}t>\gamma\right\} \right)\ge\Probability\left(\bigcap_{t=1}^{T}\left\{ \Action t\in\Gamma_{t}\right\} \right)\ge1-\delta
\]
\end{proof}

\subsection{Proof of Lemma~\ref{lem:prediction_error}}
\label{subsec:proof_of_pred_error}
\begin{proof}
Let us fix $t \in (\Exploration,T]$.
Then the tail inequality for DDR estimator (Theorem~\ref{thm:estimation_error_bound}) and Corollary~\ref{cor:min_eigen_chernoff} imply that, with probability at least $1-6\delta/T$,
\begin{equation}
\begin{split}
\norm{\Estimator{t-1}-\beta^{*}}_{2}&\le\frac{12L_{1}+8B}{\kappa\phi\sqrt{t-1}}\sqrt{\log\frac{4T}{\delta}}+\frac{\frac{9}{4}L_{1}r+\frac{\lambda S}{N}}{\kappa\phi\left(t-1\right)},\\
\min_{\beta\in\mathcal{B}_{r}^{*}}\Mineigen{\sum_{\tau=1}^{t-1}\sum_{i=1}^{N}\mu^{\prime}\left(\Context i{\tau}^{T}\beta\right)\Context i{\tau}\Context i{\tau}^{T}}&\ge\frac{\phi\kappa N\left(t-1\right)}{2}.
\label{eq:estimator_conditions}
\end{split}
\end{equation}
As the second line of event is required for the first line of event,~\eqref{eq:estimator_conditions} holds with probability at least $1-6\delta$ not $1-7\delta$.

\paragraph{Step 1 Approximation:} For each $j\in[N]$, using second-order Taylor expansion, there exists $v\in(0,1)$ such that $\bar{\beta}=v\Estimator{t-1} + (1-v)\beta^{*}$ and
\[
\mu\left(\Context jt^{T}\Estimator{t-1}\right)-\mu\left(\Context jt^{T}\beta^{*}\right)=\mu^{\prime}\left(\Context jt^{T}\beta^{*}\right)\Context jt^{T}\left(\Estimator{t-1}-\beta^{*}\right)+\frac{\mu^{\prime\prime}\left(\Context jt^{T}\bar{\beta}\right)\left\{ \Context jt^{T}\left(\Estimator{t-1}-\beta^{*}\right)\right\} ^{2}}{2}.
\]
Taking the absolute value on both sides,
\begin{equation}
\begin{split}
\abs{\mu\left(\Context jt^{T}\Estimator{t-1}\right)-\mu\left(\Context jt^{T}\beta^{*}\right)}\le&\abs{\mu^{\prime}\left(\Context jt^{T}\beta^{*}\right)\Context jt^{T}\left(\Estimator{t-1}-\beta^{*}\right)}\\
&\;+\abs{\frac{\mu^{\prime\prime}\left(\Context jt^{T}\bar{\beta}\right)\left\{ \Context jt^{T}\left(\Estimator{t-1}-\beta^{*}\right)\right\} ^{2}}{2}}\\
\le&\mu^{\prime}\left(\Context jt^{T}\beta^{*}\right)\abs{\Context jt^{T}\left(\Estimator{t-1}-\beta^{*}\right)}+\frac{L_{2}\norm{\Estimator{t-1}-\beta}_{2}^{2}}{2}\\
\le&\mu^{\prime}\left(\Context jt^{T}\beta^{*}\right)\abs{\Context jt^{T}\left(\Estimator{t-1}-\beta^{*}\right)}\\&+L_{2}\left\{ \frac{\left(12L_{1}+8B\right)^{2}}{\kappa^{2}\phi^{2}\left(t-1\right)}\log\frac{4T}{\delta}+\frac{\left(\frac{9}{4}L_{1}r+\frac{\lambda S}{N}\right)^{2}}{\kappa^{2}\phi^{2}N^{2}\left(t-1\right)^{2}}\right\} \\
\le&\mu^{\prime}\left(\Context jt^{T}\beta^{*}\right)\abs{\Context jt^{T}\left(\Estimator{t-1}-\beta^{*}\right)}\\
&\;+\frac{L_{2}\left\{ \left(12L_{1}+8B\right)^{2}+1\right\} }{\phi^{2}\kappa^{2}\left(t-1\right)}\log\frac{4T}{\delta}.
\label{eq:prediction_error_approximation}
\end{split}
\end{equation}
The second inequality holds due to Assumption~\ref{assump:boundedness} and~\ref{assump:mean_function}, and the third inequality holds by~\eqref{eq:estimator_conditions}.
The last inequality uses the fact that $t>\Exploration$ and
\[
\frac{\left(\frac{9}{4}L_{1}r+\frac{\lambda S}{N}\right)^{2}}{\kappa^{2}\phi^{2}\left(t-1\right)^{2}}\le\frac{\left(\frac{9}{4}L_{1}r+\frac{\lambda S}{N}\right)^{2}}{\kappa^{2}\phi^{2}\left(t-1\right)\Exploration}\le\frac{1}{\phi^{2}\kappa^{2}\left(t-1\right)}\log\frac{4T}{\delta}.
\]
To bound the first term in~\eqref{eq:prediction_error_approximation}, we observe that $U_{t-1}(\Estimator{t-1})=0$, where $U_t$ is the score function defined in~\eqref{eq:score_equation}.
Let $W_{t}:=\sum_{\tau=1}^{t}\sum_{i=1}^{N}\mu^{\prime}\left(\Context i{\tau}^{T}\beta^{*}\right)\Context i{\tau}\Context i{\tau}^{T}+\lambda I_{d}$.
Then by second-order Taylor expansion, there exists $v\in(0,1)$ such that $\bar{\beta}=v\Estimator{t-1} + (1-v)\beta^{*}$ and
\begin{align*}
0 &= U_{t-1}\left(\Estimator{t-1}\right) \\
&= U_{t-1}\left(\beta^{*}\right)-W_{t-1}\left(\Estimator{t-1}-\beta^{*}\right)-\frac{1}{2}\sum_{\tau=1}^{t-1}\sum_{i=1}^{N}\mu^{\prime\prime}\left(\Context i{\tau}^{T}\bar{\beta}\right)\left\{ \Context i{\tau}^{T}\left(\Estimator{t-1}-\beta^{*}\right)\right\} ^{2}\Context i{\tau}.
\end{align*}
Multiplying $\Context{j}{t}^{T} W_{t-1}^{-1}$ on both sides gives
\begin{align*}
\Context jt^{T}\left(\Estimator{t-1}-\beta^{*}\right)&=\Context jt^{T}W_{t-1}^{-1}U_{t-1}\left(\beta^{*}\right)\\
&\quad-\frac{1}{2}\sum_{\tau=1}^{t-1}\sum_{i=1}^{N}\mu^{\prime\prime}\left(\Context i\tau^{T}\bar{\beta}\right)\left\{ \Context i\tau^{T}\left(\Estimator{t-1}-\beta^{*}\right)\right\} ^{2}\Context jt^{T}W_{t-1}^{-1}\Context i\tau.
\end{align*}
Taking the absolute value on both sides,
\begin{equation}
\begin{split}
\abs{\Context jt^{T}\left(\Estimator{t-1}-\beta^{*}\right)}
\le&\abs{\Context jt^{T}W_{t-1}^{-1}U_{t-1}\left(\beta^{*}\right)}\\
&\quad+\frac{1}{2}\abs{\sum_{\tau=1}^{t-1}\sum_{i=1}^{N}\mu^{\prime\prime}\left(\Context i\tau^{T}\bar{\beta}\right)\left\{ \Context i\tau^{T}\left(\Estimator{t-1}-\beta^{*}\right)\right\} ^{2}\Context jt^{T}W_{t-1}^{-1}\Context i\tau}\\
\le&\abs{\Context jt^{T}W_{t-1}^{-1}U_{t-1}\left(\beta^{*}\right)}+\frac{L_{2}\norm{\Estimator{t-1}-\beta^{*}}_{2}^{2}}{2}\sum_{\tau=1}^{t-1}\sum_{i=1}^{N}\abs{\Context jt^{T}W_{t-1}^{-1}\Context i\tau}\\
\le&\abs{\Context jt^{T}W_{t-1}^{-1}U_{t-1}\left(\beta^{*}\right)}+\frac{L_{2}\norm{\Estimator{t-1}-\beta^{*}}_{2}^{2}\Mineigen{W_{t-1}}^{-1}N\left(t-1\right)}{2}\\
\le&\abs{\Context jt^{T}W_{t-1}^{-1}U_{t-1}\left(\beta^{*}\right)}+\frac{L_{2}}{\phi\kappa}\left\{ \frac{\left(12L_{1}+8B\right)^{2}}{\phi^{2}\kappa^{2}\left(t-1\right)}\log\frac{4T}{\delta}+\frac{\left(\frac{9}{4}L_{1}r+\frac{\lambda S}{N}\right)^{2}}{\kappa^{2}\phi^{2}(t-1)^{2}}\right\} \\
\le&\abs{\Context jt^{T}W_{t-1}^{-1}U_{t-1}\left(\beta^{*}\right)}+\frac{L_{2}\left\{ \left(12L_{1}+8B\right)^{2}+1\right\} }{\phi^{3}\kappa^{3}\left(t-1\right)}\log\frac{4T}{\delta}.
\end{split}
\label{eq:prediction_error_bound_1}
\end{equation}
The second and third inequality holds due to Assumption~\ref{assump:boundedness} and~\ref{assump:mean_function}, and the fourth inequality holds due to~\eqref{eq:estimator_conditions}.
By Assumption~\ref{assump:boundedness} and~\eqref{eq:estimator_conditions},
\begin{equation}
\begin{split}
\abs{\Context jt^{T}W_{t-1}^{-1}U_{t-1}\left(\beta^{*}\right)}=&\abs{\Context jt^{T}W_{t-1}^{-1}\left[\sum_{\tau=1}^{t-1}\sum_{i=1}^{N}\left\{ \DRreward i\tau{\Impute{t-1}}-\mu\left(\Context i\tau^{T}\beta^{*}\right)\right\} \Context i\tau-\lambda\beta^{*}\right]}\\
\le&\abs{\sum_{\tau=1}^{t-1}\sum_{i=1}^{N}\left\{ \DRreward i\tau{\Impute{t-1}}-\mu\left(\Context i\tau^{T}\beta^{*}\right)\right\} \Context jt^{T}W_{t-1}^{-1}\Context i\tau}+\lambda\abs{\Context jt^{T}W_{t-1}^{-1}\beta^{*}}\\
\le&\abs{\sum_{\tau=1}^{t-1}\sum_{i=1}^{N}\left\{ \DRreward i\tau{\Impute{t-1}}-\mu\left(\Context i\tau^{T}\beta^{*}\right)\right\} \Context jt^{T}W_{t-1}^{-1}\Context i\tau}+\lambda S\Mineigen{W_{t-1}}^{-1}\\
=&\abs{\sum_{\tau=1}^{t-1}\sum_{i=1}^{N}\left\{ \DRreward i\tau{\Impute{t-1}}-\mu\left(\Context i\tau^{T}\beta^{*}\right)\right\} \Context jt^{T}W_{t-1}^{-1}\Context i\tau}+\frac{2\lambda S}{\kappa\phi N\left(t-1\right)}.
\end{split}
\label{eq:prediction_error_bound_2}
\end{equation}
From~\eqref{eq:prediction_error_approximation},~\eqref{eq:prediction_error_bound_1} and~\eqref{eq:prediction_error_bound_2}, the bound for the prediction error is approximated by
\begin{equation}
\Prederror jt\le\mu^{\prime}\left(\Context jt^{T}\beta^{*}\right)\abs{\sum_{\tau=1}^{t-1}\sum_{i=1}^{N}\left\{ \DRreward i\tau{\Impute{t-1}}-\mu\left(\Context i\tau^{T}\beta^{*}\right)\right\} \Context jt^{T}W_{t-1}^{-1}\Context i\tau}+\frac{D_{\mu,B,\lambda,S}-3 L_1}{\phi^{3}\kappa^{3}\left(t-1\right)}\log\frac{4T}{\delta},
\label{eq:prediction_error_main_term}
\end{equation}
where
\begin{equation}
D_{\mu,B,\lambda,S}:=2L_{2}\left\{ \left(12L_{1}+8B\right)^{2}+1\right\} +2\lambda S + 3 L_1.
\label{eq:D_definition}
\end{equation}
\paragraph{Step 2 Bounding the main order term:} 
Now we bound the first term in~\eqref{eq:prediction_error_main_term}, which is the main order term.
By definition of $\DRreward i{\tau}{\Impute{t-1}}$,
\begin{equation}
\begin{split}
&\abs{\sum_{\tau=1}^{t-1}\sum_{i=1}^{N}\left\{ \DRreward i\tau{\Impute{t-1}}-\mu\left(\Context i\tau^{T}\beta^{*}\right)\right\} \Context jt^{T}W_{t-1}^{-1}\Context i\tau}\\
&\le \abs{\sum_{\tau=1}^{t-1}\sum_{i=1}^{N}\left\{ 1-\frac{\Indicator{\Action \tau=i}}{\SelectionP i\tau}\right\} D_{i,\tau}^{\Impute{t-1}}\Context jt^{T}W_{t-1}^{-1}\Context i\tau} +\abs{\sum_{\tau=1}^{t-1}\frac{\Error \tau}{\SelectionP{\Action \tau}\tau}\Context jt^{T}W_{t-1}^{-1}\Context{\Action \tau}\tau}.
\end{split}
\label{eq:prediction_martingale_decomposition}
\end{equation}
where $D_{i,\tau}^{\Impute t}:=\mu(\Context i{\tau}^{T}\Impute t)-\mu(\Context i{\tau}^{T}\beta^{*})$, and $\Error{\tau}:=Y_{\tau}-\mu(\Context{\Action{\tau}}{\tau}^{T}\beta^{*})$.
To bound the first term, we let $A_{i,\tau}:=1-\frac{\Indicator{\Action{\tau}=i}}{\SelectionP{i}{\tau}}$ and define $\breve{\mathcal{B}}:= \{\beta\in\Real^{d}:\|\beta-\beta^{*}\|_{2}\le\min\{\kappa^{1/2}N^{-1/2}d^{-1/2},r\}\}$.
By Lemma~\ref{lem:imputation_error_bound}, the imputation estimator $\Impute{t}$ is in $\breve{\mathcal{B}}$, with probability at least $1-3\delta/T$.
For any $\epsilon>0$, let $\beta_1, \ldots, \beta_{M(\epsilon)}$ be the $\epsilon$-cover for $\breve{\mathcal{B}}$.
Then for any $\beta\in\breve{\mathcal{B}}$, there exists $m\in[M(\epsilon)]$ such that $\|\beta - \beta_m\|_2 \le \epsilon$ and
\begin{align*}
&\abs{\sum_{\tau=1}^{t-1}\sum_{i=1}^{N}A_{i,\tau}D_{i,\tau}^{\beta}\Context jt^{T}W_{t-1}^{-1}\Context i\tau}\\
&\le\abs{\sum_{\tau=1}^{t-1}\sum_{i=1}^{N}A_{i,\tau}\left(D_{i,\tau}^{\beta}-D_{i,\tau}^{\beta_{m}}\right)\Context jt^{T}W_{t-1}^{-1}\Context i\tau}+\abs{\sum_{\tau=1}^{t-1}\sum_{i=1}^{N}A_{i,\tau}D_{i,\tau}^{\beta_{m}}\Context jt^{T}W_{t-1}^{-1}\Context i\tau}\\
&\le L_{1}\epsilon\sum_{\tau=1}^{t-1}\sum_{i=1}^{N}\abs{A_{i,\tau}\Context jt^{T}W_{t-1}^{-1}\Context i\tau}+\max_{m\in[M(\epsilon)]}\abs{\sum_{\tau=1}^{t-1}\sum_{i=1}^{N}A_{i,\tau}D_{i,\tau}^{\beta_{m}}\Context jt^{T}W_{t-1}^{-1}\Context i\tau}\\
&\le L_{1}\epsilon\Mineigen{W_{t-1}}^{-1}\sum_{\tau=1}^{t-1}\sum_{i=1}^{N}\abs{A_{i,\tau}}+\max_{m\in[M(\epsilon)]}\abs{\sum_{\tau=1}^{t-1}\sum_{i=1}^{N}A_{i,\tau}D_{i,\tau}^{\beta_{m}}\Context jt^{T}W_{t-1}^{-1}\Context i\tau}\\
&\le\frac{2L_{1}\epsilon\left(\gamma^{-1}+N\right)}{\kappa\phi N}+\max_{m\in[M(\epsilon)]}\abs{\sum_{\tau=1}^{t-1}\sum_{i=1}^{N}A_{i,\tau}D_{i,\tau}^{\beta_{m}}\Context jt^{T}W_{t-1}^{-1}\Context i\tau}.
\end{align*}
The second inequality holds due to Assumption~\ref{assump:boundedness} and~\ref{assump:mean_function}, the third inequality holds by Assumption~\ref{assump:boundedness} and the last inequality holds by~\eqref{eq:estimator_conditions} and the event in~\eqref{eq:action_in_superunsaturated_arms}.
Taking the supremum over $\breve{\mathcal{B}}$ gives
\begin{equation}
\begin{split}
&\abs{\sum_{\tau=1}^{t-1}\sum_{i=1}^{N}A_{i,\tau}D_{i,\tau}^{\Impute{t-1}}\Context jt^{T}W_{t-1}^{-1}\Context i\tau}\\
&\le\sup_{\beta\in\breve{\mathcal{B}}}\abs{\sum_{\tau=1}^{t-1}\sum_{i=1}^{N}A_{i,\tau}D_{i,\tau}^{\beta}\Context jt^{T}W_{t-1}^{-1}\Context i\tau}\\
&\le\frac{2L_{1}\epsilon\left(\gamma^{-1}+N\right)}{\kappa\phi N}+\max_{m\in[M(\epsilon)]}\abs{\sum_{\tau=1}^{t-1}\sum_{i=1}^{N}A_{i,\tau}D_{i,\tau}^{\beta_{m}}\Context jt^{T}W_{t-1}^{-1}\Context i\tau}\\
&\le\frac{6L_{1}\epsilon}{\kappa\phi}+\max_{m\in[M(\epsilon)]}\abs{\sum_{\tau=1}^{t-1}\sum_{i=1}^{N}A_{i,\tau}D_{i,\tau}^{\beta_{m}}\Context jt^{T}W_{t-1}^{-1}\Context i\tau},
\end{split}
\label{eq:pred_error_covering_bound}
\end{equation}
where the last inequality holds due to $\gamma > 1/(N+1)$.
Define the filtration $\Filtration{\tau}:=\History{\tau}\cup\{\Setofcontexts{1},\ldots,\Setofcontexts{t}\}$, and $\Filtration{0}=\{\Setofcontexts{1},\ldots,\Setofcontexts{t}\}$.
Then for each $m\in [M(\epsilon)]$, let
\[
P_{u}^{(m)}:=\sum_{\tau=1}^{u}\sum_{i=1}^{N}A_{i,\tau}D_{i,\tau}^{\beta_{m}}\Context jt^{T}W_{t-1}^{-1}\Context i\tau,\quad P_{0}^{(m)}=0.
\]
Then $P_u^{(m)}$ is a martingale sequence with respect to $\Filtration{u}$ because $P_{u}^{(m)}$ is $\Filtration{u+1}$-measurable, and
\begin{align*}
\CE{P_{u}^{(m)}}{\Filtration s}
=&P_{u-1}^{(m)}+\CE{\sum_{i=1}^{N}A_{i,u}D_{i,u}^{\beta_{m}}\Context jt^{T}W_{t-1}^{-1}\Context iu}{\Filtration u}\\
=&P_{u-1}^{(m)}+\CE{\sum_{i=1}^{N}A_{i,u}}{\Filtration u}D_{i,u}^{\beta_{m}}\Context jt^{T}W_{t-1}^{-1}\Context iu\\
=&P_{u-1}^{(m)}+\CE{\sum_{i=1}^{N}A_{i,u}}{\History u}D_{i,u}^{\beta_{m}}\Context jt^{T}W_{t-1}^{-1}\Context iu\\
=&P_{u-1}^{(m)}
\end{align*}
where the third inequality holds since $\History{u}$ is independent with $\Setofcontexts{u+1},\ldots,\Setofcontexts{t}$ because of Assumption~\ref{assump:iid_contexts}.
Let
\[
E_{\gamma}:=\bigcap_{t=\Exploration}^{T}\left\{ \Action t\in S_{t}\right\} \cap\bigcap_{t=1}^{T}\left\{ \SelectionP{\Action t}t>\gamma\right\} 
\]
When $E_{\gamma}$ holds, the differences of the martingale is bounded as
\begin{align*}
\abs{P_{u}^{(m)}-P_{u-1}^{(m)}}
=&\abs{\sum_{i=1}^{N}A_{i,u}D_{i,u}^{\beta_{m}}\Context it^{T}W_{t-1}^{-1}\Context iu}\\
\le&\sqrt{\sum_{i=1}^{N}A_{i,u}^{2}}\sqrt{\sum_{i=1}^{N}\left(D_{i,u}^{\beta_{m}}\Context it^{T}W_{t-1}^{-1}\Context iu\right)^{2}}\\
\le&\sqrt{N+\SelectionP{\Action u}u^{-2}}L_{1}\norm{\beta_{m}-\beta^{*}}_{2}\sqrt{\sum_{i=1}^{N}\left(\Context it^{T}W_{t-1}^{-1}\Context iu\right)^{2}}\\
\le&\sqrt{N+\gamma^{-2}}L_{1}\norm{\beta_{m}-\beta^{*}}_{2}\sqrt{\sum_{i=1}^{N}\left(\Context it^{T}W_{t-1}^{-1}\Context iu\right)^{2}}\\
\le&2L_{1}\sqrt{1+\gamma^{-2}N^{-1}}\sqrt{\frac{\kappa}{d}\sum_{i=1}^{N}\left(\Context it^{T}W_{t-1}^{-1}\Context iu\right)^{2}}\\
\le&2L_{1}\sqrt{\frac{3\kappa N}{d}\sum_{i=1}^{N}\left(\Context it^{T}W_{t-1}^{-1}\Context iu\right)^{2}},
\end{align*}
for all $u\in[t-1]$, and $m\in[M(\epsilon)]$. 
The fourth inequality holds since $\beta_m \in \breve{\mathcal{B}}$, and the last inequality holds since $\gamma\in [1/(N+1),1/N)$.
Thus we have
\[
\left\{ \abs{P_{u}^{(m)}-P_{u-1}^{(m)}}>2L_{1}\sqrt{\frac{3N\kappa}{d}\sum_{i=1}^{N}\left(\Context it^{T}W_{t-1}^{-1}\Context iu\right)^{2}}\right\} \subset E_{\gamma}^{c}.
\]
By Lemma~\ref{lem:chung_lemma}, there exists a martingale sequence $Q_{u}^{(m)}$ such that
\[
\abs{Q_{u}^{(m)}-Q_{u-1}^{(m)}}\le2L_{1}\sqrt{\frac{3N\kappa}{d}\sum_{i=1}^{N}\left(\Context it^{T}W_{t-1}^{-1}\Context iu\right)^{2}},
\]
almost surely, and 
\[
\left\{ Q_{u}^{(m)}\neq P_{u}^{(m)}\right\} \subset\left\{ \abs{P_{u}^{(m)}-P_{u-1}^{(m)}}>2L_{1}\sqrt{\frac{3N\kappa}{d}\sum_{i=1}^{N}\left(\Context it^{T}W_{t-1}^{-1}\Context iu\right)^{2}}\right\} \subset E_{\gamma}^{c},
\]
for all $u\in[t-1]$.
By Lemma~\ref{lem:Azuma_Hoeffding_inequality}, for any $a > 0$,
\begin{align*}
&\CP{\left\{ \max_{m\in[M(\epsilon)]}\abs{P_{t-1}^{(m)}}>a\right\} \cap E_{\gamma}}{\Filtration 0}\\
&\le\CP{\max_{m\in[M(\epsilon)]}\abs{Q_{t-1}^{(m)}}>a}{\Filtration 0}\\
&\le\sum_{m=1}^{M(\epsilon)}\CP{\abs{Q_{t-1}^{(m)}}>a}{\Filtration 0}\\
&\le2M(\epsilon)\exp\left(-\frac{da^{2}}{24L_{1}^{2}N\kappa\sum_{\tau=1}^{t-1}\sum_{i=1}^{N}\left(\Context jt^{T}W_{t-1}^{-1}\Context i\tau\right)^{2}}\right)\\
&\le2M(\epsilon)\exp\left(-\frac{da^{2}}{24NL_{1}^{2}\norm{\Context jt}_{W_{t-1}^{-1}}^{2}}\right),
\end{align*}
where the last inequality holds because
\begin{equation}
\begin{split}
\kappa\sum_{\tau=1}^{t-1}\sum_{i=1}^{N}\left(\Context jt^{T}W_{t-1}^{-1}\Context i\tau\right)^{2}&\le\Context jt^{T}W_{t-1}^{-1}\left(\sum_{\tau=1}^{t-1}\mu^{\prime}\left(\Context i{\tau}^{T}\beta^{*}\right)\Context i{\tau}\Context i{\tau}^{T}\right)W_{t-1}^{-1}\Context jt\\
&\le\Context jt^{T}W_{t-1}^{-1}\Context jt.
\end{split}
\label{eq:SS_to_norm}
\end{equation}
Thus, with~\eqref{eq:pred_error_covering_bound} and the event $E_{\gamma}$, with probability at least $1-\delta/(NT)$,
\begin{align*}
\abs{\sum_{\tau=1}^{t-1}\sum_{i=1}^{N}A_{i,\tau}D_{i,\tau}^{\Impute{t-1}}\Context jt^{T}W_{t-1}^{-1}\Context i\tau}\le&\frac{6L_{1}\epsilon}{\kappa\phi}+\max_{m\in[M(\epsilon)]}\abs{\sum_{\tau=1}^{t-1}\sum_{i=1}^{N}A_{i,\tau}D_{i,\tau}^{\beta_{m}}\Context jt^{T}W_{t-1}^{-1}\Context i\tau}\\
\le&\frac{6L_{1}\epsilon}{\kappa\phi}+L_{1}\norm{\Context jt}_{W_{t-1}^{-1}}\sqrt{\frac{24N}{d}\log\frac{2M(\epsilon)NT}{\delta}}.
\end{align*}
Since the radius of $\breve{\mathcal{B}}$ is $\min \{\kappa^{1/2}N^{-1/2}d^{-1/2}, r\}$, we have a bound for the covering number $M(\epsilon)$ by
\[
M(\epsilon)\le\left(\frac{2}{\epsilon}\min\left\{ \sqrt{\frac{\kappa}{Nd}},r\right\} +1\right)^{d}\le\left(\frac{3\sqrt{\kappa }}{\sqrt{Nd}\epsilon}\right)^{d}.
\]
Setting $\epsilon=\frac{3\sqrt{\kappa }}{t-1}\frac{1}{N\sqrt{N d}}$ gives
\[
M(\epsilon)\le\left\{ N\left(t-1\right)\right\} ^{d}\le\left(NT\right)^{d},
\]
and we obtain a following bound for the first term in~\eqref{eq:prediction_error_main_term}:
\begin{equation}
\abs{\sum_{\tau=1}^{t-1}\sum_{i=1}^{N}A_{i,\tau}D_{i,\tau}^{\Impute{t-1}}\Context jt^{T}W_{t-1}^{-1}\Context i\tau}\le\frac{3}{\sqrt{\kappa}\phi N\sqrt{N}\left(t-1\right)}+4L_{1}\norm{\Context jt}_{W_{t-1}^{-1}}\sqrt{3N\log\frac{2NT}{\delta}}.
\label{eq:prediction_error_1st_term}
\end{equation}
To bound the second term in~\eqref{eq:prediction_error_main_term}, let
\[
P_{u}:=\sum_{\tau=1}^{u}\frac{\Error \tau\Indicator{\SelectionP{\Action \tau}\tau>\gamma}}{\SelectionP{\Action \tau}\tau}\Context jt^{T}W_{t-1}^{-1}\Context{\Action \tau}\tau,\quad P_{0}=0.
\]
Then $P_{u}$ is a martingale sequence with respect to the filtration $\mathcal{F}_{u}:=\History{u} \cup \{\Setofcontexts{1},\ldots,\Setofcontexts{t}\}$ and $\Filtration{0}:=\{\Setofcontexts{1},\ldots,\Setofcontexts{t}\}$ whose differences are bounded by
\begin{align*}
\abs{\Delta P_{u}}:=\abs{P_{u}-P_{u-1}}\le
&\abs{\frac{\Error u\Indicator{\SelectionP{\Action u}u>\gamma}}{\SelectionP{\Action u}u}\Context jt^{T}W_{t-1}^{-1}\Context{\Action u}u}\\
\le&\frac{2B}{\gamma}\abs{\Context jt^{T}W_{t-1}^{-1}\Context{\Action u}u}\\
\le&\frac{2B}{\gamma}\norm{\Context jt}_{W_{t-1}^{-1}}\Mineigen{W_{t-1}}^{-1/2}
\end{align*}
almost surely.
Since $W_{t-1}$ is $\Filtration{0}$-measurable, we can further bound the difference under the $\Filtration{0}$-measurable event~\eqref{eq:estimator_conditions} by
\[
\abs{\Delta P_{u}}\le\frac{2B}{\gamma}\norm{\Context jt}_{W_{t-1}^{-1}}\Mineigen{W_{t-1}}^{-1/2}\le\frac{2B}{\gamma\sqrt{\kappa\phi N(t-1)}}\norm{\Context jt}_{W_{t-1}^{-1}}.
\]
The sum of conditional variances is bounded by,
\begin{align*}
\mathcal{V}_{t-1}:=\sum_{\tau=1}^{t-1}\CE{\left(\Delta P_{\tau}\right)^{2}}{\mathcal{F}_{\tau}}
=&\sum_{\tau=1}^{t-1}\CE{\frac{\Error \tau^{2}\Indicator{\SelectionP{\Action \tau}\tau>\gamma}}{\SelectionP{\Action \tau}\tau^{2}}\left(\Context jt^{T}W_{t-1}^{-1}\Context{\Action \tau}\tau\right)^{2}}{\mathcal{F}_{\tau}}\\
\le&\gamma^{-1}\sum_{\tau=1}^{t-1}\CE{\frac{\Error \tau^{2}}{\SelectionP{\Action \tau}\tau}\left(\Context jt^{T}W_{t-1}^{-1}\Context{\Action \tau}\tau\right)^{2}}{\mathcal{F}_{\tau}}\\
=&\gamma^{-1}\sum_{\tau=1}^{t-1}\sum_{i=1}^{N}\CE{\CE{\Error \tau^{2}\left(\Context jt^{T}W_{t-1}^{-1}\Context{\Action \tau}\tau\right)^{2}}{\Filtration \tau,\Action \tau=i}}{\Filtration \tau}\\
=&\gamma^{-1}\sum_{\tau=1}^{t-1}\sum_{i=1}^{N}\mu^{\prime}\left(\Context i\tau^{T}\beta^{*}\right)\left(\Context jt^{T}W_{t-1}^{-1}\Context i\tau\right)^{2}\\
\le&\gamma^{-1}\norm{\Context jt}_{W_{t-1}^{-1}}^{2},
\end{align*}
almost surely.
The last inequality holds because of~\eqref{eq:SS_to_norm}.
Thus by Lemma~\ref{lem:freedman_inequalty}, for any $a>0$,
\begin{align*}
&\CP{\left\{ \abs{\sum_{\tau=1}^{t-1}\frac{\Error \tau}{\SelectionP{\Action \tau}\tau}\Context jt^{T}W_{t-1}^{-1}\Context{\Action \tau}\tau}>a\right\} \cap E_{t}}{\Filtration 0}\\
&\le\CP{\abs{P_{t-1}}>a}{\Filtration 0}\\
&=\CP{\left\{ \abs{P_{t-1}}>a\right\} \cap\left\{ \mathcal{V}_{t}\le\gamma^{-1}\norm{\Context jt}_{W_{t-1}^{-1}}^{2}\right\} }{\Filtration 0}\\
&\le2\exp\left(-\frac{a^{2}}{2\gamma^{-1}\norm{\Context jt}_{W_{t-1}^{-1}}^{2}+\frac{4aB\gamma^{-1}}{3\sqrt{\kappa\phi N\left(t-1\right)}}\norm{\Context jt}_{W_{t-1}^{-1}}}\right)\\
&\le2\exp\left(-\frac{a^{2}}{2\max\left\{ 2\gamma^{-1}\norm{\Context jt}_{W_{t-1}^{-1}}^{2},\frac{4aB\gamma^{-1}}{3\sqrt{\kappa\phi N\left(t-1\right)}}\norm{\Context jt}_{W_{t-1}^{-1}}\right\} }\right)\\
&=2\max\left\{ \exp\left(-\frac{\gamma a^{2}}{4\norm{\Context jt}_{W_{t-1}^{-1}}^{2}}\right),\exp\left(-\frac{a\gamma\sqrt{\kappa\phi Nt}}{\frac{4}{3}B\norm{\Context jt}_{W_{t-1}^{-1}}}\right)\right\}\\
&=2 \exp\left(-\frac{\gamma a^{2}}{4\norm{\Context jt}_{W_{t-1}^{-1}}^{2}}\right),
\end{align*}
where the last equality holds due to $t>\Exploration$.
Taking the last term equal to $\delta/(NT)$ gives
\begin{equation}
\abs{\sum_{\tau=1}^{t-1}\frac{\Error \tau}{\SelectionP{\Action \tau}\tau}\Context jt^{T}W_{t-1}^{-1}\Context{\Action \tau}\tau}\le2\norm{\Context jt}_{W_{t-1}^{-1}}\sqrt{\gamma^{-1}\log\frac{2NT}{\delta}}\le2\norm{\Context jt}_{W_{t-1}^{-1}}\sqrt{2N\log\frac{2NT}{\delta}}.
\label{eq:prediction_error_2nd_term}
\end{equation}

\paragraph{Step 3 Putting the results altogether:} 
From~\eqref{eq:prediction_error_main_term},~\eqref{eq:prediction_martingale_decomposition},~\eqref{eq:prediction_error_1st_term} and~\eqref{eq:prediction_error_2nd_term},
with probability at least $1-8\delta/T$, under the event $E_t$,
\begin{align*}
\Prederror jt\le&\mu^{\prime}\left(\Context jt^{T}\beta^{*}\right)\left\{ \left(2+4L_{1}\right)\norm{\Context jt}_{W_{t-1}^{-1}}\sqrt{3N\log\frac{2NT}{\delta}}+\frac{3}{\sqrt{\kappa}\phi N\sqrt{N}\left(t-1\right)}\right\} \\
&+\frac{D_{\mu,B,\lambda,S}-3L_{1}}{\phi^{3}\kappa^{3}\left(t-1\right)}\log\frac{4T}{\delta}\\
\le&\mu^{\prime}\left(\Context jt^{T}\beta^{*}\right)\left(2+4L_{1}\right)\norm{\Context jt}_{W_{t-1}^{-1}}\sqrt{3N\log\frac{2NT}{\delta}}+\frac{D_{\mu,B,\lambda,S}}{\phi^{3}\kappa^{3}\left(t-1\right)}\log\frac{4T}{\delta}.
\end{align*}
\end{proof}

\subsection{Proof of Lemma~\ref{lem:elliptical_potential_lemma}}
\begin{proof}
By Assumption~\ref{assump:boundedness} and~\ref{assump:mean_function} 
\begin{equation}
\max_{i\in[N]}\sqrt{w_{i,t}}s_{i,t}\le\sqrt{L_{1}\Mineigen{W_{t-1}}^{-1}}\le\sqrt{\frac{2L_{1}}{\kappa N\phi\left(t-1\right)}},
\label{eq:singleton_bound_phi}
\end{equation}
where the last inequality holds with probability at least $1-\delta/T$, by Corollary~\ref{cor:min_eigen_chernoff}.
Thus, with probability at least $1-\delta$,
\[
\sum_{t=\mathcal{\Exploration}}^{T}\max_{i\in[N]}\sqrt{w_{i,t}}s_{i,t}\le\sqrt{\frac{2L_{1}}{\kappa\phi N}}\sum_{t=\Exploration}^{T}\frac{1}{\sqrt{t-1}}\le2\sqrt{\frac{2L_{1}T}{\kappa\phi N}}.
\]
\end{proof}

\subsection{Proof of Theorem~\ref{thm:regret_bound}}
\label{subsec:proof_of_regret_bound}
\begin{proof}
By Assumption~\ref{assump:boundedness} and~\ref{assump:mean_function}, 
\begin{equation}
R(T) \le \Exploration S L_1 + \sum_{t=\Exploration}^{T} \Regret{t},
\label{eq:regret_proof_exploration_decompose}
\end{equation}
almost surely.
Define the event $E_{\gamma}:=\bigcap_{t=\Exploration}^{T}\left\{ \Action t\in S_{t}\right\} \cap\bigcap_{t=1}^{T}\left\{ \SelectionP{\Action t}t>\gamma\right\}$.
By Lemma~\ref{lem:superunsaturated_arms} and~\eqref{eq:regret_decompostion} for any $x>0$,
\begin{align*}
\Probability\left(\sum_{t=\Exploration}^{T}\Regret t>x\right)\le&\Probability\left(\left\{ \sum_{t=\Exploration}^{T}\Regret t>x\right\} \cap E_{\gamma}\right)+\delta\\
\le&\Probability\left(\left\{ \sum_{t=\Exploration}^{T}2\max_{i\in[N]}\left\{ \Prederror it+\sqrt{\kappa L_{1}}\norm{\Context it}_{W_{t-1}^{-1}}\right\} >x\right\} \cap E_{\gamma}\right)+\delta.
\end{align*}
Denote $C_{L_1}:=2+4L_1$.
Applying Lemma~\ref{lem:prediction_error} and~\eqref{eq:regret_decompostion_norms} gives
\begin{align*}
\Probability&\left(\sum_{t=\Exploration}^{T}\Regret t>x\right)\\
&\le\Probability\left(\sum_{t=\Exploration}^{T}\left\{ 4C_{L_{1}}\sqrt{3L_{1}N\log\frac{2NT}{\delta}}\max_{i\in[N]}\sqrt{w_{i,t}}s_{i,t}+\frac{D_{\mu,B,\lambda,S}}{\phi^{3}\kappa^{3}\left(t-1\right)}\log\frac{4T}{\delta}\right\} >x\right)+9\delta\\
&\le\Probability\left(4C_{L_{1}}\sqrt{3L_{1}N\log\frac{2NT}{\delta}}\sum_{t=\Exploration}^{T}\max_{i\in[N]}\sqrt{w_{i,t}}s_{i,t}+\frac{D_{\mu,B,\lambda,S}}{\phi^{3}\kappa^{3}}\log T\log\frac{4T}{\delta}>x\right)+9\delta.
\end{align*}
Applying~\eqref{eq:elliptical_bound_phi} in Lemma~\ref{lem:elliptical_potential_lemma} gives
\begin{align*}
\Probability&\left(\sum_{t=\Exploration}^{T}\Regret t\!>\!x\right)\\
&\le\!\Probability\left(4C_{L_{1}}\sqrt{3L_{1}N\log\frac{2NT}{\delta}}\sqrt{\frac{8L_{1}T}{\kappa\phi N}}+\!\frac{D_{\mu,B,\lambda,S}}{\phi^{3}\kappa^{3}}\log T\log\frac{4T}{\delta}>x\right)+10\delta,
\end{align*}
and setting
\[
x=8C_{L_{1}}L_{1}\sqrt{\frac{6T}{\kappa\phi}\log\frac{2NT}{\delta}}+\frac{D_{\mu,B,\lambda,S}}{\phi^{3}\kappa^{3}}\log T\log\frac{4T}{\delta},
\]
proves the regret bound.
\end{proof}

\subsection{Relationship between the dimension and the minimum eigenvalue of contexts}
\label{subsec:phi_d_condition}
As \citet{kim2021doubly} pointed out, due to Assumption~\ref{assump:boundedness} and~\ref{assump:minimum_eigenvalue},
\[
d\phi\le d\Mineigen{\frac{1}{N}\sum_{i=1}^{N}\Expectation\left[\Context it\Context it^{T}\right]}\le\Trace{\frac{1}{N}\sum_{i=1}^{N}\Expectation\left[\Context it\Context it^{T}\right]}=\frac{1}{N}\sum_{i=1}^{N}\Expectation\left[\Trace{\Context it\Context it^{T}}\right]\le1.
\]
This implies $\phi^{-1} \ge d$.
However, \citet{bastani2021mostly} and \citet{kim2021doubly} identified the cases when $\phi^{-1}=O(d)$.
The cases include when the average of covariance of contexts across all arms has AR(1), tri-diagonal, block diagonal matrices.
They also cover well-known distributions for contexts, such as uniform distribution, truncated multivariate normal distribution on the unit ball of $\Real^{d}$.
In the following lemma, we present a condition which covers the aforementioned distributions.

\begin{lem}
\label{lem:phi_d_condition_appendix}
Let $p_i$ be the density for the marginal distribution of $X_i \in \Real^{d}$.
For $i\in[N]$, suppose that the density of contexts $p_i$ satisfies $0 < p_{\min} < p_i(x)$ for all $x$ such that $\|x\|_2 \le 1$.
Then we have
\[
\frac{1}{N}\sum_{i=1}^{N}\Expectation\left[X_{i}X_{i}^{T}\right]\succeq\left\{ \frac{p_{\min}\text{vol}\left(\mathcal{B}_{d}\right)}{\left(d+2\right)}\right\} I_{d},
\]
where $\mathcal{B}_d$ represents the $l_2$-unit ball in $\Real^{d}$.
\end{lem}

\begin{rem}
When the marginal distribution of $X_i$ is uniform, then $p_{\min} = 1/\text{vol}\left(\mathcal{B}_{d}\right)$ and thus $\phi^{-1}=d+2$.
For the truncated multivariate normal distribution with mean $0_d$ and covariance $\Sigma$, $p_{\min}=\exp\left(-\frac{\Mineigen{\Sigma}^{-1}-\Maxeigen{\Sigma}^{-1}}{2}\right)/\text{vol}\left(\mathcal{B}_{d}\right)$ and $\phi^{-1}=\left(d+2\right)\exp\left(\frac{\Mineigen{\Sigma}^{-1}-\Maxeigen{\Sigma}^{-1}}{2}\right)$.
\end{rem}

\begin{proof}
For each $i\in[N]$,
\[
\Expectation\left[X_{i}X_{i}^{T}\right]=\int_{\Real^{d}}xx^{T}p(x)dx\succeq p_{\min}\int_{\Real^{d}}xx^{T}dx=\left\{ \frac{p_{\min}\text{vol}\left(\mathcal{B}_{d}\right)}{\left(d+2\right)}\right\} I_{d},
\]
where the last equality holds due to Lemma 2 in \citet{bastani2021mostly}.
\end{proof}

\section{Proof of Theorem~\ref{thm:fast_regret_bound}}
\label{sec:fast_regret_bound_proof}

We first introduce the intuition and challenges to prove the logarithmic cumulative regret bound~\eqref{eq:fast_regret_bound}.
%\paragraph{Intuition to prove the regret bound:}
For each $i\in[N]$ and $t\in[T]$, let $\Prederror{i}{t}:=|\mu(\Context{i}{t}^{T}\Estimator{t-1})-\mu(\Context{i}{t}^{T}\beta^{*})|$ and
\begin{equation}
\mathcal{G}_{i}(t):= \Prederror{i}t+\Prederror{\Optimalarm t}t+\sqrt{\kappa L_{1}}\sqrt{\norm{\Context{i}t}_{W_{t-1}^{-1}}^{2}\!\!+\!\norm{\Context{\Optimalarm t}t}_{W_{t-1}^{-1}}^{2}},
\label{eq:threshold_value}
\end{equation} 
which is a threshold value for an arm $i$ to be super-unsaturated defined in~\eqref{eq:super_unsaturated_arms}.
Recall that $\Diff{i}{t}:=\mu(\Context{\Optimalarm{t}}{t}^{T}\beta^{*})-\mu(\Context{i}{t}^{T}\beta^{*})$ is the gap between the expected reward of arm $i$ and that of the optimal arm.
If the threshold value $\mathcal{G}_{j}(t)$ is less than the gap $\Diff{j}{t}$ for all $j \neq \Optimalarm{t}$, then all arms except for the optimal arm are {\it not} super-unsaturated.
In other words, the event
$ G_{t}:=\bigcap_{j\neq\Optimalarm t}\left\{ \Diff jt>\mathcal{G}_{j}(t)\right\} $
implies $S_t=\{\Optimalarm{t}\}$.
By resampling, \texttt{DDRTS-GLM} chooses the arm in $S_t$ with high probability (Lemma~\ref{lem:superunsaturated_arms}) and thus the optimal arm with high probability.
Thus, the instantaneous regret bound is bounded by
\begin{align*}
\Regret{t} &\le \mathcal{G}_{\Action{t}}(t)\Indicator{\Action{t}\neq\Optimalarm{t}}\\
&\le \mathcal{G}_{\Action{t}}(t)\Indicator{G_t^{c}} \\
&= \mathcal{G}_{\Action{t}}(t) \Probability\left(G_t^{c}\right) + \mathcal{G}_{\Action{t}}(t) \left\{ \Indicator{G_t^{c}} - \Probability\left(G_t^{c}\right) \right\}
\end{align*}
where the second inequality holds because the event $G_t$ implies that $\Action{t}=\Optimalarm{t}$ with high probability.
To bound the first term, we use
\begin{equation}
\Probability\left(G_{t}^{c}\right)= \Probability\left(\bigcup_{j\neq\Optimalarm t}\left\{ \Diff jt\le\mathcal{G}_{j}(t)\right\} \right) \le\Probability\left(\min_{j\neq\Optimalarm t}\Diff jt\le\max_{j\neq\Optimalarm t}\mathcal{G}_{j}(t)\right)\le h\max_{j\neq\Optimalarm t}\mathcal{G}_{j}(t),
\label{eq:G_t_probability_bound}
\end{equation}
to have
\[
\mathcal{G}_{\Action{t}}(t) \Probability\left(G_t^{c}\right) \le h \mathcal{G}_{\Action{t}}(t) \max_{j \neq \Optimalarm{t}} \mathcal{G}_{j}(t)
\le h \left\{ \max_{j\in[N]} \mathcal{G}_{i}(t) \right\}^{2}.
\]
To obtain an logarithmic cumulative bound, we need
\begin{equation}
\max_{j\in[N]}\mathcal{G}_{j}(t) = \tilde{O}(t^{-1/2}),
\label{eq:challenging_problem}
\end{equation} 
in terms of $t$.

Proving~\eqref{eq:challenging_problem} is challenging when only selected contexts are used.
Let $\widehat{\beta}^{A}_{t}$ be the solution to $\sum_{\tau=1}^{t}\left\{ Y_{\tau}-\mu\left(\Context{\Action{\tau}}{\tau}^{T}\beta\right)\right\} \Context{\Action{\tau}}{\tau}=0$, which is the score equation consists of selected contexts and rewards.
Denote the Gram matrix $A_t$ consists of selected contexts only, i.e., $A_t:=\sum_{\tau=1}^{t}\mu^{\prime}(\Context{\Action{\tau}}{\tau}^T\beta^{*})\Context{\Action{\tau}}{\tau}\Context{\Action{\tau}}{\tau}^{T}+\lambda I$.
Then, the threshold value for an arm $i$ to be super-unsaturated is
\[
\mathcal{G}_{i}^{A}(t):= \APrederror{i}{t} + \APrederror{\Optimalarm{t}}{t} +\sqrt{\kappa L_{1}}\sqrt{\norm{\Context{i}t}_{A_{t-1}^{-1}}^{2}\!\!+\!\norm{\Context{\Optimalarm t}t}_{A_{t-1}^{-1}}^{2}},
\]
where $\APrederror{i}{t}:= |\mu(\Context{i}{t}^T\widehat{\beta}_{t-1}^{A})-\mu(\Context{i}{t}^{T}\beta^{*})|$.
Lemma 3 in \citet{faury2020improved} proves
\[
\APrederror{i}{t} = O\left(\sqrt{d \log t}\right) \norm{\Context{i}{t}}_{A_{t-1}^{-1}},
\]
for all $t\in[T]$ and $i\in[N]$ and this implies
\[
\max_{j\in[N]} \mathcal{G}_{j}^{A}(t) =
O\left(\sqrt{d\log t}\right) \left( \max_{j\in[N]} \norm{\Context{j}{t}}_{A_{t-1}^{-1}} + \norm{\Context{\Optimalarm{t}}{t}}_{A_{t-1}^{-1}} \right).
\]
To prove $\max_{j\in[N]} \mathcal{G}_{j}^{A}(t) = \tilde{O}(t^{-1/2})$, we need
\begin{equation}
\max_{j\in[N]} \norm{\Context{j}{t}}_{A_{t-1}^{-1}} + \norm{\Context{\Optimalarm{t}}{t}}_{A_{t-1}^{-1}} = \tilde{O}\left(t^{-1/2}\right).
\label{eq:challenging_norm_bound}
\end{equation}
Obviously,~\eqref{eq:challenging_norm_bound} is implied by $\Mineigen{A_{t-1}} = \Omega (t)$.
However, proving a lower bound for the minimum eigenvalue is reported to be challenging (See Section 5 in \citet{li2017provably}).

We solve this challenging problem by developing the DDR estimator which uses contexts from all arms.
Instead of $A_t$, we use $W_t$ and prove $\Mineigen{W_t} = \Omega(Nt)$, which implies 
\[
\max_{j \in [N] } \|\Context{j}{t}\|_{W_{t-1}^{-1}}+\|\Context{\Optimalarm{t}}{t}\|_{W_{t-1}^{-1}} = O\left(\frac{1}{\sqrt{\phi N t}}\right).
\]
By Lemma~\ref{lem:prediction_error}, 
\[
\max_{j \in [N] } \mathcal{G}_{j}(t) = O(\sqrt{N \log NT}) \left(\max_{j \in [N]} \|\Context{j}{t}\|_{W_{t-1}^{-1}}+\|\Context{\Optimalarm{t}}{t}\|_{W_{t-1}^{-1}}\right)= O\left(\sqrt{\frac{\log NT}{\phi t}}\right).
\]
This proves~\eqref{eq:challenging_problem} and a logarithmic cumulative regret bound.
The detailed proof is as follows:

\begin{proof}
Let $\MarginExp \in (\Exploration,T]$ be the number of rounds for the exploration which will be defined later.
By Lemma~\ref{lem:superunsaturated_arms}, $\Action{t}$ is super-unsaturated for all $t\in(\MarginExp,T]$, and
\[
R(T)\le2L_{1}S^{*}\MarginExp+\sum_{t=\MarginExp}^{T}\Regret t=2L_{1}S^{*}\MarginExp+\sum_{t=\MarginExp}^{T}\Regret t\Indicator{\Action t\in S_{t}},
\]
with probability at least $1-\delta$.
Define an $\History{t}$-measurable event,
\[
G_{t}:=\bigcap_{i\neq\Optimalarm t}\left\{ \Diff jt>\mathcal{G}_{i}(t)\right\},
\]
where $\Diff{i}{t}:=\mu(\Context{\Optimalarm{t}}{t}^{T}\beta^{*})-\mu(\Context{i}{t}^{T}\beta^{*})$ and $\mathcal{G}_{i}(t)$ is defined in~\eqref{eq:threshold_value}.
By definition, when $G_t$ holds, then the super-unsaturated arm has only the optimal arm, i.e. $S_t:=\{\Optimalarm{t}\}$.
Thus, $\{\Action{t}\in S_t\}\cap G_t$ implies $\{\Action{t}=\Optimalarm{t}\}$, and
\begin{align*}
R(T)\le&2L_{1}S^{*}\MarginExp+\sum_{t=\MarginExp}^{T}\Regret t\Indicator{\Action t\in S_{t}}\Indicator{G_{t}}+\sum_{t=\MarginExp}^{T}\Regret t\Indicator{\Action t\in S_{t}}\Indicator{G_{t}^{c}}
\\=&2L_{1}S^{*}\MarginExp+\sum_{t=\MarginExp}^{T}\Regret t\Indicator{\Action t\in S_{t}}\Indicator{G_{t}^{c}}.
\end{align*}
Denote $C_{L_1}:=2+4L_1$.
Because $\Action{t}$ is super-unsaturated, we can use~\eqref{eq:regret_decompostion_norms} to have with probability at least $1-8\delta$,
\begin{equation}
\begin{split}
&\sum_{t=\MarginExp}^{T}\Regret t\Indicator{\Action t\in S_{t}}\Indicator{G_{t}^{c}}\\
&\le\sum_{t=\MarginExp}^{T}\left\{ 4C_{L_{1}}\sqrt{3L_{1}N\log\frac{2NT}{\delta}}\max_{i\in[N]}\sqrt{w_{i,t}}s_{i,t}+\frac{D_{\mu,B,\lambda,S}}{\phi^{3}\kappa^{3}\left(t-1\right)}\log\frac{4T}{\delta}\right\} \Indicator{G_{t}^{c}}\\
&\le\sum_{t=\MarginExp}^{T}\left\{ 4C_{L_{1}}L_{1}\sqrt{\frac{6}{\kappa\phi\left(t-1\right)}\log\frac{2NT}{\delta}}+\frac{D_{\mu,B,\lambda,S}}{\phi^{3}\kappa^{3}\left(t-1\right)}\log\frac{4T}{\delta}\right\} \Indicator{G_{t}^{c}}\\
&:=\sum_{t=\MarginExp}^{T}Q_{t}\Indicator{G_{t}^{c}},
\end{split}
\label{eq:Q_definition}
\end{equation}
where the second inequality holds due to~\eqref{eq:singleton_bound_phi}.
Define a filtration $\Filtration{t}:=\History{t}\cup\{\Action{t},\Reward{t}\}$.
Subtracting and adding $\CP{G_t^{c}}{\Filtration{t-1}}$ gives
\begin{align*}
\sum_{t=\MarginExp}^{T}Q_{t}\Indicator{G_{t}^{c}}=&\sum_{t=\MarginExp}^{T}Q_{t}\left\{ \Indicator{G_{t}^{c}}-\CP{G_{t}^{c}}{\Filtration{t-1}}\right\} +\sum_{t=\MarginExp}^{T}Q_{t}\CP{G_{t}^{c}}{\Filtration{t-1}}\\
\le&2\sqrt{\sum_{t=\MarginExp}^{T}Q_{t}^{2}\log\frac{1}{\delta}}+\sum_{t=\MarginExp}^{T}Q_{t}\CP{G_{t}^{c}}{\Filtration{t-1}},
\end{align*}
where the second inequality holds with probability at least $1-\delta$ due to the Azuma-Hoeffding inequality (Lemma~\ref{lem:Azuma_Hoeffding_inequality}).
By the margin condition (Assumption~\ref{assum:margin_condition}),
\begin{align*}
\CP{G_{t}^{c}}{\Filtration{t-1}}
=&\CP{\bigcup_{i\neq\Optimalarm t}\left\{ \Diff it\le\mathcal{G}_{i}(t)\right\} }{\Filtration{t-1}}\\
\le&\CP{ \min_{i\neq\Optimalarm t}\Diff it\le \max_{i\in[N]}\mathcal{G}_{i}(t) }{\Filtration{t-1}}\\
\le&\CP{\min_{i\neq\Optimalarm t}\Diff it\le Q_{t}}{\Filtration{t-1}}\\
\le&hQ_{t},
\end{align*}
where the last inequality holds because of
\begin{align*}
\max_{i\in[N]}\mathcal{G}_{i}(t)\le&2\max_{i\in[N]}\left\{ \Prederror it+\sqrt{\kappa L_{1}}\norm{\Context it}_{W_{t-1}^{-1}}\right\}\\
\le&2\left\{ \mu^{\prime}\left(\Context it^{T}\beta^{*}\right)C_{L_{1}}\sqrt{3N\log\frac{2NT}{\delta}}+\sqrt{\kappa L_{1}}\right\} \max_{i\in[N]}s_{i,t}+\frac{D_{\mu,B,\lambda,S}}{\phi^{3}\kappa^{3}\left(t-1\right)}\log\frac{4T}{\delta}\\
\le&4C_{L_{1}}\sqrt{3N\log\frac{2NT}{\delta}}\max_{i\in[N]}\sqrt{w_{i,t}}s_{i,t}+\frac{D_{\mu,B,\lambda,S}}{\phi^{3}\kappa^{3}\left(t-1\right)}\log\frac{4T}{\delta}\\
\le&Q_{t},
\end{align*}
and definition of $Q_t$ in~\eqref{eq:Q_definition}.
To use the margin condition, we need $Q_t \le \rho_0$.
To verify this, we set
\begin{equation}
\MarginExp:=\frac{6\left(4C_{L_{1}}L_{1}+D_{\mu,B,\lambda,S}\right)^{2}}{\phi^{4}\kappa^{4}\rho_{0}^{2}}\log\frac{2NT}{\delta},
\label{eq:tau_0}
\end{equation}
which gives
\begin{align*}
Q_{t}\le&\frac{1}{\sqrt{t-1}}\left\{ 4C_{L_{1}}L_{1}\sqrt{\frac{6}{\kappa\phi}\log\frac{2NT}{\delta}}+\frac{D_{\mu,B,\lambda,S}}{\phi^{3}\kappa^{3}\Exploration}\log\frac{4T}{\delta}\right\} \\
\le&\frac{1}{\sqrt{t-1}}\left\{ 4C_{L_{1}}L_{1}\sqrt{\frac{6}{\kappa\phi}\log\frac{2NT}{\delta}}+\frac{D_{\mu,B,\lambda,S}}{\phi^{2}\kappa^{2}}\sqrt{\frac{\kappa}{Nd}\log\frac{4T}{\delta}}\right\} \\
\le&\frac{4C_{L_{1}}L_{1}+D_{\mu,B,\lambda,S}}{\phi^{2}\kappa^{2}\sqrt{t-1}}\sqrt{6\log\frac{2NT}{\delta}}\\
\le& \rho_0,
\end{align*}
for $t\in(\MarginExp,T]$.
The second inequality holds due to the definition of $\Exploration$ in~\eqref{eq:exploration_term}.
Thus we have with probability at least $1-10\delta$,
\begin{equation}
R(T)\le2L_{1}S\MarginExp+2\sqrt{\sum_{t=\MarginExp}^{T}Q_{t}^{2}\log\frac{1}{\delta}}+h\sum_{t=\MarginExp}^{T}Q_{t}^{2}.
\label{eq:Q_t_fast_rate}
\end{equation}
Note that
\begin{align*}
\sum_{t=\MarginExp}^{T}Q_{t}^{2}\le&2\sum_{t=\MarginExp}^{T}\frac{96C_{L_{1}}^{2}L_{1}^{2}}{\kappa\phi\left(t-1\right)}\log\frac{2NT}{\delta}+\frac{2D_{\mu,B,\lambda,S}^{2}}{\phi^{6}\kappa^{6}\left(t-1\right)^{2}}\log^{2}\frac{4T}{\delta}\\
\le&\frac{192L_{1}^{2}C_{L_{1}}^{2}}{\kappa\phi}\log T\log\frac{2NT}{\delta}+\frac{2D_{\mu,B,\lambda,S}^{2}}{\phi^{6}\kappa^{6}\MarginExp}\log^{2}\frac{4T}{\delta}\\
\le&\frac{192L_{1}^{2}C_{L_{1}}^{2}}{\kappa\phi}\log T\log\frac{2NT}{\delta}+\frac{2\rho_{0}^{2}}{\phi^{2}\kappa^{2}}\log\frac{4T}{\delta}.
\end{align*}
Thus with probability at least $1-10\delta$,
\[
R(T) \le 2L_{1}S\MarginExp+O\left(\phi^{-2}\kappa^{-2}\log NT\right)+\frac{192hL_{1}^{2}C_{L_{1}}^{2}}{\kappa\phi}\log T\log\frac{2NT}{\delta}.
\]
\end{proof}

\section{Proof of an error bound for proposed estimators}
\label{sec:general_estimator_bound}

In this section, we provide a lemma which implies error bounds for both the imputation estimator $\Impute{t}$ and for DDR estimator.

\begin{lem}
\label{lem:general_DR_bound}
Suppose the Assumptions 1-5 and the event in~\eqref{eq:action_in_superunsaturated_arms} holds with $\gamma \in [1/(N+1),1/N)$.
At round $t\in(\Exploration,T]$, let $\Estimator{t}$ be the solution of
\[
U_{t}(\beta):=\sum_{\tau=1}^{t}\sum_{i=1}^{N}\left\{ \DRreward i\tau{\breve{\beta}}-\mu(\Context i\tau^{T}\beta)\right\} \Context i\tau-\lambda\beta=0,
\]
where $\DRreward{i}{\tau}{\breve{\beta}}$ is the pseudo reward defined in~\eqref{eq:pseudo_reward} whose the imputation estimator is $\breve{\beta}$.
%The action $\Action{u}$ in the pseudo reward is an arm chosen by \texttt{DDRTS-GLM}.
Suppose $\breve{\beta} \in \mathcal{B} \subset \mathcal{B}_{r}^{*}$ in Assumption~\ref{assump:mean_function} and let $\mathcal{N}(\epsilon, \|\cdot\|_2, \mathcal{B})$ be the $\epsilon$-covering number of $\mathcal{B}$.
Let $\rho := \max_{\beta\in\mathcal{B}} \|\beta-\beta^{*}\|_2$.
%Then for each $t\in(\Exploration,T]$, under the event $E:=\cap_{t\in[T]}\{\SelectionP{\Action{t}}{t} > \gamma\}$ for some $\gamma \in (1/(N+1),1/N]$, 
Then with probability at least $1-3\delta/T$,
\[
\norm{\Estimator t-\beta^{*}}_{2}\le\frac{3L_{1}\epsilon}{\kappa \phi}+\frac{6\rho L_{1}}{\kappa \phi\sqrt{t}}\sqrt{8\log\frac{4T\mathcal{N}(\epsilon,\|\cdot\|_{2},\mathcal{B})}{\delta}}+\frac{8B}{\kappa\phi\sqrt{t}}\sqrt{\log\frac{4T}{\delta}}+\frac{\lambda S^{*}}{\kappa\phi Nt}.
\]
for any $\lambda > 0$ and $\epsilon>0$.
\end{lem}

\begin{proof}
Denote the event
\[
E_{\gamma}:= \bigcap_{t=\Exploration}^{T}\left\{ \Action t\in S_{t}\right\} \cap\bigcap_{t=1}^{T}\left\{ \SelectionP{\Action t}t>\gamma\right\},
\]
and let us fix $t\in(\Exploration,T]$ throughout the proof.
\paragraph{Step 1. The inverse map of $U_t$ is bounded.}
For $b_{t}\in(0,r)$ set $\mathcal{B}_{b_{t}}^{*}:=\{\beta:\norm{\beta-\beta^{*}}_{2}\le b_{t}\}$, where $b_{t}$ will be specified later. 
We first prove that $U_{t}:\mathcal{B}_{b_{t}}^{*}\to\Real^{d}$ is an injective function.
Since $U_{t}$ is differentiable, the mean value theorem implies that for any $\beta_{1},\beta_{2}\in\mathcal{B}_{b_{t}}^{*}$, there exists $v\in(0,1)$ such that $\bar{\beta}=v\beta_{1}+(1-v)\beta_{2}$ and,
\[
U_{t}(\beta_{1})-U_{t}(\beta_{2})=-\left\{ \sum_{\tau=1}^{t}\sum_{i=1}^{N}\mu^{\prime}\left(\Context i{\tau}^{T}\bar{\beta}\right)\Context i{\tau}\Context i{\tau}^{T}+\lambda I\right\} \left(\beta_{1}-\beta_{2}\right).
\]
Since $\bar{\beta}\in\mathcal{B}_{b_{t}}^{*}$, we have $\mu^{\prime}(x^{T}\bar{\beta})\ge\kappa>0$
for all $x$ such that $\norm x_{2}\le1$, by Assumption~\ref{assump:mean_function}. 
Thus, we have
\[
\left\{ U_{t}(\beta_{1})-U_{t}(\beta_{2})\right\} ^{T}\left(\beta_{1}-\beta_{2}\right)\le-\lambda \norm{\beta_{1}-\beta_{2}}_{2}<0,
\]
whenever $\beta_{1}\neq\beta_{2}$. Thus, $U_{t}:\mathcal{B}_{b_{t}}^{*}\to\Real^{d}$
is an injective function. 
Using Lemma~\ref{lem:smooth_injection}, we have
\begin{equation}
\inf_{\beta:\norm{\beta-\beta^{*}}_{2}=b_{t}}\norm{U_{t}(\beta)-U_{t}(\beta^{*})}_{2}\ge y,\;\norm{U_{t}(\Estimator t)-U_{t}(\beta^{*})}_{2}\le y\;\Longrightarrow\norm{\Estimator t-\beta^{*}}_{2}\le b_{t}.
\label{eq:estimation_error_bound_chen}
\end{equation}
for any $y>0$. Using mean value Theorem, for any $\beta$ such that $\norm{\beta-\beta^{*}}_{2}=b_{t}$, there exists $v\in(0,1)$ such that $\bar{\beta}=v\beta+(1-v)\beta^{*}$ 
\begin{align*}
\norm{U_{t}(\beta)-U_{t}(\beta^{*})}_{2}= & \norm{\left\{ \sum_{\tau=1}^{t}\sum_{i=1}^{N}\mu^{\prime}\left(\Context i{\tau}^{T}\bar{\beta}\right)\Context i{\tau}\Context i{\tau}^{T}+\lambda I_{d}\right\} \left(\beta-\beta^{*}\right)}_{2}\\
\ge & \Mineigen{\sum_{\tau=1}^{t}\sum_{i=1}^{N}\mu^{\prime}\left(\Context i{\tau}^{T}\bar{\beta}\right)\Context i{\tau}\Context i{\tau}^{T}+\lambda I_{d}}\norm{\beta-\beta^{*}}_{2}\\
\ge & \frac{b_{t}\kappa N\phi t}{2},
\end{align*}
where the last inequality holds with probability at least $1-\delta/T$, using Corollary~\ref{cor:min_eigen_chernoff} with the fact that
$t \ge \Exploration \ge 32 N^{-1} \phi^{-2} \log(4T/\delta)$ and $\bar{\beta}\in \mathcal{B}_{r}^{*}$.
Thus we can set $y=\frac{b_{t}\kappa N \phi t}{2}$ in~\eqref{eq:estimation_error_bound_chen} to have
\begin{equation}
\norm{U_{t}(\Estimator t)-U_{t}(\beta^{*})}_{2}\le\frac{b_{t}\kappa N\phi t}{2}\;\Longrightarrow\norm{\Estimator t-\beta^{*}}_{2}\le b_{t}
\label{eq:inverse_map_bound}
\end{equation}

\paragraph{Step 2. Bounding the norm of the score function.}
Since $U_{t}(\Estimator t)=0$, we only need to bound,
\[
\norm{U_{t}(\Estimator t)-U_{t}(\beta^{*})}_{2} =  \norm{U_{t}(\beta^{*})}_{2}.
\]
By definition of $U_t$,
\begin{equation}
\begin{split}
\norm{U_{t}(\beta^{*})}_{2}=&\norm{\sum_{\tau=1}^{t}\sum_{i=1}^{N}\left(\DRreward i\tau{\breve{\beta}}-\mu(\Context i\tau^{T}\beta^{*})\right)\Context i\tau-\lambda\beta^{*}}_{2}\\
\le&\norm{\sum_{\tau=1}^{t}\sum_{i=1}^{N}\left\{ 1-\frac{\Indicator{\Action \tau=i}}{\SelectionP i\tau}\right\} D_{i,\tau}^{\breve{\beta}}\Context i\tau+\frac{\Indicator{\Action \tau=i}}{\SelectionP i\tau}\Error \tau\Context i\tau}_{2}+\lambda\norm{\beta^{*}}_{2}\\
\le&\norm{\sum_{\tau=1}^{t}\sum_{i=1}^{N}\left\{ 1-\frac{\Indicator{\Action \tau=i}}{\SelectionP i\tau}\right\} D_{i,\tau}^{\breve{\beta}}\Context i\tau}_{2}+\norm{\sum_{\tau=1}^{t}\frac{\Error \tau}{\SelectionP{\Action \tau}\tau}\Context i\tau}_{2}+\lambda S^{*},
\end{split}
\label{eq:U_t_score_decomposition}
\end{equation}
where $D_{i,\tau}^{\breve{\beta}}:=\mu(\Context i\tau^{T}\breve{\beta})-\mu(\Context i\tau^{T}\beta^{*})$, $\Error{\tau}:=Y_{\tau}-\mu(\Context{\Action{\tau}}{\tau}^{T}\beta^{*})$ and $A_{i,\tau}=1-\frac{\Indicator{\Action \tau=i}}{\SelectionP i\tau}$.

In \citet{kim2021doubly}, the imputation estimator is $\History{\tau}$-measurable and naturally adaptive to the filtration $\History{\tau}$ which enables to use the martingale inequality directly.
However, in our analysis $\breve{\beta}$ is possibly dependent with $\Action{\tau}$, which leads to the dependency between $A_{i,\tau}$ and $D^{\breve{\beta}}_{i,\tau}$.
Thus, we cannot use the vector martingale inequality directly to~\eqref{eq:U_t_score_decomposition}.

To handle this, we bring the $\epsilon$-covers of $\mathcal{B}$ which includes $\breve{\beta}$.
For convenience, we write the covering number $\mathcal{N}(\epsilon, \|\cdot\|_2, \mathcal{B})$ by $K(\epsilon)$.
For any $\beta \in \mathcal{B}$, there exists $n \in [K(\epsilon)]$ such that $\norm{\beta_n-\beta}_2 \le \epsilon$.
Thus, for each $\beta \in \mathcal{B}$,
\begin{align*}
\norm{\sum_{\tau=1}^{t}\sum_{i=1}^{N}A_{i,\tau}D_{i,\tau}^{\beta}\Context i\tau}_{2}\le&\norm{\sum_{\tau=1}^{t}\sum_{i=1}^{N}A_{i,\tau}\left(D_{i,\tau}^{\beta}-D_{i,\tau}^{\beta_{n}}\right)\Context i\tau}_{2}+\norm{\sum_{\tau=1}^{t}\sum_{i=1}^{N}A_{i,\tau}D_{i,\tau}^{\beta_{n}}\Context i\tau}_{2}\\
\le&L_{1}\sum_{\tau=1}^{t}\sum_{i=1}^{N}\abs{A_{i,\tau}}\norm{\Context i\tau^{T}\left(\beta-\beta_{n}\right)}_{2}+\norm{\sum_{\tau=1}^{t}\sum_{i=1}^{N}A_{i,\tau}D_{i,\tau}^{\beta_{n}}\Context i\tau}_{2}\\
\le&L_{1}\epsilon\sum_{\tau=1}^{t}\left(N+\frac{1}{\SelectionP{\Action{\tau}}{\tau}}\right)+\norm{\sum_{\tau=1}^{t}\sum_{i=1}^{N}A_{i,\tau}D_{i,\tau}^{\beta_{n}}\Context i\tau}_{2}\\
\le&L_{1}\epsilon t\left(N+\gamma^{-1}\right)+\norm{\sum_{\tau=1}^{t}\sum_{i=1}^{N}A_{i,\tau}D_{i,\tau}^{\beta_{n}}\Context i\tau}_{2}.
\end{align*}
The second inequality holds due to the mean value theorem and Assumption~\ref{assump:mean_function}.
The third inequality holds due to Assumption~\ref{assump:boundedness} and the definition of $A_{i,\tau}$, and the last inequality holds due to $\SelectionP{\Action{\tau}}{\tau}>\gamma$ in event $E_\gamma$.
Thus we have for any $\beta \in \mathcal{B}$,
\[
\norm{\sum_{\tau=1}^{t}\sum_{i=1}^{N}A_{i,\tau}D_{i,\tau}^{\beta}\Context i\tau}\le L_{1}\epsilon t\left(N+\gamma^{-1}\right)+\max_{n\in[K(\epsilon)]}\norm{\sum_{\tau=1}^{t}\sum_{i=1}^{N}A_{i,\tau}D_{i,\tau}^{\beta_{n}}\Context i\tau}_{2}.
\]
Since $\breve{\beta}\in\mathcal{B}$,
\begin{equation}
\begin{split}
\norm{\sum_{\tau=1}^{t}\sum_{i=1}^{N}A_{i,\tau}D_{i,\tau}^{\breve{\beta}}\Context i\tau}_{2}&\le\sup_{\beta\in\mathcal{B}}\norm{\sum_{\tau=1}^{t}\sum_{i=1}^{N}A_{i,\tau}D_{i,\tau}^{\beta}\Context i\tau}_{2}\\
&\le L_{1}\epsilon t\left(N+\gamma^{-1}\right)+\max_{n\in[K(\epsilon)]}\norm{\sum_{\tau=1}^{t}\sum_{i=1}^{N}A_{i,\tau}D_{i,\tau}^{\beta_{n}}\Context i\tau}_{2}
\label{eq:1st_term_covering_decomposition}
\end{split}
\end{equation}
For each $n\in[K(\epsilon)]$, and $u\in[t]$,
\[
\CE{\sum_{i=1}^{N}A_{i,\tau}D_{i,\tau}^{\beta_{n}}\Context i\tau}{\History{\tau}}=\sum_{i=1}^{N}\CE{\left(1-\frac{\Indicator{\Action{\tau}=i}}{\SelectionP i{\tau}}\right)}{\History{\tau}}D_{i,\tau}^{\beta_{n}}\Context i\tau=0.
\]
Let
\[
M_{u,n}:=M_{u}(\beta_{n})=\sum_{\tau=1}^{u}\sum_{i=1}^{N}A_{i,\tau}D_{i,\tau}^{\beta_{n}}\Context i\tau,
\]
and the sequence $\{M_{u,n}\}_{u=0}^{t}$ is a $\Real^{d}$-valued martingale adapted to $\{\History{u}\}_{u=1}^{t+1}$.

Since the differences of the martingale is bounded on the event $E_{\gamma}$, not almost surely, we cannot directly apply Lemma~\ref{lem:Hibert_concentration}.
Instead, we use Lemma~\ref{lem:chung_lemma} to obtain the bound when the difference is bounded with high probability.
Since $(\Real^d,\norm{\cdot}_2)$ is a Hilbert space, we can use Lemma~\ref{lem:dimension_reduction} to find an $\Real^{2}$-valued martingale $\{N_{u,n}\}_{u=0}^{t}$ such that
\begin{equation}
\norm{N_{u,n}}_{2}=\norm{M_{u,n}}_{2}, \quad \norm{N_{u,n}-N_{u-1,n}}_{2}=\norm{M_{u,n}-M_{u-1,n}}_{2},
\label{eq:martingale_projection}
\end{equation}
for all $u\in[t]$.
Set $N_{u,n}=(N_{u,n}^{(1)},N_{u,n}^{(2)})^{T}$.
Then for each $r=1,2$ and $u\in[t]$
\begin{align*}
\abs{N_{u,n}^{(r)}-N_{u-1,n}^{(r)}}\le&\norm{N_{u,n}-N_{u-1,n}}_{2}\\
=&\norm{M_{u,n}-M_{u-1,n}}_{2}\\
\le&\sum_{i=1}^{N}\abs{1-\frac{\Indicator{\Action \tau=i}}{\SelectionP i\tau}}L_{1}\norm{\beta_{n}-\beta^{*}}_{2}\\
\le&\left(N+\SelectionP{\Action u}u^{-1}\right)L_{1} \rho,
\end{align*}
where the last inequality holds due to $\beta_n \in\mathcal{B}$, and $\rho := \max_{\beta\in\mathcal{B}} \|\beta-\beta^{*}\|_2$.
By Lemma~\ref{lem:chung_lemma}, there exists a $\Real^2$-valued martingale $\{P_{u,n}\}_{u=0}^{t}$ such that
\begin{equation}
\abs{P_{u,n}^{(r)}-P_{u-1,n}^{(r)}}\le \left(N+\gamma^{-1}\right)L_{1}\rho,\quad\{P_{u,n}\neq N_{u,n}\}\subset\{\SelectionP{\Action u}u\le\gamma\}\subseteq E_{\gamma}^{c}
\label{eq:transform_to_bounded_martingale}
\end{equation}
for all $u\in[t]$.
Now we can use~\eqref{eq:martingale_projection},~\eqref{eq:transform_to_bounded_martingale} and Lemma~\ref{lem:Hibert_concentration} to have for any $x>0$,
\begin{align*}
\Probability\left(\left\{ \max_{n\in[K(\epsilon)]}\norm{\sum_{\tau=1}^{t}\sum_{i=1}^{N}A_{i,\tau}D_{i,\tau}^{\beta_{n}}\Context i\tau}_{2}>x\right\} \cap E_{\gamma}\right)\le&\sum_{n=1}^{K(\epsilon)}\Probability\left(\left\{ \norm{M_{t,n}}_{2}>x\right\} \cap E_{\gamma}\right)\\
=&\sum_{n=1}^{K(\epsilon)}\Probability\left(\left\{ \norm{N_{t,n}}_{2}>x\right\} \cap E_{\gamma}\right)\\
\le&\sum_{n=1}^{K(\epsilon)}\Probability\left(\norm{P_{t,n}}_{2}>x\right)\\
\le&4K(\epsilon)\exp\left(-\frac{x^{2}}{8t\left(N+\gamma^{-1}\right)^{2}\rho^{2}L_{1}^{2}}\right).
\end{align*}
Thus, under the event $E_{\gamma}$, with probability at least $1-\delta/T$, we have the following bound for the first term in~\eqref{eq:U_t_score_decomposition}:
\begin{equation}
\max_{n\in[K(\epsilon)]}\norm{\sum_{\tau=1}^{t}\sum_{i=1}^{N}A_{i,\tau}D_{i,\tau}^{\beta_{n}}\Context i\tau}_{2}\le\left(N+\gamma^{-1}\right)\rho L_{1}\sqrt{8t\log\frac{4TK(\epsilon)}{\delta}}.
\label{eq:1st_bound_summary}
\end{equation}

To bound the second term in~\eqref{eq:U_t_score_decomposition}, Assumption~\ref{assump:bounded_rewards} gives $\abs{\Error{\tau}} \le  2B$, and
\begin{align*}
\norm{\frac{\Error \tau\Indicator{\SelectionP{\Action \tau}\tau>\gamma}}{\SelectionP{\Action \tau}\tau}\Context{\Action \tau}\tau}_{2}&\le2\gamma^{-1}B,\\
\CE{\frac{\Error \tau\Indicator{\SelectionP{\Action \tau}\tau>\gamma}}{\SelectionP{\Action \tau}\tau}\Context{\Action \tau}\tau}{\History \tau,\Action \tau}&=\CE{\Error \tau}{\History \tau,\Action \tau}\frac{\Indicator{\SelectionP{\Action \tau}\tau>\gamma}}{\SelectionP{\Action \tau}\tau}\Context{\Action \tau}\tau=0,
\end{align*}
for all $\tau\in[t]$.
Thus the sequence
\[
\left\{ \sum_{\tau=1}^{u}\frac{\Error{\tau}\Indicator{\SelectionP{\Action{\tau}}{\tau}>\gamma}}{\SelectionP{\Action{\tau}}{\tau}}\Context i{\tau}\right\}_{u=0}^{t},
\]
is a $\Real^{d}$-valued martingale sequence with respect to $\{\History{u}\cup{\Action{u}}\}_{u=0}^{t}$.
By Lemma~\ref{lem:Hibert_concentration}, for any $x>0$,
\begin{align*}
\Probability\left(\left\{ \norm{\sum_{\tau=1}^{t}\frac{\Error \tau}{\SelectionP{\Action \tau}\tau}\Context{\Action \tau}\tau}_{2}>x\right\} \cap E_{\gamma}\right)\le
&\Probability\left(\norm{\sum_{\tau=1}^{t}\frac{\Error \tau\Indicator{\SelectionP{\Action \tau}\tau>\gamma}}{\SelectionP{\Action \tau}\tau}\Context{\Action \tau}\tau}_{2}>x\right)\\
\le&4\exp\left(-\frac{x^{2}}{16t\gamma^{-1}B}\right).
\end{align*}
Thus, under the event $E_{\gamma}$, with probability at least $1-\delta/T$,
\begin{equation}
\norm{\sum_{\tau=1}^{t}\frac{\Error \tau}{\SelectionP{\Action \tau}\tau}\Context{\Action \tau}\tau}_{2} \le 
\frac{4B}{\gamma}\sqrt{t\log\frac{4T}{\delta}}.
\label{eq:2nd_bound_summary}
\end{equation}
Thus, by~\eqref{eq:U_t_score_decomposition},~\eqref{eq:1st_term_covering_decomposition},~\eqref{eq:1st_bound_summary} and~\eqref{eq:2nd_bound_summary}, under the event $E_{\gamma}$ and with probability at least $1-2\delta/T$
\begin{equation}
\norm{U_{t}\left(\beta^{*}\right)}_{2}\le L_{1}\epsilon t\left(N+\gamma^{-1}\right)+\left(N+\gamma^{-1}\right)\rho L_{1}\sqrt{8t\log\frac{4TK(\epsilon)}{\delta}}+\frac{4B}{\gamma}\sqrt{t\log\frac{4T}{\delta}}+\lambda S^{*}.
\label{eq:U_t_bound}
\end{equation}

\paragraph{Step 3. Applying the bound to~\eqref{eq:inverse_map_bound}}
Setting
\[
b_{t}:=\frac{\left(N+\gamma^{-1}\right)\epsilon L_{1}}{\kappa N\phi}+\frac{2\left(N+\gamma^{-1}\right)\rho L_{1}}{\kappa N\phi\sqrt{t}}\sqrt{8\log\frac{4TK(\epsilon)}{\delta}}+\frac{4B}{\kappa N\gamma\phi\sqrt{t}}\sqrt{\log\frac{4T}{\delta}}+\frac{\lambda S^{*}}{\kappa\phi Nt},
\]
in~\eqref{eq:inverse_map_bound} gives the bound $\|\Estimator{t}-\beta^{*}\|_2 \le b_t$ with probability at least $1-3\delta/T$, under the event $E_{\gamma}$.
Since $\gamma > 1/(N+1)$, we have
\begin{align*}
\norm{\Estimator t-\beta^{*}}_{2}\le&b_{t}\\
\le&\frac{3L_{1}\epsilon}{\kappa \phi}+\frac{6\rho L_{1}}{\kappa \phi\sqrt{t}}\sqrt{8\log\frac{4TK(\epsilon)}{\delta}}+\frac{8B}{\kappa\phi\sqrt{t}}\sqrt{\log\frac{4T}{\delta}}+\frac{\lambda S^{*}}{\kappa\phi Nt}.
\end{align*}
\end{proof}

\subsection{An error bound for the imputation estimator}
\label{subsec:imputation_estimator_bound}

Now we are ready to prove the tight estimation error bound for the imputation estimator stated in~\eqref{eq:imputation_estimator_condition}.

\begin{lem}
\label{lem:imputation_error_bound}
Suppose the Assumptions 1-5 and the event in~\eqref{eq:action_in_superunsaturated_arms} holds with $\gamma \in [1/(N+1),1/N)$.
Let $\Impute{t}$ be the imputation estimator for \texttt{DDRTS-GLM}.
Set
\begin{equation}
\Exploration:=\max\left\{ \frac{Nd}{\kappa},r^{2}\right\} \frac{\left(\frac{9}{4} L_{1} S +48L_1 S\sqrt{d}+8B+\frac{\lambda}{N}S^{*}\right)^{2}}{\kappa^{2}\phi^{2}}\log\frac{4T}{\delta}
\label{eq:exploration_term}
\end{equation}
Then for each $t\in(\Exploration,T]$, with probability at least $1-3\delta/T$,
\[
\norm{\Impute{t}-\beta^{*}}_2 \le \min\left\{ \sqrt{\frac{\kappa}{Nd}}, \: r \right\}.
\]
\end{lem}

\begin{proof}
Since $\Impute{t}$ is a solution of
\[
\sum_{\tau=1}^{t}\sum_{i=1}^{N}\left\{ \DRreward i\tau{\NMLE t}-\mu(\Context i\tau^{T}\beta)\right\} \Context i\tau-\lambda\beta=0,
\]
and $\NMLE{t} \in \mathcal{B}_S := \{\beta \in \Real^{d}: \|\beta\|_2 \le S\}$, we can apply Lemma~\ref{lem:general_DR_bound} with
\begin{align*}
\rho&:=\max_{\beta:\norm{\beta}_{2}\le S}\norm{\beta-\beta^{*}}\le2S,\\
\epsilon&:=\frac{3S}{4t},\\
\mathcal{N}(\epsilon,\|\cdot\|_2,\mathcal{B}_S)&\le\left(1+\frac{2S}{\epsilon}\right)^{d}\le\left(\frac{3S}{\epsilon}\right)^{d}\le\left(4T\right)^{d},
\end{align*}
where the first inequality of the last line is from a bound for covering number of $\|\cdot\|_2$-ball of radius $2S$.
Applying Lemma~\ref{lem:general_DR_bound}, we have
\begin{align*}
\norm{\Impute t-\beta^{*}}_{2}\le&\frac{\frac{9}{4}L_{1}S}{\kappa\phi t}+\frac{48S L_1}{\kappa\phi\sqrt{t}}\sqrt{d\log\frac{4T}{\delta}}+\frac{8B}{\kappa\phi\sqrt{t}}\sqrt{\log\frac{4T}{\delta}}+\frac{\lambda S^{*}}{\kappa\phi Nt}\\
\le&\frac{\frac{9}{4} L_{1} S +48L_1 S\sqrt{d}+8B+\frac{\lambda}{N}S^{*}}{\kappa\phi\sqrt{t}}\sqrt{\log\frac{4T}{\delta}}.
\end{align*}
where the second inequality holds due to $\gamma > 1/(N+1)$.
Setting the last term smaller than $\min\{\sqrt{\kappa/(Nd)},r\}$ proves the Lemma.
\end{proof}

\subsection{An error bound for DDR estimator}
With the bound of the imputation estimator in Lemma~\ref{lem:imputation_error_bound}, we can derive a fast rate of the error bound of DDR estimator.
\citet{bastani2020online} prove a tail inequality with $1/\sqrt{t}$ rate under the setting where the contexts are same over all arms.
\citet{kim2021doubly} prove a tail inequality under the linear contextual bandit setting.
Under our setting, Theorem~\ref{thm:estimation_error_bound} is the first tail bound in GLB.
The tail inequality is used to bound the second-order terms by $O(1/t)$ that arises in the proof of Lemma~\ref{lem:prediction_error}.

\begin{thm}
\label{thm:estimation_error_bound} 
(A tail inequality for DDR estimator) 
Suppose the Assumptions 1-5 and the event in~\eqref{eq:action_in_superunsaturated_arms} holds with $\gamma \in [1/(N+1),1/N)$.
%Let $\Exploration$ be the number of rounds required for the imputation estimator $\Impute{t}$ to satisfy~\eqref{eq:imputation_estimator_condition}.
For each $t \in (\Exploration,T]$, let $\Estimator t$ be the DDR estimator defined by the solution of~\eqref{eq:score_equation}.
%Then under the event $E:=\cap_{t\in[T]}\{\SelectionP{\Action{t}}{t} >\gamma \}$ for $\gamma \in (1/(N+1),1/N]$, 
Then with probability at least $1-6\delta/T$,
\begin{equation}
\norm{\Estimator t-\beta^{*}}_{2} \!\le\! \frac{12L_{1}+8B}{\kappa\phi\sqrt{t}}\sqrt{\log\frac{4T}{\delta}}+\frac{\frac{9}{4}L_{1}r+\frac{\lambda S^{*}}{N}}{\kappa\phi t},
\label{eq:estimation_error_bound}
\end{equation}
for any $\lambda>0$.
\end{thm}

\begin{proof}
The DDR estimator is the solution of
\[
\sum_{\tau=1}^{t}\sum_{i=1}^{N}\left\{ \DRreward i\tau{\Impute{t} t}-\mu(\Context i\tau^{T}\beta)\right\} \Context i\tau-\lambda\beta=0,
\]
By Lemma~\ref{lem:imputation_error_bound}, $\Impute{t} \in \breve{\mathcal{B}}:=\{\beta: \|\beta-\beta^{*}\|_2 \le \min\{\kappa^{1/2} N^{-1/2}d^{-1/2}, r \}\}$ with probability at least $1-3\delta/T$.
Thus, we can apply Lemma~\ref{lem:general_DR_bound} with
\begin{align*}
\rho&:=\max_{\beta\in\breve{\mathcal{B}}}\norm{\beta-\beta^{*}}_{2}\le\min\left\{ \sqrt{\frac{\kappa}{Nd}},r\right\},\\
\epsilon&:=\frac{3}{4t}\sqrt{\frac{\kappa}{Nd}},\\
\mathcal{N}\left(\epsilon,\norm{\cdot}_{2}\breve{\mathcal{B}}\right)&\le\left(1+\frac{2}{\epsilon}\min\left\{ \sqrt{\frac{\kappa}{Nd}},r\right\} \right)^{d}\le\left(\frac{3}{\epsilon}\sqrt{\frac{\kappa}{Nd}}\right)^{d}\le\left(4T\right)^{d}.
\end{align*}
to have with probability at least $1-6\delta/T$,
\begin{align*}
\norm{\Estimator t-\beta^{*}}_{2}\le&\frac{\frac{9}{4}L_{1}}{\kappa\phi t}\min\left\{ \sqrt{\frac{\kappa}{Nd}},r\right\} +\frac{12L_{1}}{\kappa\phi\sqrt{t}}\min\left\{ \sqrt{\frac{\kappa}{Nd}},r\right\} \sqrt{d\log\frac{4T}{\delta}}+\frac{8B}{\kappa\phi\sqrt{t}}\sqrt{\log\frac{4T}{\delta}}\\
&+\frac{\lambda S^{*}}{\kappa\phi Nt}\\
\le&\frac{12L_{1}+8B}{\kappa\phi\sqrt{t}}\sqrt{\log\frac{4T}{\delta}}+\frac{\frac{9}{4}L_{1}r+\frac{\lambda S^{*}}{N}}{\kappa\phi t}.
\end{align*}
\end{proof}

\section{A dimension-free bound for Hilbert-valued martingales}
In this section, we present a dimension-free bound for Hilbert-valued martingales whose norm is bounded.
In \citet{kim2021doubly}, the dimension-free bounds for the vector-valued and matrix-valued martingales are proved.
Lemma~\ref{lem:Hibert_concentration} generalizes the case to Hilbert-valued martingales when the differences are bounded. 
This inequality is useful to eliminate the dependency on the dimension of the Hilbert space.

\begin{lem}
\label{lem:Hibert_concentration}
(A dimension-free bound for Hilbert-valued martingales with bounded differences) 
Let $\mathcal{H}$ be a Hilbert space equipped with norm $\norm{\cdot}_{\mathcal{H}}$ and $\{N_{t}\}_{t=0}^{\infty}$ be a $\mathcal{H}$-valued martingale sequence with $N_{0}=0$. 
Suppose for each $t$, there exists a constant $b_{t}>0$ such that $\norm{N_{t}-N_{t-1}}_{\mathcal{H}}\le b_t$, almost surely.
Then for each $t$, 
\[
\Probability\left(\norm{N_{t}}_{\mathcal{H}}>x\right)\le4\exp\left(-\frac{x^{2}}{8\sum_{\tau=1}^{t}b_{\tau}^{2}}\right),
\]
holds with any $x>0$, and with probability at least $1-\delta$,
\[
\norm{N_{t}}_{\mathcal{H}}\le \sqrt{8\sum_{\tau=1}^{t} b_{\tau}^{2}\log\frac{4}{\delta}}.
\]
\end{lem}

\begin{proof}
By Lemma~\ref{lem:dimension_reduction}, there exists a $\Real^{2}$-valued
martingale $M_{t}$ such that 
\[
\norm{N_{t}}_{\mathcal{H}}=\norm{M_{t}}_{2},\;\norm{N_{t}-N_{t-1}}_{\mathcal{H}}=\norm{M_{t}-M_{t-1}}_{2},
\]
for any $t=1,2,\ldots.$, and $M_{0}=0$. Set $M_{t}=(M_{t}^{(1)},M_{t}^{(2)})$.
Then for each $t$ and $r=1,2$,
\begin{align*}
\abs{M_{t}^{(r)}-M_{t-1}^{(r)}}\le & \norm{M_{t}-M_{t-1}}_{2}=\norm{N_{t}-N_{t-1}}_{2}.
\end{align*}
Thus $\abs{M_{t}^{(r)}-M_{t-1}^{(r)}}\le b_{t}$ for all $t$, almost surely. 
By Lemma~\ref{lem:Azuma_Hoeffding_inequality}, for any $x>0$, 
\[
\Probability\left(\abs{M_{t}^{(r)}}\ge x\right)\le2\exp\left(-\frac{x^{2}}{2\sum_{\tau=1}^{t}b_{\tau}^{2}}\right).
\]
Since
\[
\norm{N_{t}}_{\mathcal{H}}=\norm{M_{t}}_{2}\le\abs{M_{t}^{(1)}}+\abs{M_{t}^{(2)}},
\]
by Lemma~\ref{lem:Azuma_Hoeffding_inequality},
\begin{align*}
\Probability\left(\norm{N_{t}}_{\mathcal{H}}\ge x\right)\le & \Probability\left(\abs{M_{t}^{(1)}}>\frac{x}{2}\right)+\Probability\left(\abs{M_{t}^{(2)}}>\frac{x}{2}\right)\\
\le & 4\exp\left(-\frac{x^{2}}{8\sum_{\tau=1}^{t}b_{\tau}^{2}}\right).
\end{align*}
Setting the last term smaller than $\delta \in (0,1)$ proves the Lemma.
\end{proof}

\subsection{A dimension-free bound for the Gram matrix}
Using this inequality we can obtain a dimension-free bound for the Gram matrix.

\begin{cor}
\label{cor:min_eigen_chernoff} 
Suppose the Assumptions 1-5 holds.
Then for each $t \ge 32 N^{-1} \phi^{-2} \log(4T/\delta)$, with probability at least $1-\delta/T$,
\begin{equation}
\Mineigen{\sum_{\tau=1}^{t}\sum_{i=1}^{N}\mu^{\prime}\left(\Context i{\tau}^{T}\beta\right)\Context i{\tau}\Context i{\tau}^{T}+\lambda I_{d}}\ge\frac{\phi N \kappa t}{2},
\label{eq:min_eigen_chernoff}
\end{equation}
holds for any $\beta\in\mathcal{B}_{r}^{*}$, $\delta\in(0,1)$ and $\lambda > 0$,
\end{cor}

\begin{proof}
Set $\XX i{\tau}=\Context i{\tau}\Context i{\tau}^{T}$. 
By Assumption 1 and 3, we have for any $\lambda > 0 $
\[
\Mineigen{\sum_{\tau=1}^{t}\sum_{i=1}^{N}\mu^{\prime}\left(\Context i{\tau}^{T}\beta\right)\XX i{\tau} + \lambda I_{d}}\ge\kappa\Mineigen{\sum_{\tau=1}^{t}\sum_{i=1}^{N}\XX i{\tau}},
\]
for all $\beta \in \mathcal{B}_r^{*}$.
Since $\Mineigen A=-\Maxeigen{-A}\ge-\norm A_{F}$, for $A\in\Real^{d \times d}$,
\begin{align*}
\Mineigen{\sum_{\tau=1}^{t}\sum_{i=1}^{N}\XX i{\tau}}\ge&\Mineigen{\sum_{\tau=1}^{t}\sum_{i=1}^{N}\XX i{\tau}-\Expectation\left[\sum_{\tau=1}^{t}\sum_{i=1}^{N}\XX i{\tau}\right]}+\Mineigen{\Expectation\left[\sum_{\tau=1}^{t}\sum_{i=1}^{N}\XX i{\tau}\right]}\\
\ge&-\norm{\sum_{\tau=1}^{t}\sum_{i=1}^{N}\XX i{\tau}-\Expectation\left[\sum_{\tau=1}^{t}\sum_{i=1}^{N}\XX i{\tau}\right]}_{2}+\phi Nt\\
\ge&-\norm{\sum_{\tau=1}^{t}\sum_{i=1}^{N}\XX i{\tau}-\Expectation\left[\sum_{\tau=1}^{t}\sum_{i=1}^{N}\XX i{\tau}\right]}_{F}+\phi Nt,
\end{align*}
where the second inequality holds due to Assumption 5.
Thus,
\begin{equation}
\begin{split}
\Probability&\left(\Mineigen{\sum_{\tau=1}^{t}\sum_{i=1}^{N}\mu^{\prime}
\left(\Context i{\tau}^{T}\beta\right)\XX i{\tau}+I_{d}}\le\frac{\kappa N\phi t}{2}\right)\\
&\le\Probability\left(\Mineigen{\sum_{\tau=1}^{t}\sum_{i=1}^{N}\XX i{\tau}}\le\frac{N\phi t}{2}\right)\\
&\le\Probability\left(\norm{\sum_{\tau=1}^{t}\sum_{i=1}^{N}\XX i{\tau}-\Expectation\left[\sum_{\tau=1}^{t}\sum_{i=1}^{N}\XX i{\tau}\right]}_{F}\!\!>\frac{N\phi t}{2}\right).
\end{split}
\end{equation}
By Assumption~\ref{assump:boundedness} and~\ref{assump:iid_contexts} and Lemma~\ref{lem:Hibert_concentration},
\[
\Probability\left(\norm{\sum_{\tau=1}^{t}\sum_{i=1}^{N}\XX i{\tau}-\Expectation\left[\sum_{\tau=1}^{t}\sum_{i=1}^{N}\XX i{\tau}\right]}_{F}\!\!>\frac{N\phi t}{2}\right)\le4\exp\left(-\frac{N\phi^{2}t}{32}\right)
\]
Setting the right term smaller than $\delta/T$ proves the result.
\end{proof}

\section{Technical lemmas}

\begin{lem}
\label{lem:chung_lemma}
\citep[Lemma 1, Theorem 32]{chung2006concentration}
For a filtration $\Filtration 0\subset\Filtration 1\subset\cdots\subset\Filtration T$, suppose each random variable $X_{t}$ is $\Filtration t$-measurable martingale, for $0\le t\le T$. 
Let $B_{t}$ denote the bad set associated with the following admissible condition: 
\[
\abs{X_{t}-X_{t-1}}\le c_{t},
\] 
for $1\le t\le T$, where $c_{1},\ldots,c_{n}$ are non-negative numbers. 
Then there exists a collection of random variables $Y_{0},\ldots,Y_{T}$ such that $Y_{t}$ is $\Filtration t$-measurable martingale such that 
\[
\abs{Y_{t}-Y_{t-1}}\le c_{t},
\]
and $\{\omega:Y_{t}(\omega)\neq X_{t}(\omega)\}\subset B_{t}$, for $0\le t\le T$.  
\end{lem}

\begin{rem}
We found a counter example where Lemma~\ref{lem:chung_lemma} does not hold.
Suppose for each $t\in[T]$, $\Filtration{t}:=\{X_1,\ldots,X_t\}$ and
\[
X_{t}-X_{t-1}:=
\begin{cases}
2-\frac{{t+1}^{2}}{\delta} & \text{with probability }\frac{\delta}{{t+1}^{2}},\\
2-\frac{1}{1-\frac{\delta}{{t+1}^{2}}} & \text{with probability }1-\frac{\delta}{{t+1}^{2}}.
\end{cases}
\]
Then $X_t$ is $\Filtration{t}$ measurable martingale  with $X_0=0$.
Set $B_t:=\{|X_t-X_{t-1}|>1\}$.
By Lemma~\ref{lem:chung_lemma}, there exists a $\Filtration{t}$-measurable martingale $Y_t$ such that $|Y_t-Y_{t-1}|<1$ and
$\{\omega:Y_{t}(\omega)\neq X_{t}(\omega)\}\subset B_{t}$, for $0\le t\le T$.
By Lemma~\ref{lem:Azuma_Hoeffding_inequality},
\[
\Probability\left(X_{T}\ge\sqrt{2T\log\frac{1}{\delta}}\right)\le\Probability\left(\left\{ X_{T}\ge\sqrt{2T\log\frac{1}{\delta}}\right\} \cap B_{T}^{c}\right)+\Probability\left(B_{T}\right)\le\Probability\left(Y_{T}\ge\sqrt{2T\log\frac{1}{\delta}}\right)+\sum_{t=1}^{T}\frac{\delta}{{t+1}^{2}}\le2\delta.
\]
However, by definition of $X_t$,
\[
\Probability\left(X_{T}\ge\sqrt{2T\log\frac{1}{\delta}}\right)\ge\Probability\left(X_{T}\ge T\right)\ge\Probability\left(X_{T}=2T-\sum_{t=1}^{T}\frac{1}{1-\frac{\delta}{\left(t+1\right)^{2}}}\right)=1-\sum_{t=1}^{T}\frac{\delta}{\left(t+1\right)^{2}}\ge1-\delta,
\]
holds for $\delta \ge e^{-T/2}$, which is a contradiction to the first inequality.
\label{rem:chung_lemma}
\end{rem}

\begin{lem}
\label{lem:dimension_reduction} \citep[Lemma 2.3]{lee2016} Let $\{N_{t}\}$
be an martingale on a Hilbert space $(\mathcal{H},\norm{\cdot}_{\mathcal{H}})$.
Then there exists an $\Real^{2}$-valued martingale $\{M_{t}\}$ such
that for any time $t\ge0$, $\norm{M_{t}}_{2}=\norm{N_{t}}_{\mathcal{H}}$
and $\norm{M_{t+1}-M_{t}}_{2}=\norm{N_{t+1}-N_{t}}_{\mathcal{H}}$.
\end{lem}

\begin{lem}
\citep{azuma1967weighted} 
\label{lem:Azuma_Hoeffding_inequality} 
(Azuma-Hoeffding inequality)
If a super-martingale $(Y_{t};t\ge0)$ corresponding to filtration
$\Filtration t$, satisfies $\abs{Y_{t}-Y_{t-1}}\le c_{t}$ for some
constant $c_{t}$, for all $t=1,\ldots,T,$ then for any $a\ge0$,
\[
\Probability\left(Y_{T}-Y_{0}\ge a\right)\le e^{-\frac{a^{2}}{2\sum_{t=1}^{T}c_{t}^{2}}}.
\]
\end{lem}

\begin{lem}
\label{lem:smooth_injection}
\citep[Lemma A]{chen1999strong}
Let $H$ be a smooth injection from $\Real^p$ to $\Real^p$ with $H(x_0)=y_0$.
Define $B_{\delta}(x_0):=\{x\in\Real^{p},\norm{x-x_0}_2 \le \delta \}$ and $S_{\delta}=\{x\in\Real^p,\norm{x-x_0}=\delta \}$.
Then $\inf_{x\in S_{\delta}(x_0)}\norm{H(x)-y_0}_2\ge r$ implies:
\[
H^{-1}\left(B_r(y_0)\right) \subseteq B_{\delta}\left(x_0\right)
\]
\end{lem}

\begin{lem}
\label{lem:freedman_inequalty}
\citep[Theorem 4]{dani2008stochastic} (\citeauthor{freedman1975tail}) Suppose $X_1$,\ldots,$X_T$ is a martingale difference sequence adapted to the filtration $\Filtration{0},\ldots,\Filtration{T-1}$, and $b$ is an uniform upper bound on the steps $X_i$.
Let $V$ denote the sum of conditional variances,
\[
V_T := \sum_{t=1}^{T} \CV{X_t^2}{\Filtration{t-1}}.
\]
Then, for every $a,v>0$
\[
\Probability\left(\left\{ \sum_{t=1}^{T}X_{t}\ge a\right\} \cap\left\{ V_{T}\le v\right\} \right)\le\exp\left(-\frac{a^{2}}{2v+2ab/3}\right).
\]
\end{lem}

\section{Additional experiment results}
In this section, we provide additional experiment results with confidence intervals.
We run the same experiment as in \ref{subsec:simluation_data} except that $T=2000$ and a fixed $\beta^{*}$ over the 5 repeated runs. 
%with $T=2000$.
Figure~\ref{fig:_sim_sd} shows the comparison of the average and standard deviation of the cumulative regrets.
The proposed algorithm outperforms four other candidates in all four scenarios, and their ranges of one standard deviation do not overlap with those of other models.

%------------------------------------------
%Simulation Figures
%------------------------------------------
\begin{figure}[t]
\centering
\begin{tabular}{cc}
\includegraphics[width=0.45\textwidth]{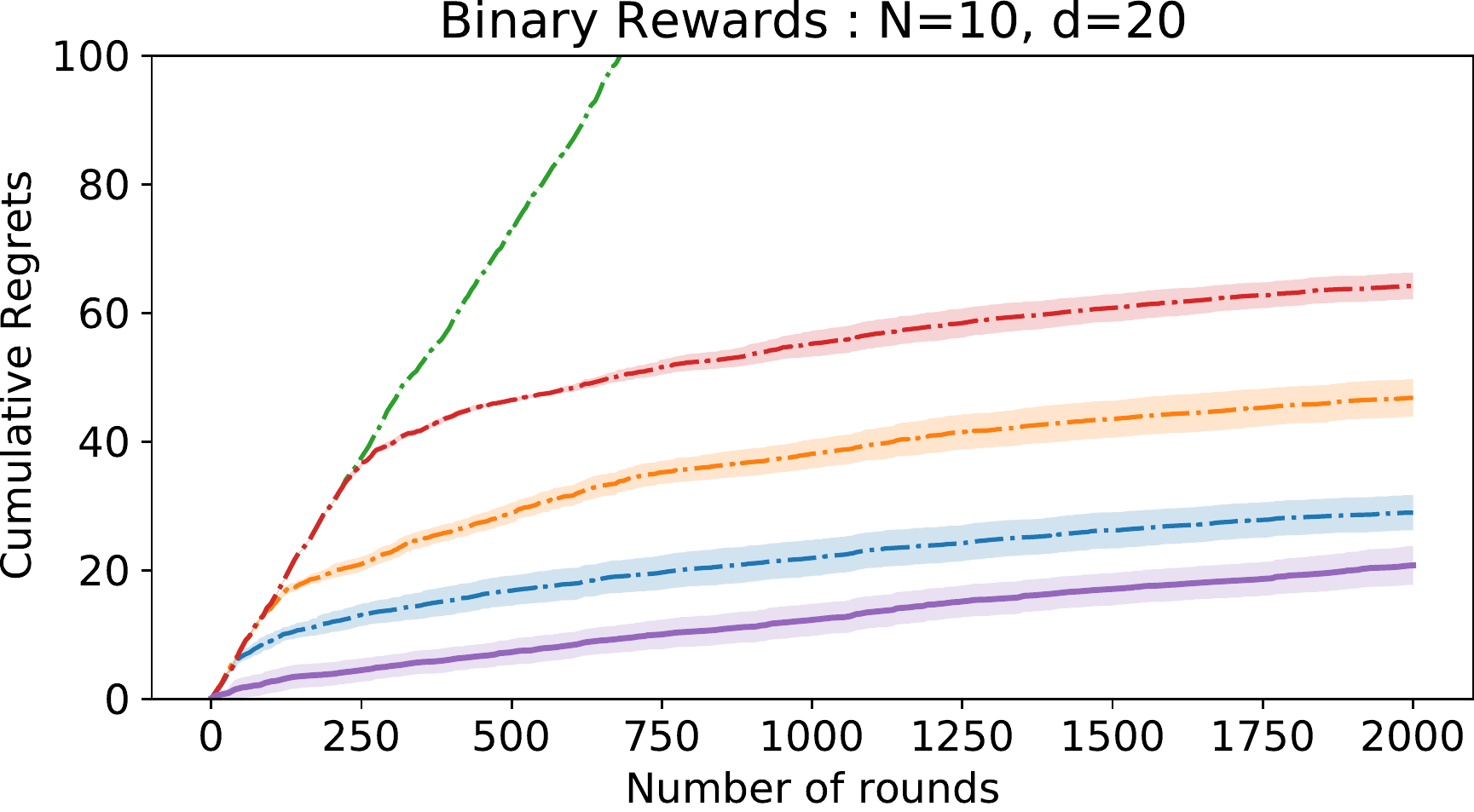}& \includegraphics[width=0.45\textwidth]{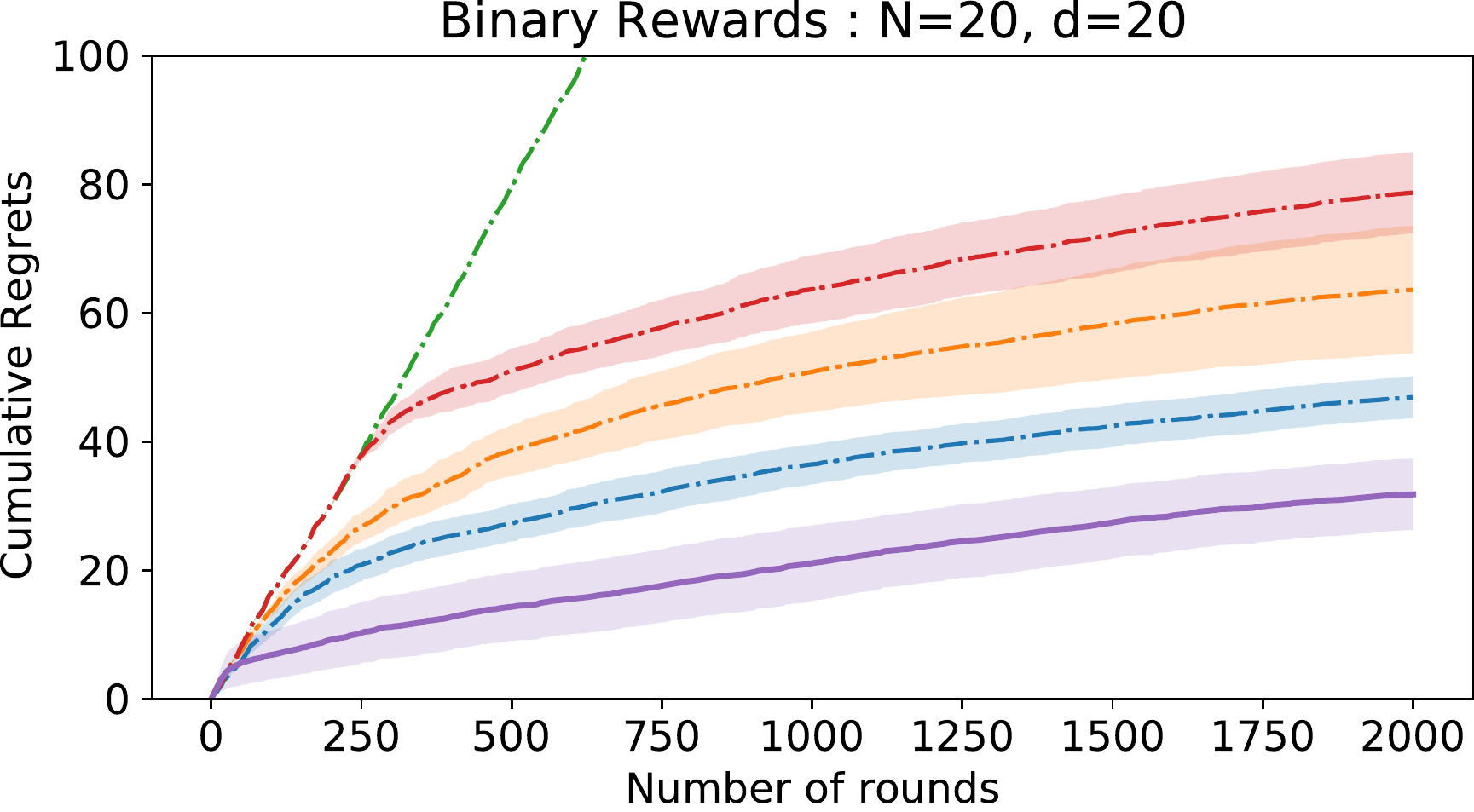}\\ \includegraphics[width=0.45\textwidth]{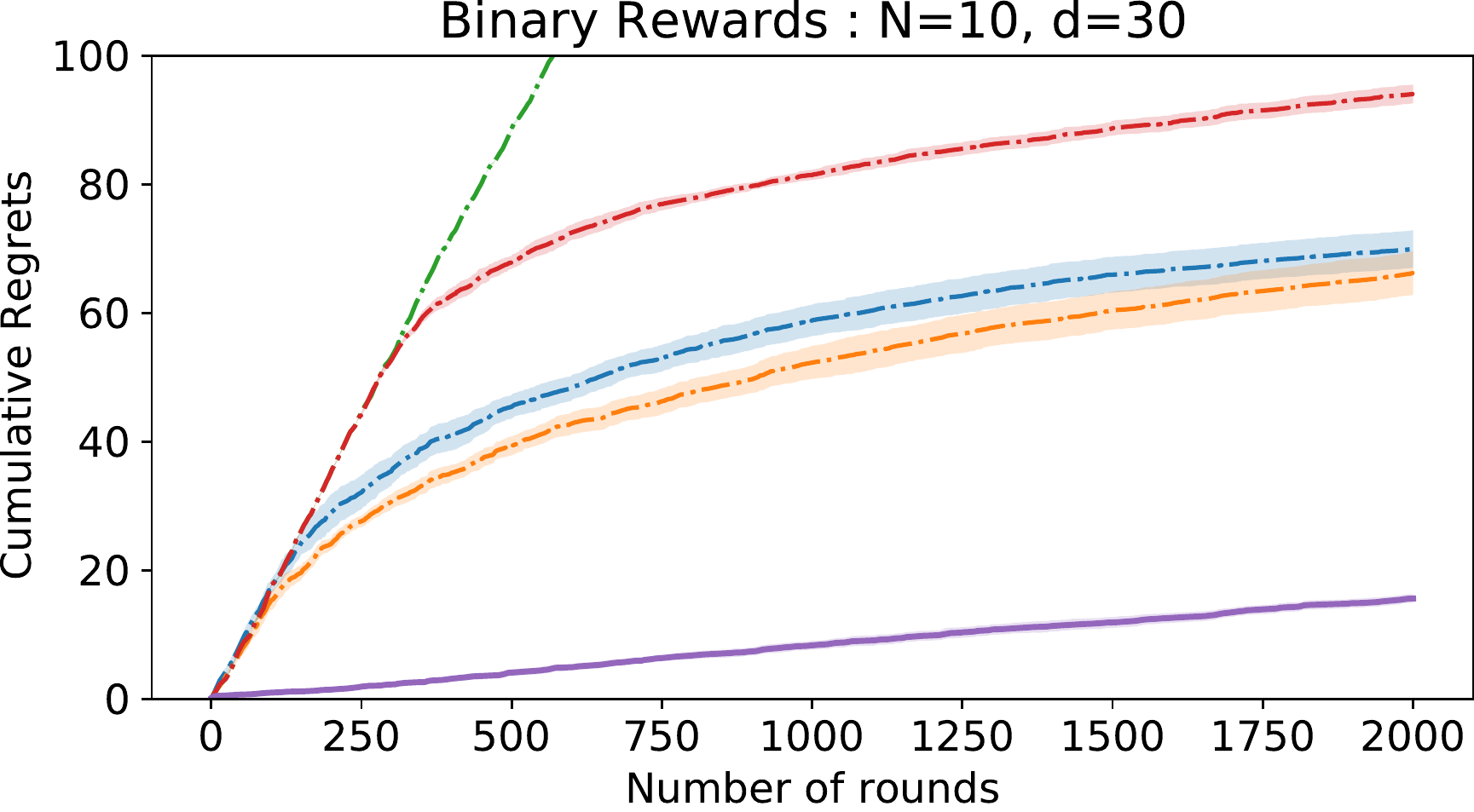}& \includegraphics[width=0.45\textwidth]{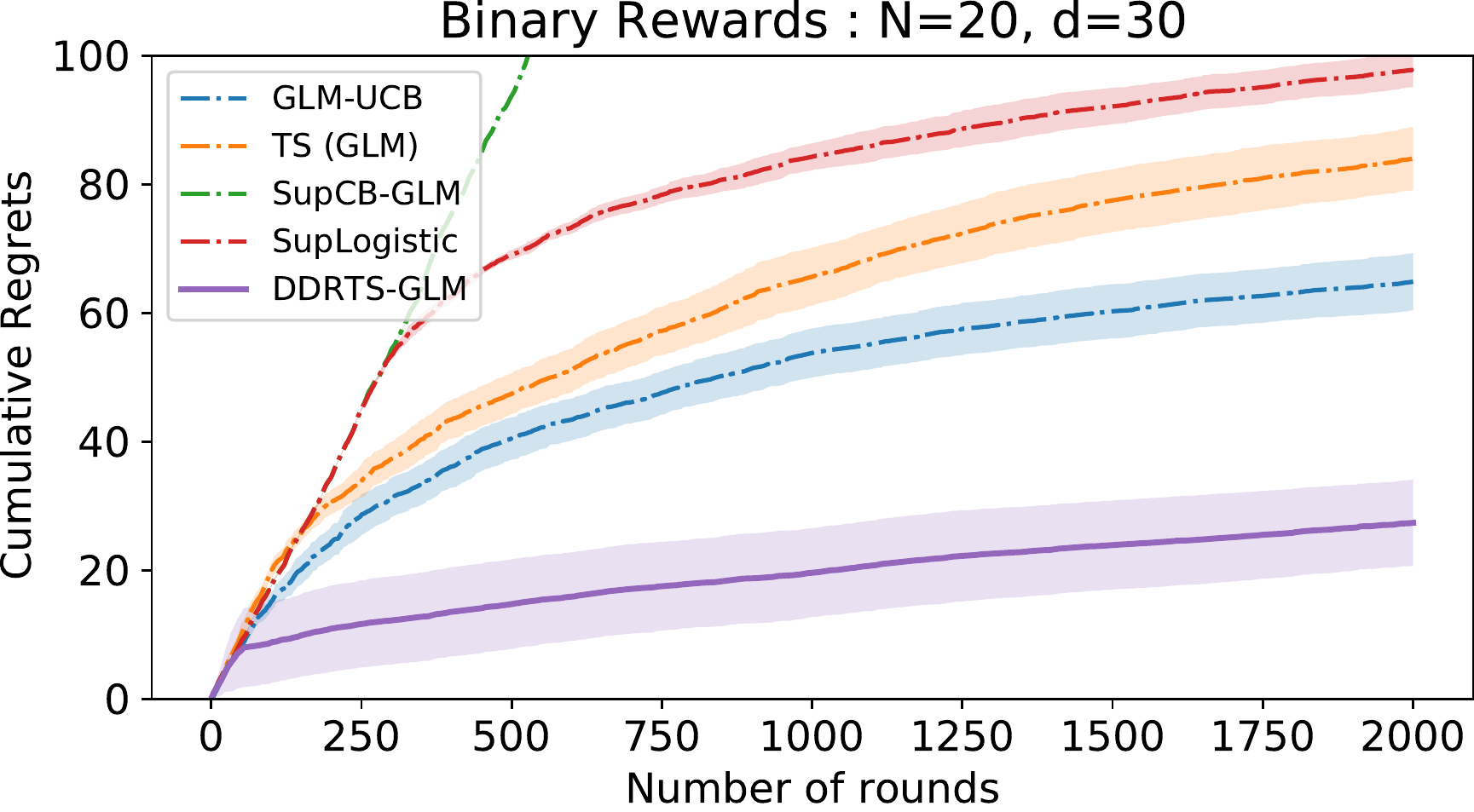}
\end{tabular}
\caption{Comparison of the average cumulative regret on synthetic dataset over 5 repeated runs with $T=2000$. The shaded area in the figures represents the one standard deviation range at each round.}% 4 pairs of $N$ and $d$.} 
\label{fig:_sim_sd}
\end{figure}

\section{Limitations}
\label{sec:limitations}
\begin{enumerate}
    \item 
    In our proposed algorithm \texttt{DDRTS-GLM}, there are more computations in resampling, imputation estimator and DDR estimator than those in $\texttt{LinTS}$ variants.
    In resampling, computation of $\Tildepi{m_t}{t}$ for at most $M_t$ times is required. 
    In computing the imputation estimator and DDR estimator, we use all contexts and the time complexity for estimation increases in $N$.
    However, these additional computations are minor when running the algorithm because resampling does not occur in most of the rounds and the estimation consists of strictly convex optimization.
    
    \item 
    The regret bound~\eqref{eq:regret_bound_phi} holds for certain context distributions which satisfies Assumption \ref{assump:iid_contexts} and Assumption~\ref{assump:minimum_eigenvalue}.
    The $\tilde{O}(\sqrt{\kappa^{-1}dT})$ regret bound holds for several practical cases stated in Section~\ref{subsec:phi_d_condition}.
    These assumptions do not hold in general.
    Even with this limitations, our work is the first among those for \texttt{LinTS} variants to propose novel regret analyses achieving an $\tilde{O}(\sqrt{\kappa^{-1}dT})$ regret bound for GLBs.
    
    \item 
    The lower bound is missing under our settings and assumptions and the optimality in our setting is not proved.
    However, finding a lower bound for the GLB problem is challenging and this will require another substantial work.
    To our knowledge, any lower bound for GLB problems is yet to be reported even in standard assumptions.
\end{enumerate}

\section{Computation of the selection probability}
\label{sec:pi_computation}
We refer to Section H in \citet{kim2021doubly} which proposed a Monte-Carlo estimate and showed that the estimate is efficiently computable.
In this section, we provide details of how to compute the selection probability, $\SelectionP{i}{t}$.
Because the counter example in Remark~\ref{rem:chung_lemma} can be applied our case when $N=2$ and $\Tildepi{1}{t}=1-\delta/t^2$ and $\Tildepi{2}{t}=\delta/t^2$, for some $t\in[T]$.
To avoid this example, we adjust our $\Tildepi{i}{t}$ to $\SelectionP{i}{t}$ such that
\[
{\pi}_{i,t}:=
\begin{cases}
\gamma+\epsilon_{t} & \Tildepi{i}{t}>\gamma\\
\frac{\gamma}{2} & \Tildepi{i}{t}\le\gamma
\end{cases}
\]
where $\epsilon_{t}>0$ is a constant that makes $\sum_{i=1}^{N}\SelectionP{i}{t}=1$ hold for each $t\in[T]$.
In this way, we obtain 
\[
\min_{i\in[N]}\SelectionP{i}{t} \ge \frac{\gamma}{2} \ge \frac{1}{2(N+1)} > \frac{\delta}{t^2}, 
\]
and this avoids the counter example in Remark~\ref{rem:chung_lemma}.
By resampling $\SelectionP{i}{t}$ instead of $\Tildepi{i}{t}$, we obtain the same theoretical results and the regret bounds in Theorem~\ref{thm:regret_bound} and Theorem~\ref{thm:fast_regret_bound} hold accordingly.

\end{document}